\documentclass[conference,compsocconf]{IEEEtran}
\usepackage{cite}
\ifCLASSINFOpdf
\else
\fi
\usepackage[cmex10]{amsmath}
\usepackage[noend]{algorithmic}
\ifCLASSOPTIONcompsoc
  \usepackage[caption=false,font=normalsize,labelfont=sf,textfont=sf]{subfig}
\else
  \usepackage[caption=false,font=footnotesize]{subfig}
\fi
\usepackage{fixltx2e}
\usepackage{url}
\usepackage{latexsym}
\usepackage{amssymb,amsthm}
\usepackage{enumitem}
\usepackage{bm}
\usepackage{anyfontsize}

\usepackage{tikz}
\usetikzlibrary{positioning,arrows,fit,calc}
\usetikzlibrary{decorations.pathmorphing}
\usepackage{pgfplots}
\pgfplotsset{compat=newest} %=1.10}

\usepackage{pifont}

\usepackage{multicol}

\usepackage{mathdots}

\newtheorem{theorem}{Theorem}
\newtheorem{proposition}[theorem]{Proposition}

\newtheorem{lemma}[theorem]{Lemma}
\newtheorem{remark}[theorem]{Remark}
\newtheorem{example}[theorem]{Example}

\usepackage{color}

\newcommand{\OWLQL}{\textsl{OWL\,2\,QL}}

\newcommand{\dlliter}{\textsl{DL-Lite$_R$}}

\def\Tmc{\ensuremath{\mathcal{T}}}
\def\Amc{\ensuremath{\mathcal{A}}}
\def\Kmc{\ensuremath{\mathcal{K}}}
\def\Imc{\ensuremath{\mathcal{I}}}
\def\Cmc{\ensuremath{\mathcal{C}}}

\def\cn{\ensuremath{\mathsf{N_{1}}}}
\def\rn{\ensuremath{\mathsf{N_{2}}}}
\def\rni{\ensuremath{\mathsf{N}^\pm_{2}}}

\def\ainds{\ensuremath{\mathsf{inds}}}

\def\sig{\mathsf{sig}}

\def\q{\mathbf{q}}
\def\vars{\mathsf{vars}}
\def\avars{\mathsf{avars}}

\def\canmod{\Cmc_{\Tmc,\Amc}}
\def\CmC{\Cmc_{\Tmc_\Cir^\vec{x},\Amc_\Cir}}
\def\qclin{\q_{\Cir}^\mathsf{lin}}
\def\qc{\q_{\Cir}}
\def\dcx{D_{\Cir}^\vec{x}}

\def\emptyword{\varepsilon}

\def\goal{\mathsf{goal}}
\def\twfn{\ensuremath{f^{\mathsf{tw}}_{\q,\Tmc}}}%{\ensuremath{f_{H^{\q}_\Tmc}}}
\def\homfn{\ensuremath{f^{\mathsf{tw}'}_{\q,\Tmc}}}
\def\primfn{\ensuremath{f^{\mathsf{prim}}_{\q,\Tmc}}}
\def\primfnP{\ensuremath{f^{\mathsf{prim}}_{\q_P,\Tmc_P}}}
\def\primfnn{\ensuremath{f^{\mathsf{prim}}_{\q_n,\Tmc_n}}}
\def\primsuper{\mathsf{prim}}
\def\nameprimfn{primitive evaluation}

\newcommand{\NOT}{\ensuremath{{\rm NOT}}}

\newcommand{\size}{\sigma}

\def\twset{\Theta^\q_\Tmc}

\def\numtypes{M}
\def\type{\xi}

\def\atom{\eta}

\def\twidth{t}

\def\valpha{\vec{\alpha}}

\def\vgamma{\vec{\gamma}}

\def\np{\ensuremath{\textsc{NP}}}
\def\lspace{\ensuremath{\textsc{L}}}

\def\nlpoly{\ensuremath{\textsc{NL}/\text{{\rm poly}}}}
\def\nppoly{\ensuremath{\textsc{NP}/\text{{\rm poly}}}}
\def\ppoly{\ensuremath{\textsc{P}/\text{{\rm poly}}}}
\def\ncone{\ensuremath{\textsc{NC}^1}}
\def\sac{\ensuremath{\SAC}}

\newcommand{\leftt}{\mathsf{{Left}}}
\newcommand{\rightt}{\mathsf{{Right}}}
\newcommand{\Reach}{\mathsf{{Reach}}}

\def\ind{\ainds}

\newcommand{\Cir}{{\boldsymbol{C}}}

\newcommand{\FO}{\text{FO}}
\newcommand{\PE}{\text{PE}}
\newcommand{\NDL}{\text{NDL}}

\newcommand{\tr}{\mathfrak{t}_\mathsf{r}}
\newcommand{\ti}{\mathfrak{t}_\mathsf{i}}
\renewcommand{\t}{\mathfrak{t}}
\renewcommand{\l}{\mathfrak{l}}

\newcommand{\NL}{\textsc{NL}}
\newcommand{\SAC}{\textsc{SAC}^1}
\newcommand{\LOGCFL}{\textsc{LOGCFL}}

\newcommand{\THP}{\textrm{THGP}}
\newcommand{\HP}{\textrm{HGP}}

\newcommand{\AND}{\ensuremath{{\rm AND}}}
\newcommand{\OR}{\ensuremath{{\rm OR}}}

\newcommand{\la}{\langle}
\newcommand{\ra}{\rangle}

\def\canmodrho{\Cmc_{\Tmc}^\varrho}

\def\logcfl{\mathrm{LOGCFL}}
\def\frontier{\mathsf{Frontier}}
\def\bbbdalgo{TreeQuery}
\def\bbarbalgo{BLQuery}
\def\compnode{MapCore}
\def\compnodeanon{MapAnon}
\def\compedge{MapEdge}

\def\llred{\ensuremath{\lspace^{\logcfl}}}
\def\stack{\mathsf{Stack}}
\def\false{\textbf{false}}
\def\true{\textbf{true}}
\def\yes{\textbf{yes}}
\def\no{\textbf{no}}
\def\stackheight{\mathsf{Height}}%{\mathsf{StackHeight}}
\def\children{\mathsf{Children}}
\def\deepest{\mathsf{Deepest}}

\def\dom{\mathsf{dom}}

\def\aor{U}%{\mathsf{or}}
\def\leftand{L}%{\mathsf{la}}
\def\rightand{R}%{\mathsf{ra}}
\def\trueleaf{A}

\def\tds{\mathsf{mtd}_t}
\def\TW{\Theta^\q_\Tmc}
\def\bvars{\mathsf{border}}
\def\qwo{\q^{U,\Omega}}

\tikzstyle{or-gate}=[rectangle,draw,inner sep=1.5mm]
\tikzstyle{and-gate}=[rectangle,draw,inner sep=1.5mm]
\tikzstyle{input}=[circle,draw,minimum size=4mm,text width=3mm]
\tikzstyle{vertex-c}=[circle,draw,minimum size=2mm,fill=gray!40]
\tikzstyle{vertex-sq}=[rectangle,draw,minimum size=2mm,fill=gray!40]
\tikzstyle{vertex}=[rectangle,draw,minimum size=0.5mm,fill=black,inner sep=0mm]
\tikzstyle{vertex-l}=[circle,draw,minimum size=1mm,fill=gray!40,inner sep=0mm]
\tikzstyle{round-rect}=[rectangle,rounded corners=3pt,draw,inner sep=1.5mm]

\renewcommand{\vec}[1]{\mathbf{#1}}

\begin{document}
\title{%TEMP TITLE
Tree-like Queries in OWL 2 QL: \\ Succinctness and Complexity Results 
%Succinctness and Complexity Results for %Answering 
%Tree-like Queries in OWL 2 QL
%Succinctness and Complexity Landscapes for Query Rewriting and Query Answering in OWL 2 QL
%A Parameterized Study of the Complexity of Query Rewriting and Query Answering in OWL 2 QL%: A Closer Look % A Parameterized Study
}

%\author{\IEEEauthorblockN{Meghyn Bienvenu}
%\IEEEauthorblockA{Laboratoire de Recherche en Informatique\\
%CNRS \& Universit\'{e} Paris-Sud\\
%Orsay, France}
%\and
%\IEEEauthorblockN{Stanislav Kikot}
%\IEEEauthorblockA{Twentieth Century Fox\\
%\and
%\IEEEauthorblockN{Volodya Podolskii}
%\IEEEauthorblockA{Steklov Mathematical Institute\\
%Moscow, Russia\\


\author{\IEEEauthorblockN{Meghyn Bienvenu\IEEEauthorrefmark{1},
Stanislav Kikot\IEEEauthorrefmark{2} and
Vladimir Podolskii\IEEEauthorrefmark{3}}
\IEEEauthorblockA{\IEEEauthorrefmark{1}Laboratoire de Recherche en Informatique,
CNRS \& Universit\'{e} Paris-Sud,
Orsay, France}
\IEEEauthorblockA{\IEEEauthorrefmark{2}Institute for Information Transmission Problems \& MIPT, Moscow, Russia}
\IEEEauthorblockA{\IEEEauthorrefmark{3}Steklov Mathematical Institute \&
National Research University Higher School of Economics,
Moscow, Russia}
}

\maketitle

\begin{abstract}
This paper investigates the impact of query topology on the difficulty of answering conjunctive queries in the presence of OWL 2 QL ontologies. Our first contribution is to clarify the worst-case size of positive existential (PE), non-recursive Datalog (NDL), and first-order (FO) rewritings for various classes of tree-like conjunctive queries, ranging from linear queries to bounded treewidth queries. Perhaps our most surprising result is a superpolynomial lower bound on the size of PE-rewritings that holds already for linear queries and ontologies of depth 2. %significantly strengthening earlier negative results. 
More positively, we show that polynomial-size 
NDL-rewritings always exist for tree-shaped queries with a bounded number of leaves (and arbitrary ontologies), and for bounded treewidth queries paired with bounded depth ontologies. For FO-rewritings, we equate %show that 
the existence of polysize rewritings % is equivalent to 
with well-known problems in Boolean circuit complexity. As our second contribution, %To complement these succinctness results, 
%we complement these succinctness results by performing a detailed 
we analyze the computational complexity of query 
answering and establish tractability results (either \NL- or \LOGCFL-completeness) for
a range of query-ontology pairs. % different classes of tree-like queries. 
%When combined with existing results,
%our results yield a complete picture of the succinctness and complexity landscapes for the considered classes of queries and ontologies.
Combining our new results with those from the literature %existing results, our new results, when combined together give us the two middle clouds in Figure~\ref{fig:1} and, together with known results, 
yields 
a complete picture of the succinctness and complexity landscapes 
for the considered classes of queries and ontologies.
\end{abstract}


% \fi
\IEEEpeerreviewmaketitle

\section{Introduction}
%\todo{First part quite a mess, trying many things...}
%In recent years, there has been 
Recent years have witnessed a growing interest from both the knowledge representation and 
database communities in \emph{ontology-based data access} (OBDA), 
in which the conceptual knowledge provided 
by an ontology is exploited when querying data. Formally, 
given an ontology $\Tmc$ (logical theory), a data instance $\Amc$ (set of ground facts), and a conjunctive query (CQ) $\q(\vec{x})$,
the problem is to compute the \emph{certain answers} to $\q$,
%over $(\Tmc, \Amc)$, 
that is, the tuples of constants $\vec{a}$ that satisfy
$\Tmc, \Amc \models \q(\vec{a})$. 

As scalability is crucial in data-intensive applications, much of the work on OBDA focuses % there has been a push towards 
on so-called 
%much 
%of the work on ontology-based data access has focused on 
`lightweight' ontology languages, which provide useful modelling features while retaining good computational properties. 
The DL-Lite family \cite{CDLLR07} of lightweight description logics has played a particularly prominent 
role, as witnessed by the recent introduction of the OWL 2 QL profile \cite{profiles} (based upon DL-Lite)
into the W3C-endorsed ontology language OWL 2. 
%Much of the practical work  
%%Ontology-based data access is currently a major topic in description logic (DL) research, 
%%Much of the work 
%on ontology-based data access %CQ answering %targets knowledge bases formulated using 
%has been centered around %focuses on 
%lightweight description logics of
%the DL-Lite family \cite{CDLLR07} of description logics and the corresponding OWL 2 QL profile \cite{profiles}. 
%The popularity of these languages is due to the 
%
The popularity of these languages is due to the fact that they enjoy %the possess a desirable property, known as 
\emph{first-order (FO) rewritability}, which means that 
for every CQ $\q(\vec{x})$ and %every 
ontology $\Tmc$, there exists a computable FO-query $\q'(\vec{x})$  (called a \emph{rewriting})
such that the certain answers to $\q(\vec{x})$ over $(\Tmc, \Amc)$ coincide %precisely
with the answers of %the FO-query 
$\q'(\vec{x})$ over the data instance $\Amc$ (viewed as an FO interpretation). 
First-order rewritability provides a means of reducing the \emph{entailment} problem of identifying certain answers % CQ answering in the presence of an ontology -- an \emph{entailment} problem --
to the simpler problem of FO \emph{model checking}; the latter %problem 
can be rephrased as SQL query evaluation and 
delegated to highly-optimized relational database management systems (RDBMSs). 
% rephrased as evaluating SQL queries over relational databases. 
%to the evaluation of FO ($\sim$ SQL) queries in relational databases -- model-checking .  
%, which can be handled by highly optimized 
%relational database management systems.
%
This appealing theoretical result %has 
spurred the development of numerous %a plethora of %great many 
%different 
query rewriting algorithms %have been proposed 
for OWL 2 QL and its extensions, cf.\
\cite{CDLLR07,Perez-UrbinaMH09,DBLP:conf/kr/RosatiA10,DBLP:conf/cade/ChortarasTS11,DBLP:conf/icde/GottlobOP11,DBLP:conf/esws/Rosati12,Perez-Urbina12,DBLP:conf/aaai/EiterOSTX12,DBLP:conf/rr/KonigLMT12,iswc13}.
Most %of these algorithms 
produce rewritings expressed as unions of conjunctive queries (UCQs), and 
experimental evaluation has shown that % demonstrated that %the size of such 
%generated 
such rewritings may be huge, making them difficult, or even impossible, to evaluate %them 
using %standard 
RDBMSs. %relational database management systems. 

The aim of this paper is to gain a better understanding of 
%better understand 
the difficulty of 
query rewriting and query answering in OWL 2 QL and %in particular, 
how it varies depending on the topology of the query. 

\smallskip 

%Such results highlight the importance of choosing the correct target query language for rewritings 
%\todo{for efficiency in practice, need reasonably-sized rewritings that are easy to generate and to evaluate}
%
%understand when reasonable-sized rewritings are possible,
%when not the case, give up or adopt alternative algorithmic approach?

%objective of this paper is to study the impact of query topology on difficulty of rewriting and answering queries in OWL 2 QL.
%so conjunctive query (CQ) answering 
%can be reduced to evaluation of FO ($\sim$ SQL) queries in relational databases. 
%More formally, first-order rewritability means that for every CQ $\q(\vec{x})$ and every TBox $\Tmc$, there exists an FO-query $\q'(\vec{x})$
%such that the \emph{certain answers} to $\q(\vec{x})$ over a knowledge base $(\Tmc, \Amc)$  (that is, the tuples
%of individuals $\vec{a}$ that satisfy
%$\Tmc, \Amc \models \q(\vec{a}$) coincide with the answers of the FO-query
%$\q'(\vec{x})$ over the ABox $\Amc$ (viewed as a database). 

%A great many different query rewriting algorithms have been proposed in the literature \cite{}.
%Most of these output rewritings expressed as unions of conjunctive queries (UCQs), and 
%the size of such rewritings can be huge, making it difficult, or even impossible, to evaluate them 
%using standard relational database management systems. 

\begin{figure*}[t]
\setlength{\abovecaptionskip}{3pt plus 0pt minus 0pt}
\centering
\scalebox{.85}{
\scriptsize
\begin{tikzpicture}
\begin{axis}[
    hide z axis,
    xmin=0, xmax=9,
    ymin=0, ymax=9,
        xtick={1,2,3,5,7,8},
    xtick style={color=black},
    xticklabels={$1$,$2$,$3$,$\dots$,$d$,arb},
%    xtick={1,2,3,4,5,6,8},
%    xtick style={color=black},
%    xticklabels={$1$,$2$,$\dots$,$d$,$\dots$,PIP,arb},
    ytick={1,2,3,4,5,6,7,8},
    ytick style={color=black},
    yticklabels={$2$,$\dots$,$\ell$,trees,%tree width
    tw 2,$\dots$,btw,arb},
    xlabel={\textsc{Ontology depth}},
    ylabel={Number of leaves \raisebox{4mm}{\textsc{Query shape}} %$\qquad\qquad$ 
    Treewidth$\qquad$}
]
\end{axis}
\node (x1) at (1.6,0.5) [] {};
\node (x2) at (6.3,2) [] {};
\node [fill=gray!10,draw,thin,rounded corners,inner sep=1mm,fit=(x2) (x1)] 
{$\nlpoly$: no poly PE but poly NDL \\
Thms.\ \ref{linear-lower}, \ref{bbcq-ndl}, \ref{nbps-conditional}
};
\node (y1) at (1.6,2.6) [] {};
\node (y2) at (5.2,4.4) [] {};
\node [fill=gray!30,draw,thin,rounded corners,inner sep=1mm,fit=(y2) (y1)] 
{$\SAC$: no poly PE but poly NDL\\
Thms.\ \ref{btw-ndl}, \ref{btw-fo}
};
\node (z1) at (5.8,2.6) [] {};
\node (z2) at (6.3,5.2) [] {};
\node [fill=gray!60,draw,thin,rounded corners,inner sep=1mm,fit=(z2) (z1)] 
{\rotatebox{90}{{\kern -6mm $\nppoly$ \ \ \cite{DBLP:conf/icalp/KikotKPZ12}}}};
\node (v1) at (1.6,5.0) [] {};
\node (v2) at (5.2,5.2) [] {};
\node [fill=gray!60,draw,thin,rounded corners,inner sep=1mm,fit=(v2) (v1)] 
{\raisebox{-4.2mm}{\vbox{\hbox{$\nppoly$: no polysize PE or NDL}\hbox to 4cm{\hfil\cite{lics14-KKPZ}\hfil}}}};
\node (t1) at (0.5,0.5) [] {};
\node (t2) at (1,4.4) [] {};
\node [draw,thin,rounded corners,inner sep=1mm,fit=(t2) (t1)] 
{Thm. \ref{depth-one-btw} \\[4mm] poly PE, FO and NDL\\[4mm] \cite{lics14-KKPZ}
};
\node (s1) at (0.5,5.0) [] {};
\node (s2) at (1,5.2) [] {};
\node [fill=gray!10,draw,thin,rounded corners,inner sep=1mm,fit=(s2) (s1)] 
{\fontsize{6}{7}\selectfont \\{\hspace*{-0.8mm}$\nlpoly$}\\\raisebox{0.4mm}{\cite{lics14-KKPZ}} };
\end{tikzpicture}
}
\quad
\scalebox{.85}{
\scriptsize
\begin{tikzpicture}
\begin{axis}[
    hide z axis,
    xmin=0, xmax=9,
    ymin=0, ymax=9,
    xtick={1,2,3,5,7,8},
    xtick style={color=black},
    xticklabels={$1$,$2$,$3$,$\dots$,$d$,arb},
    ytick={1,2,3,4,5,6,7,8},
    ytick style={color=black},
    yticklabels={$2$,$\dots$,$\ell$,trees,%tree width
    tw 2,$\dots$,btw,arb},
    xlabel={\textsc{Ontology depth}},
    ylabel={Number of leaves \raisebox{4mm}{\textsc{Query shape}} %$\qquad\qquad$ 
    Treewidth$\qquad$}
]
\end{axis}
\node (x1) at (0.5,0.5) [] {};
\node (x2) at (5.2,2) [] {};
\node [fill=gray!10,draw,thin,rounded corners,inner sep=1mm,fit=(x2) (x1)] 
{\NL-complete
\\[1.5mm] 
{\fontsize{6}{7}\selectfont $\geq$: \cite{CDLLR07}  / DBs  $\quad \leq$: Thm.\ \ref{nl-bb}}
};
\node (y1) at (0.5,2.6) [] {};
\node (y2) at (5.2,4.4) [] {};
\node [fill=gray!30,draw,thin,rounded corners,inner sep=1mm,fit=(y2) (y1)] 
{\LOGCFL-complete
\\[1.5mm]
{\fontsize{6}{7}\selectfont  $\geq$: \cite{DBLP:journals/jacm/GottlobLS01} $\quad\leq$: Thm. \ref{logcfl-btw}}
};
\node (z1) at (5.8,2.6) [] {};
\node (z2) at (6.3,5.2) [] {};
\node [fill=gray!60,draw,thin,rounded corners,inner sep=1mm,fit=(z2) (z1)] 
{\rotatebox{90}{{\kern -10mm NP-c \ \ {\fontsize{6}{7}\selectfont $\geq$: \cite{DBLP:conf/icalp/KikotKPZ12} $\,\,\leq$: \cite{CDLLR07}}}}};
%\node (v1) at (1.6,5.0) [] {};
%\node (v2) at (5.2,5.2) [] {};
%\node [fill=gray!60,draw,thin,rounded corners,inner sep=1mm,fit=(v2) (v1)] 
%{\raisebox{-4.2mm}{\vbox{\hbox{$\nppoly$: no polysize PE or NDL}\hbox to 4cm{\hfil\cite{lics14-KKPZ}\hfil}}}};
%\node (t1) at (0.5,0.5) [] {};
%\node (t2) at (1,4.4) [] {};
%\node [draw,thin,rounded corners,inner sep=1mm,fit=(t2) (t1)] 
%{Thm. \ref{depth-one-btw} \\[4mm] poly PE, FO and NDL\\[4mm] \cite{lics14-KKPZ}
%};
\node (s1) at (0.5,5.0) [] {};
\node (s2) at (5.2,5.2) [] {};
\node [fill=gray!60,draw,thin,rounded corners,inner sep=1mm,fit=(s2) (s1)] 
{\raisebox{-2.2mm}{\vbox{NP-complete {\fontsize{6}{7}\selectfont $\qquad \geq$: DBs $\quad\leq$: \cite{CDLLR07}}}}};
\node (z1) at (5.8,0.5) [] {};
\node (z2) at (6.3,2) [] {};
\node [fill=gray!30,draw,thin,rounded corners,inner sep=1mm,fit=(z2) (z1)] 
{\rotatebox{90}{{\kern -8mm \raisebox{2mm}{\vbox{\mbox{\LOGCFL-c}\\[1mm]  \mbox{\fontsize{5}{6}\selectfont $\geq,\leq$: Thm. \ref{logcfl-c-arb}%\ref{logcfl-lower-arb}, \ref{logcfl-upper-arb}
}
 }
%\vbox{\mbox{\LOGCFL-c}}
}}}};
\end{tikzpicture}
}
\caption{Succinctness landscape for query rewriting [left] and %combined 
complexity landscape for query answering [right]. 
We use the following abbreviations: 
`arb' for `arbitrary', `(b)tw' for `(bounded) treewidth',  %`btw' for `bounded treewidth', 
`poly' for `polynomial-size', `DBs' for `inherited from databases', and `c' for `complete'.  
On the left, `$\nppoly$' indicates that `polysize FO-rewritings only if $\nppoly \subseteq \ncone$'
and $C \in \{\nlpoly,\SAC\}$ means 
% non-uniform complexity class $C$ means 
`polysize FO-rewritings iff $C \subseteq \ncone$'.
}
\label{fig:1}
\vspace*{-.25cm}
\end{figure*}
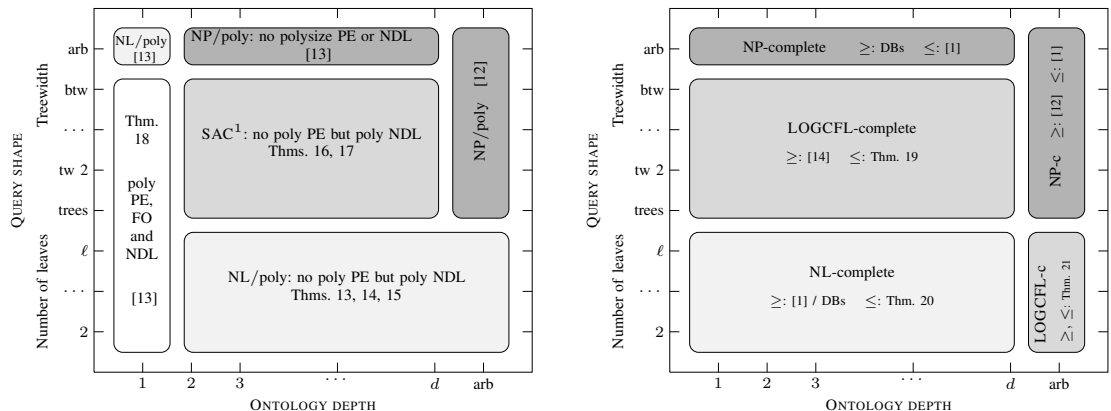

\noindent\textbf{Succinctness of Query Rewriting} It is not difficult to see that exponential-size rewritings are unavoidable 
if rewritings are given as UCQs (consider %for instance 
the query $%\q(x) = 
B_1(x) \wedge \dots \wedge B_n(x)$ and the ontology $\{A_i(x) \rightarrow B_i(x) \mid 1 \leq i \leq n\}$). 
A natural %(and non-trivial) 
question is whether an exponential blowup can be avoided 
%This naturally raises the question
%of whether this difficulty is due to the chosen representation (UCQ), or whether ti. Namely, do 
%polynomial-size rewritings exist if 
%Does the same hold 
by moving to other standard query languages,
like positive existential (PE) queries, non-recursive datalog (NDL) queries, or first-order (FO-) queries.
More generally, under what conditions can we ensure polynomial-size rewritings? %first-order (FO) queries.
%The practical difficulty of producing short query rewritings led to a recent line 
%of work aimed at answering the following fundamental questions: 
%When do short rewritings exist? How does the size of rewritings vary depending 
%on the chosen representation of rewritings?
%This fundamental question inspired a recent line of work aimed at understanding 
%the succinctness landscape of query rewriting in OWL 2 QL. 
A first (negative) answer was given in \cite{DBLP:conf/icalp/KikotKPZ12}, 
which %considered OWL 2 QL TBoxes and 
proved exponential lower bounds for 
the worst-case size of PE- and NDL-rewritings, as well as a superpolynomial lower 
bound %on the size of 
for FO-rewritings (%assuming %the latter results holds 
under the widely-held 
assumption that 
$\np \not \subseteq \ppoly$). 
Interestingly, all three results hold already for \emph{tree-shaped} CQs, which are 
a well-studied and practically relevant class of CQs that often enjoy better computational properties, cf.\ \cite{DBLP:conf/vldb/Yannakakis81,DBLP:conf/ijcai/BienvenuOSX13}. 
While the queries used in the proofs had a simple structure, the ontologies induced full binary trees of depth~$n$. % of labelled nulls. 
This raised the question of whether better results could be obtained by considering 
restricted classes of ontologies. % of a fixed depth, that is, for which the 
A recent study \cite{lics14-KKPZ} explored %whether better results could be obtained by considering 
%restricted classes of TBoxes. 
this question for ontologies of \emph{depth} 1 and 2, that is, ontologies for which the trees of labelled nulls 
appearing in the canonical model (aka chase) are guaranteed to be of depth at most 1 or 2
% whose existential rules induce trees of labelled nulls 
%of depth at most 1 or 2. 
% \todo{generate canonical
%models whose elements are at most 1 or 2 `steps away' from the ABox 
(see Section \ref{sec:prelims} for a formal definition). 
It was shown that for depth 1 ontologies, polysize PE-rewritings do not exist, polysize NDL-rewritings
do exist, and polysize FO-rewritings exist iff $\nlpoly \subseteq \ncone$. 
For depth 2 ontologies, neither polysize PE- nor NDL-rewritings exist, and polysize FO-rewritings do not exist 
unless $\np \subseteq \ppoly$. These results used simpler ontologies, but the considered CQs were
no longer tree-shaped. For depth 1 ontologies, this distinction is crucial, as it was further shown in \cite{lics14-KKPZ}  that polysize PE-rewritings \emph{do} exist for tree-shaped CQs.

While existing results go a fair way towards understanding the succinctness landscape 
of query rewriting in OWL 2 QL,  %(cf. Figure~\ref{fig:1}, two leftmost clouds and two clouds labelled \nppoly), 
a number of questions remain open:
\begin{itemize}
\item What happens if we consider tree-shaped queries and bounded depth ontologies?
\item What happens if we consider generalizations or restrictions of tree-shaped CQs?
\end{itemize}

\noindent\textbf{Complexity of Query Answering} 
%The preceding s
Succinctness results help us understand when polysize rewritings are possible, 
but they say little about the complexity of query answering itself. 
On the one hand, the existence of polysize rewritings is not sufficient to guarantee efficient query answering,
since small rewritings may nonetheless be difficult to produce and/or evaluate.
%(consider for example arbitrary CQs
%and the empty ontology). 
%Indeed, if we take the empty ontology, then any CQ is a rewriting of itself, yet the query answering problem is \NP-hard
%(for a less trivial example, consider the polysize result for \NDL-rewritings of arbitrary CQs and depth 1 ontologies). 
%For instance, we have seen 
%that polysize \NDL-rewritings exist for all CQs and depth 1 ontologies, yet   compare the 
%\np-hardness of answering arbitrary CQs over depth 1 ontologies with the polysize 
%arbitrary This is the case, 
%for instance, for   and 
On the other hand, negative results show that query rewriting may not always be practicable, 
but they leave open whether another approach to query answering might yield better results.
It is therefore important to investigate the complexity landscape of query answering, independently of any algorithmic approach. 

We briefly review the relevant literature. 
In relational databases, it is well-known that CQ answering is \np-complete in the general case.
A seminal result by Yannakakis established the tractability of answering tree-shaped CQs \cite{DBLP:conf/vldb/Yannakakis81},
and this result was later extended to wider classes of queries, most notably to bounded treewidth CQs \cite{DBLP:journals/tcs/ChekuriR00}.
Gottlob et al.\ \cite{DBLP:conf/icalp/GottlobLS99} pinpointed the precise complexity of answering tree-shaped and bounded treewidth CQs, 
showing both problems to be complete for the class \LOGCFL\ of all languages logspace-reducible to context-free languages \cite{DBLP:journals/jcss/Venkateswaran91}.  
In the presence of arbitrary OWL 2 QL ontologies, the \np\ upper bound for arbitrary CQs continues to hold \cite{CDLLR07},
but answering tree-shaped queries becomes \np-hard \cite{DBLP:conf/icalp/KikotKPZ12}. 
Interestingly, the latter problem was recently proven tractable in \cite{DBLP:conf/ijcai/BienvenuOSX13} for DL-Lite$_{\mathsf{core}}$ (a slightly less expressive logic than OWL 2 QL), 
raising the hope that other restrictions might also yield tractability. %the \np-hardness result for OWL 2 QL may not be that robust. 
We therefore have the following additional question:
\begin{itemize}
\item How do the aforementioned restrictions on queries and ontologies impact the complexity of query answering?
%\item How do the complexity and succinctness landscapes compare?
\end{itemize}

\noindent\textbf{Contributions} In this paper, we address the preceding questions by providing a complete picture of both the 
worst-case size of query rewritings %succinctness of query rewriting 
and the complexity of query answering for tree-shaped queries, their restriction to \emph{linear and bounded leaf
queries} (i.e. tree-shaped CQs with a bounded number of leaves), and their generalization to \emph{bounded treewidth queries}. 
Figure \ref{fig:1} gives an overview of new and existing results.
%\todo{we use non-uniform complexity classes as a measure of succinctness.}
%Our results reuse many of the techniques from \cite{}, including  
%but also introduce new tools, such 
%as interval and tree hypergraphs and SAC^1
%More specifically,

%With regards to
Regarding  succinctness, we establish a superpolynomial lower bound
on the size of PE-rewritings that holds already for linear queries and depth 2 ontologies, % of depth 2, 
significantly strengthening earlier negative results. 
For NDL-rewritings, the situation is brighter: we show that 
polysize 
rewritings always exist for 
bounded branching queries (and arbitrary OWL 2 QL 
ontologies), and for bounded treewidth queries and bounded depth ontologies.
We also prove that the succinctness problems concerning FO-rewritings are equivalent to 
well-known problems in %Boolean 
circuit complexity: $\nlpoly \subseteq \ncone$ in the case of 
linear and bounded leaf queries, and $\sac\subseteq \ncone$ in the case of tree-shaped
and bounded treewidth queries and bounded depth ontologies. 
Finally, to complete the succinctness landscape, we show that the result from \cite{lics14-KKPZ} 
that all tree-shaped queries
and depth 1 ontologies have polysize PE-rewritings generalizes to the wider class of bounded treewidth queries.
%Combining our new results with those from the literature %existing results, our new results, when combined together give us the two middle clouds in Figure~\ref{fig:1} and, together with known results, 
%yields 
%a complete picture of the succinctness landscape 
%for the considered classes of queries and ontologies (Fig.\ \ref{fig:1}, left). 
To prove our results, we establish tight connections
between Boolean functions induced by queries and ontologies and the non-uniform complexity 
classes $\nlpoly$ and $\sac$, reusing and further extending the 
machinery developed in \cite{DBLP:conf/icalp/KikotKPZ12,lics14-KKPZ}. 

Our complexity analysis reveals that all query-ontology combinations %under consideration 
that have not already been shown \np-hard are in fact tractable. Specifically, 
in the case of bounded depth ontologies, we prove membership in \LOGCFL\ for bounded treewidth queries (generalizing the result in \cite{DBLP:conf/icalp/GottlobLS99})
and membership in \NL\ for bounded leaf queries. 
We also show \LOGCFL-completeness for linear and bounded leaf queries in the presence of arbitrary OWL 2 QL ontologies. 
This last result is the most interesting technically, as %its
 upper and lower bounds 
rely on two different characterizations of the class \LOGCFL. 

%\todo{TO DO: compare landscapes}
%
%\todo{TO DO: why important -- bounded depth, tree-like queries both common in practice, \LOGCFL, \NL\ considered highly parallelizable}
%%Ontologies with finite depth, BTW / tree-shaped queries relevant, occur often in practice
%

%To put somewhere: 
%\begin{itemize}
%\item Ontologies with finite depth, BTW / tree-shaped queries relevant, occur often in practice
%\item First systematic investigation of combined complexity of QA in OWL 2 QL
%\item Mention Gottlob result: We consider so-called \emph{pure} FO-rewritings that do not allow constants, see  \cite{AIJ-forthcoming,DBLP:conf/kr/GottlobS12} for discussion and related results.
%\item Mention IJCAI paper results (here or later?)
%\end{itemize}
For lack of space, some proofs are deferred to the appendix.

\section{%Querying \OWLQL\ Knowledge Bases}
Preliminaries}
\label{sec:prelims}
%To facilitate the presentation of the technical results, 
\subsection{Querying \OWLQL\ Knowledge Bases}
We will work with the fragment of \OWLQL\ profile \cite{profiles}
that corresponds to the description logic \dlliter\ \cite{CDLLR07}, 
as the latter covers the most important features of \OWLQL\  and simplifies the technical treatment. 
Moreover, to make the paper accessible to a wider audience, 
we eschew the more common OWL and description logic notations in favour of traditional first-order logic (FO) syntax. 

%We assume countably infinite, mutually disjoint sets \cn\ and \rn\ of \emph{unary} and \emph{binary relation symbols}. 
%A data instance 

\begin{figure*}[t]
\centering
\begin{minipage}[b]{0.2\textwidth}
\footnotesize
\begin{align*}
\Tmc_0 =
 \{& P(x,y) \rightarrow R(x,y), \\
& P(x,y)\rightarrow U(y,x),  \\ %&R(x,y) \rightarrow V(y,x), \\
& A(x) \rightarrow \exists y P(x,y),  \\ &\exists y P(y,x) \rightarrow \exists y S(x,y), \\
& \exists y S(y,x) \rightarrow \exists y R(x,y),\\ 
& \exists y S(y,x) \rightarrow \exists y T(y,x), \\ 
& \exists y P(y,x) \rightarrow B(x)\} \\[2mm]
\Amc_0 = \{&A(a), R(a,c)\} 
\end{align*}
\end{minipage}
\quad
\scalebox{.95}{
\begin{tikzpicture}[>=latex, point/.style={circle,draw=black,thick,minimum size=1.5mm,inner sep=0pt}, wiggly/.style={thick,decorate,decoration={snake,amplitude=0.3mm,segment length=2mm,post length=1mm}},
query/.style={thick},
tw/.style={shorten <= 0.1cm, shorten >= 0.1cm,dashed},yscale=1,xscale=0.9]\footnotesize
%
%\node[fill=black] (a) at (1,3) [point, label=above:{$A$}, label=left:{$a$}]{};
%\node[fill=black] (c) at (2.5,3) [point, label=right:{$c$}]{};
%\node (d1) at (1,2) [point, fill=white, label=left:{$ar$}, label=right:{$B$}]{};
%\node (d2) at (1,1) [point, fill=white, label=left:{$ars$}]{};
%\node (d31) at (0.5,0) [point, fill=white, label=left:{$arsv$}]{};
%\node (d32) at (1.5,0) [point, fill=white, label=right:{$arst^-$}]{};
%
%%
%\draw[->,query] (a) to node[above] {$v$} (c);
%\draw[->,wiggly] (a)  to node [right]{$r, v, u^-$} (d1);
%\draw[->,wiggly] (d1)  to node [right]{$s$} (d2);
%\draw[->,wiggly] (d2)  to node [above,sloped]{ $v$} (d31);
%\draw[->,wiggly] (d2)  to node [above,sloped]{$t^-$} (d32);

%% for LICS version, capital letters for roles
%
\node[fill=black] (a) at (1,3) [point, label=above:{$A$}, label=left:{$a$}]{};
\node[fill=black] (c) at (2.5,3) [point, label=right:{$c$}]{};
\node (d1) at (1,2) [point, fill=white, label=left:{$aP$}, label=right:{$B$}]{};
\node (d2) at (1,1) [point, fill=white, label=left:{$aPS$}]{};
\node (d31) at (0.5,0) [point, fill=white, label=left:{$aPSR$}]{};
\node (d32) at (1.5,0) [point, fill=white, label=right:{$aPST^-$}]{};
\draw[->,query] (a) to node[above] {$R$} (c);
\draw[->,wiggly] (a)  to node [left]{\scriptsize $P, R, U^-$} (d1);
\draw[->,wiggly] (d1)  to node [left]{\scriptsize $S$} (d2);
\draw[->,wiggly] (d2)  to node [above,sloped]{\scriptsize $R$} (d31);
\draw[->,wiggly] (d2)  to node [above,sloped]{\scriptsize $T^-$} (d32);
\end{tikzpicture}}
\quad
\scalebox{.95}{
\begin{tikzpicture}[>=latex, point/.style={circle,draw=black,thick,minimum size=1.5mm,inner sep=0pt,fill=white},
spoint/.style={rectangle,draw=black,thick,minimum size=1.5mm,inner sep=0pt,fill=white},
ipoint/.style={circle,draw=black,thick,minimum size=1.5mm,inner sep=0pt,fill=lightgray},
wiggly/.style={thick,decorate,decoration={snake,amplitude=0.3mm,segment length=2mm,post length=1mm}},
query/.style={thick},yscale=1,xscale=1]\footnotesize
\node[ipoint,label=left:{$y_1$}, label=above:$B$] (y1) at (1,3) {};
\node[ipoint,label=left:{$y_2$}] (y2) at (0,2) {};
\node[ipoint,label=right:{$y_3$}] (y3) at (2,2) {};
\node[ipoint,label=left:{$x_1$}] (x1) at (0,1) {};
\node[ipoint,label=left:{$y_4$}] (y4) at (1.5,1) {};
\node[ipoint,label=right:{$y_5$}] (y5) at (2.5,1) {};
\node[ipoint,label=left:{$x_2$}] (x2) at (1.5,0) {};
%%\node[ipoint,label=right:{$y_6$}] (y6) at (2.5,0) {};
%
%
\draw[->,query] (y2) to node[above,sloped] {$P$} (y1);
\draw[->,query] (y1) to node[above,sloped] {$S$} (y3);
\draw[->,query] (y2) to node[left] {$R$} (x1);
\draw[->,query] (y4) to node[pos=0.4,above,sloped] {$S$} (y3);
\draw[->,query] (y5) to node[pos=0.4,above,sloped] {$T$} (y3);
\draw[->,query] (y4) to node[left] {$U$} (x2);
%
%% for LICS version, capital letters for roles
%
%\draw[->,query] (y5) to node[right] {$u$} (y6);
%\draw[->,query] (y2) to node[above,sloped] {\scriptsize $R$} (y1);
%\draw[->,query] (y1) to node[above,sloped] {\scriptsize $S$} (y3);
%\draw[->,query] (y2) to node[left] {\scriptsize $V$} (x1);
%\draw[->,query] (y4) to node[pos=0.4,above,sloped] {\scriptsize $S$} (y3);
%\draw[->,query] (y5) to node[pos=0.4,above,sloped] {\scriptsize $T$} (y3);
%\draw[->,query] (y4) to node[left] {\scriptsize $U$} (x2);
%%\draw[->,query] (y5) to node[right] {\scriptsize $U$} (y6);
\end{tikzpicture}}
\quad
\begin{minipage}[b]{0.15\textwidth}
\scriptsize
\begin{align*}
x_1 &\mapsto c\\
x_2 &\mapsto a\\
y_1 &\mapsto aP\\
y_2 &\mapsto a\\
y_3 &\mapsto aPS\\
y_4 &\mapsto aP\\
y_5 &\mapsto aPST^- \\[.2cm]
\end{align*}
\end{minipage}
%\footnotesize
%\begin{align*}
%\q'=& \q \vee \\
%(P(y_2,y_1) \wedge R(y_2,x_1) \wedge U(y_1,x_2)) \wedge \q_S \\
%& () \wedge q_P
%%\intertext{where:}
%%\q_S=&\\
%%\q_P=&
%\end{align*}

\caption{From left to right: the KB $(\Tmc_0,\Amc_0)$, its canonical model 
$\Cmc_{\Tmc_0, \Amc_0}$, the tree-shaped query $\q_0(x_1,x_2)$, and the homomorphism $h_0: \q_0(c,a) \rightarrow \canmod$.}
\label{ex-fig}
\vspace*{-.25cm}
\end{figure*}

\subsubsection{Knowledge bases} \label{prelims:kb}
We assume countably infinite, mutually disjoint sets \cn\ and \rn\ of \emph{unary} and \emph{binary predicate} names. 
We will typically use the characters  $A$, $B$ for unary predicates and $P$, $R$ for binary predicates. 
For a binary predicate $P$, we will use $P^-$ to denote the \emph{inverse} of $P$ and will \emph{treat an atom $P^-(t,t')$ as 
shorthand for $P(t',t)$} (by convention, $P^{--}=P$). The set of binary predicates and their inverses is denoted $\rni$,
and we use $\varrho$ to refer to its elements.

An \OWLQL\ \emph{knowledge base} (KB) can be seen as a \emph{pair of FO theories} $(\Tmc, \Amc)$, 
constructed using predicates from $\cn$ and $\rn$. 
The FO theory $\Tmc$, called the \emph{ontology} (or TBox), consists of finitely many sentences (or \emph{axioms}) of the forms
\begin{align*}
& \forall x\,\big(\tau(x) \to \tau'(x)\big), 
& \forall x,y\,\big(\varrho(x,y) \to \varrho'(x,y)\big), \\
& \forall x\, \big(\tau(x) \land \tau'(x) \to \bot \big),
& \forall x,y\,\big(\varrho(x,y) \land \varrho'(x,y) \to \bot\big),
%& \forall x\, \varrho(x,x) , &
%& \forall x\,\big(\varrho(x,x) \to \bot\big),
\end{align*}
where $\varrho \in \rni$ (see earlier) and 
%(cf.\ preceding paragraph and 
%the formula 
$\tau(x)$ %(called \emph{classes} or \emph{concepts}) and 
%$\varrho(x,y)$ (called \emph{properties} or \emph{roles}) are 
is defined as follows:
\begin{align*}
\tau(x) \ &::= \ \ A(x) \quad (A \in \cn) \ \mid \ \exists y\,\varrho(x,y) \quad (\varrho \in \rni) % \qquad (\text{where } A \in \cn) \\
%\varrho(x,y) \ & ::= \ \top \ \mid \ P(x,y) \ \mid \ P(y,x) \qquad (\text{where } P \in \rn). 
\end{align*}
Note that to simplify notation, we will %typically 
omit the universal quantifiers when writing ontology axioms. 
The \emph{signature} of $\Tmc$, denoted $\sig(\Tmc)$, is the set of predicate names in $\Tmc$,
and the \emph{size} of $\Tmc$, written $|\Tmc|$, is the number of symbols in $\Tmc$. 

The second theory $\Amc$, called the \emph{data instance} (or ABox), is a finite set of \emph{ground facts}.
%using the predicates from $\cn \cup \rn$. 
We use $\ind(\Amc)$ to denote the set of individual constants appearing in $\Amc$. 

The \emph{semantics} of KB $(\Tmc, \Amc)$ is the standard FO semantics of $\Tmc \cup \Amc$.
Interpretations will be given as pairs $\mathcal{I} = (\Delta^\mathcal{I}, \cdot^\mathcal{I})$, with $\Delta^\mathcal{I}$
the domain and $\cdot^\mathcal{I}$ the interpretation function;
models, satisfaction, consistency, and entailment are defined as usual.

%\smallskip
%\noindent \textbf{Query answering} 
\subsubsection{Query answering} A \emph{conjunctive query} (CQ) $\q(\vec{x})$ is an FO formula 
$\exists \vec{y}\, \varphi(\vec{x}, \vec{y})$, where $\varphi$ is a conjunction of atoms of the forms $A(z_1)$ or $R(z_1,z_2)$ with $z_i \in \vec{x} \cup \vec{y}$. The free variables $\vec{x}$ %in $\q$
are called \emph{answer variables}. Note that we assume w.l.o.g.\ %without loss of generality 
that CQs do not contain constants, and where convenient, we regard a CQ %$\q$
 as the set of its atoms. %, and the other variables are quantified variables. 
We use $\vars(\q)$ (resp.\ $\avars(\q)$) % $\qvars(\q)$ 
to denote the set of variables (resp.\ answer variables) % and answer variables, and quantified variables 
of $\q$. The \emph{signature} and \emph{size} of $\q$, defined similarly to above, are denoted $\sig(\q)$ and $|\q|$ respectively. 
% for the \emph{signature} and \emph{size} of $\q$, defined like the analogous notions above. 
%The \emph{signature} of $\q$, denoted $\sig(\q)$, is the set of predicates in $\q$. %Note that w
%
%and we say that $\q$ has tree
%

A tuple $\vec{a}\subseteq \ind (\mathcal{A})$ is a \emph{certain answer} to $\q(\vec{x})$ 
over $\mathcal{K} = (\mathcal{T},\mathcal{A})$ if  $\mathcal{I} \models \q(\vec{a})$ for all $\mathcal{I} \models \mathcal{K}$;
 in this case we write $\mathcal{K} \models \q(\vec{a})$. 
By first-order semantics, $\Imc \models \q(\vec{a})$ iff there 
is a mapping $h: \vars(\q) \rightarrow \Delta^{\Imc}$ such that 
(i) $h(z) \in A^{\Imc}$ whenever $A(z) \in \q$,
(ii) $(h(z), h(z')) \in r^{\Imc}$ whenever $r(z,z') \in \q$,
and (iii) $h$ maps $\vec{x}$ to $\vec{a}^\Imc$. 
If the first two conditions are satisified, then $h$ is a \emph{homomorphism} from $\q$ to $\Imc$, 
and we write $h \colon \q \to \Imc$. If (iii) also holds, 
then we write $h \colon \q(\vec{a}) \to \Imc$. 

%When we speak of the complexity of query answering, 
%we mean the complexity of the decision problem of recognizing certain answers. 

%\smallskip
%\noindent\textbf{Canonical model} 
\subsubsection{Canonical model}
We recall that every consistent \OWLQL\ KB $(\Tmc,\Amc)$ possesses a \emph{canonical model} (or chase) $\canmod$
with the property that 
\begin{align}
\Tmc,\Amc \models \q(\vec{a}) \quad \text{ iff }\quad \canmod \models \q(\vec{a}) \label{eq:canmod}
\end{align}
for every CQ $\q$ and tuple $\vec{a} \subseteq \ainds(\Amc)$. % of individuals from $\Amc$.
Thus, query answering in \OWLQL\ corresponds to \emph{deciding existence of a homomorphism of the query into the canonical model}. 

Informally, $\canmod$ is obtained from $\Amc$ by repeatedly applying the axioms in $\Tmc$, 
introducing fresh elements (labelled nulls) as needed to serve as witnesses for the existential quantifiers.  
Formally, the domain $\Delta^{\canmod}$ of $\canmod$ 
consists of $\ind(\Amc)$ and all words  
$a \varrho_1 \varrho_2 \ldots \varrho_n$ ($n \geq 1$) with $a \in \ind(\Amc)$ and $\varrho_i \in \rni$ ($1 \leq i \leq n)$ such that 
{\tolerance=3000
\begin{itemize}
\item $\Tmc, \Amc \models \exists y \varrho_1(a,y)$ and  $\Tmc, \Amc \models \varrho_1(a,b)$ for no $b \in \ainds(\Amc)$; 
\item for every $1 \leq i < n$: $\mathcal{T} \models \exists y\, \varrho_i(y,x) \rightarrow \exists y \, \varrho_{i+1}(x,y)$
  and $\Tmc \not \models \varrho_i(y,x) \rightarrow \varrho_{i+1}(x,y)$. %and $R_i^- \ne R_{i+1}$.
\end{itemize}
}
\noindent Predicate names are interpreted as follows:
\begin{align*}
 A^{\canmod} =   & \,    \{ a \in \ainds(\Amc) \mid \Tmc,\Amc \models A(a) \}  \cup \\
& \{ w\varrho \in \Delta^{\canmod} \mid 
\Tmc \models \exists y\, \varrho(y,x) \rightarrow A(x) \}   \\
 P^{\canmod} =   &\,    \{ (a,b) \mid \Tmc, \Amc \models P(a,b)  \} \, \cup  \\
& \,   \{ (w,w\varrho)
%\in  \Delta^{\canmod} \times \Delta^{\canmod}
\mid   \Tmc \models \varrho(x,y) \rightarrow P(x,y) \} \,\cup  \\ 
  & \,  \{ (w\varrho, w) 
%\in  \Delta^{\canmod} \times \Delta^{\canmod}
\mid \
\Tmc \models \varrho(y,x) \rightarrow P(x,y) \} %\\
 %a^{\Imc_{\Tmc,\Amc}} =    &\,    a \qquad \text{ for all } a \in  \ainds(\Amc)  
\end{align*}
Every constant $a \in \ainds(\Amc)$ is interpreted as itself: $a^{\canmod}=a$. 

Many of our constructions will exploit the fact that the canonical model has a forest structure: %the canonical model is forest-shaped: 
there is a \emph{core} involving the 
individual constants from the dataset  
and an \emph{anonymous part} consisting of 
trees of labelled nulls rooted at the constants. 

\begin{example} Figure \ref{ex-fig} presents a KB $(\Tmc_0, \Amc_0)$,
its canonical model $\Cmc_{\Tmc_0, \Amc_0}$, a CQ $\q_0$, and a homomorphism
$\q_0(c,a) \rightarrow \Cmc_{\Tmc_0, \Amc_0}$ witnessing that $(c,a)$
is a certain answer to $\q_0$. % over $(\Tmc_0, \Amc_0)$. 
% $\Tmc, \Amc \models \q_0(c,a)$, 
\end{example}

\subsubsection{The considered classes of queries and ontologies}
For ontologies, the parameter of interest is the depth of an ontology.
An ontology $\Tmc$ is \emph{of depth} $\omega$ 
if there is a data instance $\Amc$ such that the domain of $\canmod$ is infinite;
%at least one of $\Cmc_\Tmc^R(a)$ is infinite; 
$\Tmc$ is \emph{of depth} $d$, $0 \le d < \omega$, if $d$ is the greatest number such that 
some $\canmod$ %$\Cmc_\Tmc^R(a)$ 
contains an element of the form $a \varrho_1 \dots \varrho_{d}$. 
Clearly, the depth of $\Tmc$
can be computed in polynomial time,
and if $\Tmc$ is of finite depth, then its depth cannot exceed $2 |\Tmc|$. 

The various classes of tree-like queries considered in this paper are defined 
by associating with every CQ $\q$ the undirected graph $G_\q$ 
whose vertices are the variables of $\q$,
and which contains an edge $\{u,v\}$ whenever $\q$ contains
some atom $R(u,v)$ or $R(v,u)$. We call a CQ $\q$ \emph{tree-shaped} 
if the graph $G_\q$ is acyclic, and we say that $\q$ has $k$ \emph{leaves} if 
the graph $G_\q$ contains exactly $k$ vertices of degree 1. 
A \emph{linear} CQ is a tree-shaped CQ with $2$ leaves.  

The most general class of queries we consider are \emph{bounded treewidth queries}. 
We recall that a \emph{tree decomposition} of an undirected graph $G=(V,E)$ is a pair $(T,\lambda)$
such that $T$ is an (undirected) tree and $\lambda$ assigns a label
$\lambda(N) \subseteq V$ %(called a \emph{bag})
to every node $N$ of $T$ such that the following conditions are satisfied:
\begin{enumerate}
\item For every $v \in V$, there exists a node $N$ with $v \in \lambda(N)$.
\item For every 
$e \in E$, 
there exists a node $N$ with $e \subseteq \lambda(N)$.
\item For every $v \in V$, the nodes $\{N\mid v \in \lambda(N)\} $ induce a connected subtree of~$T$.
\end{enumerate}
The \emph{width of a tree decomposition} $(T, \lambda)$ is equal to $\mathsf{max}_{N} |\lambda(N)| - 1$,
and the \emph{treewidth of a graph} $G$ is the minimum width over all tree decompositions of $G$.
The \emph{treewidth of a CQ} $\q$ is defined as the treewidth of the graph~$G_\q$.

\subsection{Query Rewriting and Boolean Functions}\label{sec:fns}
We next recall the definition of query rewriting and show how the (worst-case) size of rewritings can be
related to representations of particular Boolean functions. 
We assume the reader is familiar with Boolean circuits~\cite{Arora&Barak09,Jukna12},
built using $\AND$, $\OR$, $\NOT$ and input %,  input and output
gates.
The \emph{size of a circuit} $\Cir$, denoted $|\Cir|$, is defined as the number of its gates.
We will be particularly interested in \emph{monotone circuits} (that is, circuits with no $\NOT$ gates).
(Monotone) \emph{formulas} are (monotone) circuits whose underlying graph is a tree.
\subsubsection{Query rewriting}
With every data instance $\Amc$, we associate the interpretation $\Imc_\Amc$ 
whose domain is $\ainds(\Amc)$
and whose interpretation function makes true precisely the facts in~$\Amc$.
%Suppose $\Tmc$ is a TBox and $\q(\vec{x})$ a CQ. 
We say an FO formula $\q'(\vec{x})$ with free variables $\vec{x}$ and without constants %\footnote{Here we focus on so-called \emph{pure rewritings}, as considered in \cite{} and used in existing query rewriting systems. Impure rewritings, which exploit existential quantification over fixed constants, behave differently with respect to succinctness. We invite the reader to consult \cite{AIJ-forthcoming} for detailed discussion.} 
 is an \emph{\FO-rewriting of a CQ $\q(\vec{x})$ and an ontology $\Tmc$ } if, for any data instance $\Amc$  and tuple $\vec{a}\subseteq \ind(\Amc)$, we have $\Tmc, \Amc \models \q(\vec{a})$ iff $\Imc_\Amc \models \q'(\vec{a})$. 
If $\q'$ is a positive existential formula (i.e.\ it only uses $\exists$, $\land$, $\lor$), then it is called a \emph{\PE-rewriting} of $\q$ and $\Tmc$. 

We also consider rewritings in the form of nonrecursive Datalog queries.
We remind the reader that a \emph{Datalog program} %(typically denoted $\Pi$) %
is a finite set of rules %Horn clauses
$\forall \vec{x}\, (\gamma_1 \land \dots \land \gamma_m \to \gamma_0)$,
where each $\gamma_i$ is an atom of the form $G(x_1,\dots,x_l)$ with  $x_i \in \vec{x}$. 
The atom $\gamma_0$ is called the \emph{head} of the rule, 
and $\gamma_1,\dots,\gamma_m$ its \emph{body}. 
All variables in the head must also occur in the body. 
A predicate $G$ \emph{depends} on a predicate $H$ in program $\Pi$ if $\Pi$ contains a rule whose head predicate is $G$ and whose body contains $H$. The program $\Pi$ is called \emph{nonrecursive} if there are no cycles in the dependence relation for $\Pi$. 
For a nonrecursive Datalog program $\Pi$ and a predicate $\goal$, 
we say that $(\Pi,\goal)$ is an \emph{\NDL-rewriting of $\q$ and $\Tmc$} in case $\Tmc,\Amc \models \q(\vec{a})$ iff  $\Pi,\mathcal{A} \models \goal(\vec{a})$, 
for every data instance $\Amc$ and tuple $\vec{a} \subseteq \ind(\mathcal{A})$.

\begin{remark}
Observe that we disallow constants in rewritings, that is, we consider so-called 
\emph{pure rewritings},  as studied in \cite{DBLP:conf/icalp/KikotKPZ12,lics14-KKPZ} 
and implemented in existing rewriting systems. 
 %Suppressing constants, 
%We focus on so-called \emph{pure rewritings} (as in
%\cite{DBLP:conf/icalp/KikotKPZ12,lics14-KKPZ}) as used in existing query rewriting systems. 
Impure rewritings, which use existential quantification over fixed constants, behave differently regarding %with respect to 
succinctness \cite{DBLP:conf/kr/GottlobS12}. Please see %We invite the reader to consult
 \cite{AIJ-forthcoming} for detailed discussion.
\end{remark}
\vspace*{-1mm}
\subsubsection{Upper bounds via tree witness functions}
The upper bounds on rewriting size shown in \cite{lics14-KKPZ} rely on
associating a Boolean function $\twfn$ % associating a hypergraph $H^{\q}_\Tmc$
with every query $\q$ and ontology $\Tmc$. % and then associating a Boolean function $f_H$ with every hypergraph~$H$. 
%As $\twfn$ is defined using the notion of tree witnesses, we start by recalling 
The definition of the function 
$\twfn$ makes essential use of the notion of tree witness \cite{KR10our}, which we recall next. %hypergraph is defined in terms of tree witnesses, 
%we first recall the definition of tree witnesses. %The definition involves tree witnesses, which we recall first. 
%but to do so, we must first recall the definition of tree witnesses.
%Consider a CQ $\q$ and a TBox $\Tmc$. 

For every $\varrho \in \rni$, 
we let $\canmodrho$ be the canonical model of the KB 
$(\Tmc \cup \{A_\varrho(x) \rightarrow \exists y \varrho(x,y)\}, \{A_\varrho(a)\})$,
where $A_\varrho$ is a fresh unary predicate. 
Given a CQ $q$  %with existential variables $\vec{y}$ 
and a pair $\t = (\tr, \ti)$ of disjoint subsets of $\mathsf{vars}(\q)$ %of variables in $\q$, 
such that $\ti\subseteq \vars(\q) \setminus \avars(\q)$
and $\ti \ne\emptyset$, % ($\tr$ can be empty), 
we set
$
\q_\t \ = \ \{\, S(\vec{z}) \in \q \mid \vec{z} \subseteq \tr\cup \ti \text{ and } \vec{z}\not\subseteq \tr\,\}.
$
The pair $\t = (\tr, \ti)$ is called a \emph{tree witness for $\q$ and $\Tmc$ generated by $\varrho$} 
if there is a homomorphism $h \colon \q_\t  \to \canmodrho$ sending $\tr$ to $a$ and 
$\q_\t$ is a minimal subset of $\q$ that contains all atoms involving a variable from $\ti$.
We denote by $\twset$ (resp.\ $\twset[\varrho]$) the set of tree witnesses for $\q$ and $\Tmc$ (resp.\ generated by $\varrho$).

\begin{example} 
There are 3 tree witnesses for $\q_0$ and $\Tmc_0$: %from Figure \ref{}:
$\t^1%=(\t_r^1, \t_i^1)
\hspace{-0.4mm}
= 
\hspace{-0.4mm}
(\{x_2,y_2\}, \{y_1,y_3,y_4,y_5\})$,  
$\t^2%=(\t_r^2, \t_i^2)
\hspace{-0.4mm}
= 
\hspace{-0.4mm}
(\{y_1,y_4\}, \{y_3, y_5\})$, and $\t^3 =(\{y_3\},\{y_5\})$, generated by $P$, $S$, and $T^-$ respectively. 
%\begin{align*}
%\t^1&=(\t_r^1, \t_i^1)= (\{x_2,y_2\}, \{y_1,y_3,y_4,y_5\})\\
%\t^2&=(\t_r^2, \t_i^2)= (\{y_1,y_4\}, \{y_5\})
%\end{align*}
%The homomorphism $h_0$ maps $\q_{\t^{1}}= \q \setminus \{R(y_2,x_1)\}$ into the 
%anonymous part of $\Cmc_{\Tmc_0,\Amc_0}$ and the remaining atom $R(y_2,x_1)$ into the core. 
\end{example}

%Intuitively, each subquery $\q_\t$ induced by a tree witness $\t \in $ corresponds to a (minimal) part of the query that can be mapped into 
%the anonymous part, and 
Every homomorphism of $\q$ into $\canmod$ 
induces a partition of $\q$ into subqueries $\q_{\t_1}, \ldots, \q_{\t_n}$ ($\t_i \in \twset$) that are mapped 
into % (not necessarily distinct)
%trees of 
the anonymous part and the remaining atoms %$\q \setminus \cup_{i=1}^n \q_{\t_i}$ 
that are mapped into the core. 
The \emph{tree witness function for $\q$ and $\Tmc$} %, introduced in \cite{lics14-KKPZ},
captures the different %partitions of 
ways of partitioning $\q$:
\vspace*{-1.5mm}
\begin{equation*}
\twfn = \!\!\!\!\bigvee_{\substack{\Theta \subseteq \twset\\ \text{ independent}}}
\bigg(\bigwedge_{\atom \in \q \setminus \q_\Theta} \!\!p_\atom 
 \wedge  \bigwedge_{\t \in \Theta} p_\t \bigg) \vspace*{-1.5mm}
 %\big(\!\!\bigwedge_{z,z'\in\t} \!\!p_{z=z'} \wedge\!\! \bigvee_{\substack{\varrho \in \rni,\\ \t \in \twset[\varrho]}}\!\! \bigwedge_{z \in \t} p_{z}^\varrho\big)\bigg)  
\end{equation*}
Here $p_\atom $ and $p_\t$ are Boolean variables, 
%In the previous equation, 
$\q_\Theta$ stands for %is shorthand for 
$\bigcup_{\t \in \Theta} \q_\t$, and `$\Theta$ independent' 
means %that 
$\q_\t \cap \q_\t' = \emptyset$ for all $\t \neq \t' \in \Theta$. % with $\t \neq \t'$. }

%By substituting the atom $\varrho$ for $p_\varrho$ and a query guaranteeing the satisfaction of 
In \cite{lics14-KKPZ}, it is shown how a Boolean formula or circuit computing $\twfn$ can be 
transformed into
%show how to construct 
a rewriting of $\q$ and $\Tmc$.
%from a Boolean formula / circuit that computes $\twfn$.  
%representations of the function $\twfn$ as a (monotone) Boolean formula or circuit
%can be used to construct PE-, NDL-, or FO-rewritings. 
Thus, the circuit complexity of $\twfn$ 
provides an upper bound on the size of rewritings of $\q$ and $\Tmc$. 

\begin{theorem}[from \cite{lics14-KKPZ}]\label{TW2rew}
If $\twfn$  is computed by a \textup{(}monotone\textup{)} Boolean formula $\chi$ then
there is a \textup{(}\PE-\textup{)} \FO-rewriting of $\q$ and $\Tmc$ of size $O(|\chi| \cdot |\q| \cdot |\Tmc|)$.

If $\twfn$ is computed by a monotone Boolean circuit $\Cir$ then there is an \NDL-rewriting of $\q$ and $\Tmc$ of size $O(|\Cir|\cdot |\q| \cdot |\Tmc|)$.
\end{theorem}

Observe that %the function 
$\twfn$ contains a variable $p_\t$ for every tree witness $\t$, and so 
it can only be used to show 
polynomial upper bounds in cases where $|\twset|$
is bounded polynomially in $|\q|$ and $|\Tmc|$.
%
%In order to obtain useful upper bounds even % in the case where
%when there may be exponentially many
%tree witnesses, we
We therefore introduce the following variant: %a variant of $\twfn$: 
%This motivates us to introduce the following variant of $\twfn$: %a slightly different function 
\begin{equation*}
\homfn = \!\!\!\!\bigvee_{\substack{\Theta \subseteq \twset\\ \text{ independent}}}
\bigg(\bigwedge_{\atom \in \q \setminus \q_\Theta} \!\!p_\atom 
 \wedge  \bigwedge_{\t \in \Theta} \big(\!\!\bigwedge_{z,z'\in\t} \!\!p_{z=z'} \wedge\!\! \bigvee_{\substack{\varrho \in \rni,\\ \t \in \twset[\varrho]}}\!\! \bigwedge_{z \in \t} p_{z}^\varrho\big)\bigg)  
\end{equation*}
% that is obtained from $\twfn$ by
%replacing each variable $p_\t$ by:
%$$\bigwedge_{i < j, x_i, x_j \in \t} p_{i=j} \,\,\wedge \bigvee_{\substack{R \in \rni, \t \in \twset[R]}} \,\, \bigwedge_{x_i \in \t} p_{i}^R$$
%where $\q_\Theta = \bigcup_\t \q_\t$. 
Intuitively, % the variable
we use $p_{z=z'}$ to enforce that variables $z$ and $z'$ are mapped to elements of $\canmod$ that
begin by the same individual constant and $p^{\varrho}_z$ to ensure that $z$ is mapped to an element whose initial constant $a$
 satisfies $\Tmc, \Amc \models \exists y \varrho(a,y)$.

We observe that the number of variables in $\homfn$ is polynomially bounded in $|\q|$ and $|\Tmc|$.
Moreover, we can prove that it 
%but $\homfn$ 
has the same properties as $\twfn$ regarding upper bounds.
% contains at most $|q| (|q| + |\Tmc| + 1)$ variables.
%The following
%theorem shows that it has the same properties as $\twfn$ regarding upper bounds. %be used in place of $\twfn$ to obtain upper bounds. %\todo{something}

\begin{theorem}\label{Hom2rew} Thm.\ \ref{TW2rew} remains true %continues to hold
 if $\twfn$ is replaced by $\homfn$.
\end{theorem}

\subsubsection{Lower bounds via \nameprimfn\ functions} %% NOTE: macro for name of such functions
In order to obtain lower bounds on the size of rewritings, % it will prove convenient to
we associate with each pair $(\q, \Tmc)$ a third function $\primfn$
%to each pair $\q, \Tmc$ we associate yet another Boolean function
%$f^P_{\q, \Tmc}$
that describes the result of evaluating $\q$ over data instances containing a single individual constant. %single-constant data instances.
Given an assignment $\gamma: \mathsf{sig}(\Tmc) \cup \mathsf{sig}(\q)  \to \{0,1\}$, 
%Boolean vectors $\valpha : \cn \cap (\mathsf{sig}(\Tmc) \cup \mathsf{sig}(\q))  \to \{0,1\}$ and
%$\vbeta : \rn \cap (\mathsf{sig}(\Tmc) \cup \mathsf{sig}(\q)) \to \{0,1\}$,  
we let
$$\Amc_\gamma = \{ A(a) \mid \gamma(A) = 1\} \cup \{ R(a,a) \mid \gamma(R) = 1\}$$ 
%$$\Amc(\vec{\alpha},\vec{\beta}) = \{ A(a) \mid \vec{\alpha}(A) = 1\} \cup \{ R(a,a) \mid \vec{\beta}(R) = 1\}$$ 
and set
$\primfn(\gamma) = 1$ iff $\Tmc, \Amc_\gamma \models \q(\vec{a})$, 
%$\primfn(\vec{\alpha}, \vec{\beta}) = 1$ iff $\Tmc, \Amc(\vec{\alpha}, \vec{\beta}) \models \q(\vec{a})$, 
where $\vec{a}$
is the tuple of $a$'s of the required length.
We call $f^\primsuper_{\q, \Tmc}$ the \emph{\nameprimfn\ function}
for $\q$ and $\Tmc$.

\begin{theorem}[implicit in \cite{lics14-KKPZ}]\label{rew2prim}
If $\q'$ is a \textup{(}\PE-\textup{)} \FO-rewriting of $\q$ and $\Tmc$, then
$\primfn$ is computed by a \textup{(}monotone\textup{)} Boolean formula  of size $O(|\q'|)$.

If $(\Pi, G)$ is an \NDL-rewriting of $\q$ and $\Tmc$, then $\primfn$ is computed by a
monotone Boolean circuit of size $O(|\Pi|)$.
\end{theorem}

\section{Succinctness Results for Query Rewriting}
\label{sec:succ}
In this section, we relate the upper and lower bound functions from Section \ref{sec:fns} to non-uniform models of computation,
%which enables us to transfer results form circuit complexity to the setting of query rewriting. 
which allows us to exploit results from circuit complexity to infer bounds on rewriting size. % the size of rewritings. 
As in \cite{lics14-KKPZ}, we use hypergraph programs (defined next) as a useful intermediate formalism. 

\subsection{Tree Hypergraph Programs (\THP s)}
A hypergraph takes the form $H = (V,E)$,  where $V$ is a set of \emph{vertices} and $E \subseteq 2^V$ a set of \emph{hyperedges}.
A subset $E' \subseteq E$ is \emph{independent} if $e \cap e' = \emptyset$, for any distinct $e,e' \in E'$.

A \emph{hypergraph program ($\HP$)} $P$ consists of a hypergraph $H_P = (V_P,E_P)$ and a %label
function $\l_P$  that labels every vertex with $0$, $1$, or a conjunction of literals built from a set $L_P$ of propositional 
variables. % $p_i$ or $\neg p_i$. % where the $p_i$ are propositional variables.
% (here the $p_i$ are propositional variables, distinct from the $p_v, p_e$ introduced earlier).
An \emph{input} for $P$ is a valuation of $L_P$. % the propositional variables in $P$'s labels.
The  \HP\ $P$ computes the Boolean function $f_P$ defined as follows:  
$f_P(\vec{\alpha})=1$ iff there is an independent subset
of $E$ that \emph{covers all zeros}---that is, contains every vertex in $V$ whose label evaluates to $0$ under $\vec{\alpha}$. 
%that are labelled with 0 under $\vec{\alpha}$. 
%We say that the \HP\ $P$ \emph{computes} a Boolean function $f$ in case,
%for every input $\vec{\alpha}$, we have $f(\vec{\alpha})=1$ if and only if there is an independent subset
%of $E$ that \emph{covers all zeros}---that is, contains all the vertices in $V$ that are labelled with 0
%under $\vec{\alpha}$. 
A \HP\  is \emph{monotone} if there are no negated
variables among its vertex labels. The \emph{size} $|P|$ of \HP\ $P$ is  $|V_P| + |E_P| + |L_P|$. 
In what follows, we will focus on a particular subclass of \HP s whose hyperedges correspond to subtrees of a %given 
tree. 
%The formal definition uses the notion of (generalized) interval, defined as follows.  
Formally, given a tree $T=(V_T,E_T)$ and $u,v \in V_T$, %with vertices $u$ and $v$, 
the \emph{interval} $\la u,v \ra$
is the set of edges that appear on the unique simple path connecting $u$ and $v$.
If $v_1, \ldots, v_k \in V_T$, then the \emph{generalized interval} $\la v_1, \ldots, v_k \ra$
is defined as the union of intervals $\la v_i, v_j \ra$ over all pairs $(i,j)$.
We call $v \in V_T$ a \emph{boundary vertex} for generalized interval $I$ (w.r.t.\ $T$) 
if there exist edges %if there exist $u,u' \in V_T$ such that 
$\{v,u\} \in  I$ and $\{v,u'\} \in E_T \setminus I$. 
A hypergraph $H = (V_H,E_H)$ is a \emph{tree hypergraph}\footnote{Our definition of tree hypergraph is a minor variant of the 
%closely related, though not identical, to the
notion of (sub)tree hypergraph (aka hypertree) from graph theory, cf.\  \cite{Flament1978223,Brandstadt:1999:GCS:302970,Bretto:2013:HTI:2500991}.} 
if there is a tree $T=(V_T,E_T)$ such that $V_H=E_T$ and every hyperedge in
$E_H$ is a \emph{generalized interval of $T$ all of whose boundary vertices have degree 2 in $T$}.
%A \HP\ %$H = (V,E)$
A \emph{tree hypergraph program ($\THP$)} is an \HP\ 
based on %$H$ is
a tree hypergraph. As a special case, we have \emph{interval hypergraphs} \cite{Brandstadt:1999:GCS:302970,Bretto:2013:HTI:2500991}
and \emph{interval \HP s}, whose underlying trees have %which are based upon trees having 
exactly 2 leaves. 
%\subsection{From \twfn\ to Non-deterministic Branching Programs}
%
%\subsection{From \homfn\ to $\SAC$ Circuits}
%
%\subsection{From \twfn\ and \homfn\ to Non-uniform Models of Computation}
%\vspace*{-3mm}
\subsection{Primitive Evaluation Function and \THP s}\label{ssec:prim}
\vspace{-.5mm}
Our first step will be to show how functions given by \THP s
can be computed using primitive evaluation functions.

Consider a \THP\ $P=(H_P, \l_P)$ whose underlying tree $T$ has vertices  %based upon the tree hypergraph %program %$P$
%whose hypergraph
%$H= (V,E)$. Let $T$ be the underlying  that is
%based upon the tree $T$ whose vertices are 
$v_1, \ldots, v_n$, and 
%Let $v_1, \ldots, v_n$ be the vertices of $T$, and
 let $T^{\downarrow}$ be the directed tree obtained from $T$ by fixing its leaf $v_1$ as the root and orienting
edges away from~$v_1$.
We wish to construct a tree-shaped CQ $\q_P$ and an ontology $\Tmc_P$ of depth 2
whose \nameprimfn\ function $f_{\q_P, \Tmc_P}^\primsuper$ can be used to compute $f_P$. % the function given by $P$.
%The construction generalizes the one from the preceding section for linear queries. % that encodes interval hypergraphs
%using linear queries.
%reduction from interval hypergraphs to linear queries
%from the previous section.
The query $\q_P$ is obtained by simply `doubling' the edges in $T^{\downarrow}$:
\begin{equation*}
%\q_P ~=~
\q_P = \exists \vec{y}\, \bigwedge_{(v_i, v_j) \in T^{\downarrow}} (S_{ij}(y_i, y_{ij}) \land S'_{ij}(y_{ij}, y_{j})).
\end{equation*}
%Since $H$ is a tree hypergraph, $\q_H$ is a tree-shaped query. Moreover, if $H$ is a linear hypergraph, then $\q_H$ is linear. 
%\begin{itemize}
%\item $r_{01}(y_{01}, y_1)$
%\item $r_{ij}'(y_i, y_{ij})$ and $r_{ij}(y_{ij}, y_j)$, if $v_j$ is a child of $v_i$ in $T'$
%\end{itemize}
%Essentially, $\q_H$ is obtained from $T'$ by %renaming $v_i$ to $y_i$,
%splitting every edge in $T'$ in two and using different role names for every edge.
% and adding an incoming edge to the root variable $y_1$.
%\medskip
The ontology $\Tmc_P$ is defined as the union of $\Tmc_e$ over all hyperedges $e \in E_P$.
%We define $\Tmc_P$ as the union of the TBoxes $\Tmc_e$, with $e$ a hyperedge in $E$.
%Consider some 
Let $e = \la v_{i_1}, \ldots, v_{i_m} \ra \in E_P$ with $v_{i_1}$ the vertex in $e$ that is highest in $T^{\downarrow}$,
and suppose w.l.o.g.\ that every $v_{i_j}$ is either a boundary vertex of $e$ or a leaf in $T$. %  and suppose w.l.o.g.\
%that $v_{i_1}$ is the vertex in $e$ that is highest in $T^{\downarrow}$. T
Then $\Tmc_e$ is defined as follows:
%contains the following axioms:
%contains
%$B_e \sqsubseteq \exists s_e$, $ \exists s_e^- \sqsubseteq \exists s_e'$, and the axioms:
\vspace{-1mm}
\begin{align*}
%\Tmc_e =
& \{B_e(x) \rightarrow \exists y R_e(x,y), \exists y R_e(y,x) \rightarrow \exists y R_e'(x,y)\} \ \cup \\
%B_e \sqsubseteq \exists s_e \quad  \exists s_e^- \sqsubseteq \exists s_e' & %\qquad \exists s_e^- \sqsubseteq \exists s_e'
% \\
%& \{R_e(x,y) \rightarrow P_{i_1, k}(x,y) \mid  \{v_{i_1},v_k\} \in e\}\ \cup \\
%s_e \sqsubseteq r_{i_1, k} &\quad \text{ if $\{v_{i_1},v_k\} \in e$} \\ % and $(v_{i_1}, v_k)$ in $T'$ }\\
%%\text{ if $v_k \in e$ is a child of $v_{i_1}$ in $T'$ }\\
& \{R_e(x,y) \rightarrow S_{i_1, k}(x,y) \mid \{v_{i_1},v_k\} \in e \}\ \cup \\
%s_e^- \sqsubseteq  r_{j_\ell, i_\ell}' &\quad \text{ if $1 < \ell \leq n$ and $(v_{j_\ell},v_{i_\ell}) \in T^{\downarrow}$ (so, $\{v_{j_\ell},v_{i_\ell}\} \in e$)} \\
& \{R_e(y,x) \rightarrow  S_{j_\ell, i_\ell}'(x,y) \mid 1 < \ell \leq n, (v_{j_\ell},v_{i_\ell}) \in T^{\downarrow} \}\ \cup \\
%%the parent of $v_{i_\ell}$ in $T'$ is $e_{j_\ell}$}\\
%s_e' \sqsubseteq r_{j,k}^- & \quad \text{ if $\{v_j, v_k\} \in e$, $(v_j, v_k) \in T^{\downarrow}$, and $v_j \neq v_{i_1}$}\\
%& \{R_e'(x,y) \rightarrow P_{j,k}(y,x) \mid \{v_j, v_k\} \in e, (v_j, v_k) \in T^{\downarrow}, v_j \neq v_{i_1} \}\ \cup \\
%s_e' \sqsubseteq r'_{j, k} &\quad \text{ if $\{v_j, v_k\} \in e$, $(v_j, v_k) \in T^{\downarrow}$,  and $v_k \neq v_{i_\ell}$ for any $1 < \ell \leq m$ }
& \{R_e'(x,y) \rightarrow S'_{j, k}(x,y)\mid \{v_j, v_k\} \in e, (v_j, v_k) \in T^{\downarrow}, \\
&\hspace*{3.75cm} v_k \neq v_{i_\ell} \text{ for all } 1 < \ell \leq m  \}\ \cup \\
& \{R_e'(x,y) \!\rightarrow\! S_{j,k}(y,x) \! \mid \! \{v_j, v_k\} \in e, (v_j, v_k) \in T^{\downarrow}, v_j \neq v_{i_1} \! \}. 
\end{align*}
\vspace{-.5mm}
Observe that both $\q_P$ and $\Tmc_P$ are of polynomial size in $|P|$ 
and that $\q_P$ has the same number of leaves as $T$. 
% PUT BACK IN FINAL VERSION 
\vspace{-1mm}
\begin{example}\label{thp-ex}
Consider a \THP\ $P$ whose tree hypergraph has vertices 
$\{\{v_1, v_2\}, \{v_2, v_3\}, \{v_2,v_6\}, \{v_3, v_4\}, \{v_4,v_5\}\}$
and a single hyperedge $e = \langle v_1,v_4,v_6\rangle$.
Fixing $v_1$ as root and applying the above construction, we obtain the query $\q_P$ and the canonical model $\Cmc^{R_e}_{\Tmc_P}$ depicted 
below: %in Figure \ref{thp-ex-fig}. 
\newline
{\centering
\scalebox{.85}{
\begin{tikzpicture}[>=latex, point/.style={circle,draw=black,thick,minimum size=1.5mm,inner sep=0pt,fill=white},
spoint/.style={rectangle,draw=black,thick,minimum size=1.5mm,inner sep=0pt,fill=white},
ipoint/.style={circle,draw=black,thick,minimum size=1.5mm,inner sep=0pt,fill=lightgray},
wiggly/.style={thick,decorate,decoration={snake,amplitude=0.3mm,segment length=2mm,post length=1mm}},
query/.style={thick},yscale=1,xscale=1]\footnotesize
\node[ipoint,label=above:{$y_1$}] (y1) at (1,2) {};
\node[ipoint,label=left:{$y_{12}$}] (y12) at (1,1) {};
\node[ipoint,label=left:{$y_2$}] (y2) at (1,0) {};
\node[ipoint,label=left:{$y_{26}$}] (y26) at (0,1) {};
\node[ipoint,label=above:{$y_{6}$}] (y6) at (0,2) {};
\node[ipoint,label=above:{$y_{23}$}] (y23) at (1.75,1) {};
\node[ipoint,label=right:{$y_3$}] (y3) at (2.5,0) {};
\node[ipoint,label=left:{$y_{34}$}] (y34) at (3.25,1) {};
\node[ipoint,label=above:{$y_{4}$}] (y4) at (3.25,2) {};
\node[ipoint,label=above:{$y_{45}$}] (y45) at (4,2) {};
\node[ipoint,label=above:{$y_{5}$}] (y5) at (4.75,2) {};
\node at (-.3, 0){$\q_e$};
\node at (4.7, 0){$\q_P$};
%%\node[ipoint,label=right:{$y_6$}] (y6) at (2.5,0) {};
%
\draw[ultra thin,dashed,rounded corners=8] (-0.7,-0.25) rectangle +(4.15,2.75);
\draw[->,query] (y1) to node[left] {\scriptsize$S_{12}$} (y12);
\draw[->,query] (y12) to node[left,near start] {\scriptsize$S_{12}'$} (y2);
\draw[->,query] (y2) to node[left] {\scriptsize$S_{26}$} (y26);
\draw[->,query] (y2) to node[right, near start] {\scriptsize$S_{23}$} (y23);
\draw[->,query] (y26) to node[left] {\scriptsize$S_{26}'$} (y6);
\draw[->,query] (y23) to node[right, near start] {\scriptsize$S_{23'}$} (y3);
\draw[->,query] (y3) to node[right] {\scriptsize$S_{34}$} (y34);
\draw[->,query] (y34) to node[left, near end] {\scriptsize$S_{34}'$} (y4);
\draw[->,query] (y4) to node[below] {\scriptsize$S_{45}$} (y45);
\draw[->,query] (y45) to node[below] {\scriptsize$S_{45}'$} (y5);
%\draw[->,query] (y23) to node[right] {\scriptsize$S_{23}'$} (y3);
%\draw[->,query] (y24) to node[left] {\scriptsize$S_{24}'$} (y4);
%\draw[->,query] (y1) to node[above,sloped] {$S$} (y3);
%\draw[->,query] (y2) to node[left] {$R$} (x1);
%\draw[->,query] (y4) to node[pos=0.4,above,sloped] {$S$} (y3);
%\draw[->,query] (y5) to node[pos=0.4,above,sloped] {$T$} (y3);
%\draw[->,query] (y4) to node[left] {$U$} (x2);
\end{tikzpicture}
}
\,\,
\scalebox{.85}{
\begin{tikzpicture}[>=latex, point/.style={circle,draw=black,thick,minimum size=1.5mm,inner sep=0pt}, wiggly/.style={thick,decorate,decoration={snake,amplitude=0.3mm,segment length=2mm,post length=1mm}},
query/.style={thick},
tw/.style={shorten <= 0.1cm, shorten >= 0.1cm,dashed},yscale=1,xscale=0.9]\footnotesize
\node[fill=black] (a) at (0,3) [point, label=right:{$A_{R_e}$}, label=left:{$a$}]{};
%\node[fill=black] (c) at (2.5,3) [point, label=right:{$c$}]{};
\node (d1) at (0,1.75) [point, fill=white]{};
\node (d2) at (0,0.5) [point, fill=white]{};
\node at (3, 3){$\Cmc^{R_e}_{\Tmc_P}$};
%\node (d31) at (0.5,0) [point, fill=white, label=left:{$aPSR$}]{};
%\node (d32) at (1.5,0) [point, fill=white, label=right:{$aPST^-$}]{};
%\node at (0, 0.25){};
%\draw[->,query] (a) to node[above] {$R$} (c);
\draw[->,wiggly] (a)  to node [right]{\scriptsize $R_e, S_{12}, S_{34}'^-,S_{26}'^-$} (d1);
\draw[->,wiggly] (d1)  to node [right]{\scriptsize $R_e', S_{12}', S_{23}^-,S_{23}',S_{34}^-, S_{26}^-$} (d2);
%\draw[->,wiggly] (d2)  to node [above,sloped]{\scriptsize $R$} (d31);
%\draw[->,wiggly] (d2)  to node [above,sloped]{\scriptsize $T^-$} (d32);
\end{tikzpicture}}
\newline
} 
\noindent
Observe how the axioms in $\Tmc_P$ ensure that the subquery $\q_e$ induced by the hyperedge $e$ %of $\q_P$ lying between $y_1$, $y_4$, and $y_6$ can be mapped 
maps into $\Cmc^{R_e}_{\Tmc_P}$.
%\noindent Observe how the axioms in $\Tmc_p$ ensure that the subquery of $\q_P$ lying between $y_1$, $y_4$, and $y_6$ can be mapped into $\Cmc^{R_e}_{\Tmc_P}$. 
\end{example}

The next theorem specifies how $f_P$ can be computed using % \todo{subfunction} of 
the \nameprimfn\ function $f_{\q_P, \Tmc_P}^\primsuper$.

\begin{theorem}\label{tree-hg-to-query}
%Consider a 
Let $P=(H_P, \l_P)$ be a \THP. % computes~$g$. 
For every input $\alpha$ for $P$, % $\valpha: V \to \{0,1\}$ and $\vbeta: E \to \{0,1\}$,
 $f_P(\alpha) = 1$ iff \mbox{$f_{\q_P,\Tmc_P}^\primsuper(\vgamma)= 1$}, %$f_{\q_P, \Tmc_P}^{hom}(\gamma)=1$, 
where $\vgamma$ is defined as follows: 
$\vgamma(B_e) = 1$, $\vgamma(R_e)=\vgamma(R_e')=0$, % for all $e \in E$,
and $\vgamma(S_{ij})=\vgamma(S_{ij}')= \valpha(\l_P(\{v_i,v_j\}))$. 
%\begin{itemize}
%\item $\gamma(B_e) = \vbeta(e)$
%\item $\gamma(r_{ij})=\gamma(r_{ij}')= \valpha(\{v_i,v_j\})$
%\end{itemize}
\end{theorem}

\subsection{Bounded Treewidth Queries, \THP s, and $\SAC$}
We next show how the modified tree witness functions associated
with bounded treewidth queries and bounded depth ontologies 
can be computed using \THP s. We then relate \THP s to the 
non-uniform complexity class  $\SAC$. 

\smallskip

%We begin by showing how to construct a $\THP$ that computes $\homfn$,
Suppose we are given a TBox $\Tmc$ of depth $d$,
a CQ $\q$, and a tree decomposition $(T, \lambda)$
of $G_\q$ of width~$\twidth$, and we wish to define a $\THP$ that computes $\homfn$. 
% NOTE: add next sentence to proof in appendiw!
%We may suppose w.l.o.g.\ that $T$ contains
%at most $(2|q|-1)^2$ nodes, cf.\ \cite{books/sp/Kloks94}.
%Consider a TBox $\Tmc$, % of depth $d$,
%a CQ $\q$, and a tree decomposition $(T, \lambda)$
%of $G_\q$ of width $\twidth$.
In order to more easily refer to the variables in $\lambda(N)$,
we construct functions $\lambda_1, \ldots, \lambda_\twidth$
such that $\lambda_i(N) \in \lambda(N)$ and $\lambda(N) = \{\lambda_i(N) \mid 1 \leq i \leq \twidth\}$.

The basic idea underlying the construction of the $\THP$ 
is as follows: for each node $N$ in $T$, %the tree decomposition of $\q$, 
we select a data-independent %abstract 
description of the way
the variables in $\lambda(N)$ are homomorphically mapped into the canonical model. 
These % `abstract'
 descriptions %of partial homomorphisms
are given by tuples from $W_d^t=\{(w_1, \ldots, w_\twidth)\mid w_i \in (\rni \cap \sig(\Tmc))^*, |w_i|\leq d  \}$,
where the $i$th word $\vec{w}[i]$ of tuple $\vec{w} \in W_d^t $ 
indicates that variable $\lambda_i(N)$ is mapped to an element of the form $a \, \vec{w}[i]$. 
The tuple assigned to node $N$ must be compatible with the restriction of $\q$ to $\lambda(N)$
and with the tuples of neighbouring nodes. 
Formally, we say %that
 $\vec{w} %, %(w_1, \ldots, w_b),
%\sim)
\in W_d^t$ is  \emph{compatible with node} $N$ %in $T$
 if the following conditions hold:
\begin{itemize}
%\item if $\lambda_i(N)=\lambda_j(N)$, then $\vec{w}[i]=\vec{w}[j]$ and $i \sim j$
\item if $A(\lambda_i(N)) \in \q$ and $\vec{w}[i] \neq \emptyword$, then 
$\vec{w}[i]= w' \varrho$ for some $\varrho \in \rni$ with $\Tmc \models \exists y\, \varrho(y,x) \rightarrow A(x)$
\item if $R(\lambda_i(N),\lambda_j(N)) \in \q$, then one of the following holds:
\begin{itemize}
\item $\vec{w}[i] = \vec{w}[j] = \emptyword$
\item $\vec{w}[j]= \vec{w}[i] \cdot \varrho$ with $\Tmc \models \varrho(x,y) \rightarrow R(x,y)$ %and $i \sim j$
\item $\vec{w}[i] = \vec{w}[j] \cdot \varrho$ with $\Tmc \models \varrho(x,y) \rightarrow R(y,x)$ % and $i \sim j$
\end{itemize}
\end{itemize}
%\begin{itemize}
%\item if $x_i \in \avars(\q)$, then $\vec{w}[i] = \emptyword$
%\item if $A(x_i) \in \q$ and $\vec{w}[i]$ ends by $R$, then $\exists R^- \sqsubseteq_\Tmc A$
%\item if $r(x_i, x_j) \in \q$ and either $\vec{w}[i] \neq \emptyword$ or $\vec{w}[j] \neq \emptyword$, then $i \sim j$ and one of the following holds:
%\begin{itemize}
%\item $\vec{w}[j] = \vec{w}[i] S $ and $S \sqsubseteq_\Tmc r$
%\item $\vec{w}[i]  = \vec{w}[j] S$ and $S \sqsubseteq_\Tmc r^-$
%\end{itemize}
%\end{itemize}
%
%
A pair $(\vec{w}, \vec{w}')$ %$(\vec{w}_1, \vec{w}_2)$ %from  $$W_d$$
is \emph{compatible with
the pair of nodes} $(N, N')$ if $\lambda_i(N)= \lambda_j(N')$ implies that $\vec{w}[i] = \vec{w}'[j]$.
%$(N_1, N_2)$ if $\lambda_i(N_1)= \lambda_j(N_2)$ implies that $\vec{w}_1[i] = \vec{w}_2[j]$.
%\begin{itemize}
%%\item $N_1$ and $N_2$ are adjacent in $T$
%%\item $(\vec{w}_1, \sim_1)$ is consistent with $N_1$ and $(\vec{w}_2, \sim_2)$ is consistent with $N_2$
%\item $\lambda_i(N_1)= \lambda_j(N_2)$ implies that $\vec{w}_1[i] = \vec{w}_2[j]$
%%\item if $\lambda_i(N_1)= \lambda_j(N_2)$ and $\lambda_{i'}(N_1)= \lambda_{j'}(N_2)$,
%%then $i \sim_1 i'$ iff $j \sim_2 j'$
%\end{itemize}

Let $W_d^t = \{\vec{w}_1, \ldots, \vec{w}_\numtypes\}$, and consider
the tree $T'$ obtained from $T$ by replacing every edge $\{N_i, N_j\}$ by the edges:
\begin{tikzpicture}[>=latex, point/.style={circle,draw=black,thick,minimum size=1.5mm,inner sep=0pt,fill=white},
spoint/.style={rectangle,draw=black,thick,minimum size=1.5mm,inner sep=0pt,fill=white},
ipoint/.style={circle,draw=black,thick,minimum size=1mm,inner sep=0pt,fill=black},
wiggly/.style={thick,decorate,decoration={snake,amplitude=0.3mm,segment length=2mm,post length=1mm}},
query/.style={thick},yscale=1,xscale=1]\footnotesize
\node[ipoint,label=left:{$N_i$}] (Ni) at (0,0) {};
\node[ipoint,label=above:{$u_{ij}^1$}] (uij1) at (0.5,0) {};
\node[ipoint,label=below:{$v_{ij}^1$}] (vij1) at (1,0) {};
\node[ipoint,label=above:{$u_{ij}^2$}] (uij2) at (1.5,0) {};
\node[ipoint,label=below:{$v_{ij}^2$}] (vij2) at (2,0) {};
\node (d1) at (2.5,0) {$\ldots$};
\node[ipoint,label=above:{$u_{ij}^M$}] (uijm) at (3,0) {};
\node[ipoint,label=below:{$v_{ij}^M$}] (vijm) at (3.5,0) {};
\node[ipoint,label=below:{$v_{ji}^M$}] (vjim) at (4,0) {};
\node[ipoint,label=above:{$u_{ji}^M$}] (ujim) at (4.5,0) {};
\node (d1) at (5,0) {$\ldots$};
\node[ipoint,label=below:{$v_{ji}^2$}] (vji2) at (5.5,0) {};
\node[ipoint,label=above:{$u_{ji}^2$}] (uji2) at (6,0) {};
\node[ipoint,label=below:{$v_{ji}^1$}] (vji1) at (6.5,0) {};
\node[ipoint,label=above:{$u_{ji}^1$}] (uji1) at (7,0) {};
\node[ipoint,label=right:{$N_j$}] (Nj) at (7.5,0) {};
%%\node[ipoint,label=right:{$y_6$}] (y6) at (2.5,0) {};
%
\draw (Ni) to (uij1);
\draw (uij1) to (vij1);
\draw (vij1) to (uij2);
\draw (uij2) to (vij2);
\draw (uijm) to (vijm);
\draw (vijm) to (vjim);
\draw (vjim) to (ujim);
\draw (vji2) to (uji2);
\draw (uji2) to (vji1);
\draw (vji1) to (uji1);
\draw (uji1) to (Nj);
\end{tikzpicture}
%}
%\begin{align*}
%%\{N_i, u_{ij}^1\}
%& \{N_i, u_{ij}^1\},  \{u_{ij}^1,v_{ij}^1\},  \{v_{ij}^1, u_{ij}^2\},  \{u_{ij}^2, v_{ij}^2\}, \ldots, \{u_{ij}^\numtypes, v_{ij}^\numtypes\}, \\ %\{v_{ij}^\numtypes, v_{ji}^\numtypes\}  \\%M_{ij} M_{ji}
%& \{v_{ij}^\numtypes, v_{ji}^\numtypes\}, \{u_{ji}^\numtypes ,v_{ji}^\numtypes\}, \ldots, %\{u_{ji}^2, v_{ji}^2\}, 
%\{v_{ji}^1, u_{ji}^2\}, \{u_{ji}^1,v_{ji}^1\}, \{N_j, u_{ji}^1\}.
%%  \{u_{ji}^1,v_{ji}^1\},  \{v_{ji}^1, u_{ji}^2\},  \{u_{ji}^2, v_{ji}^2\}, \ldots \{u_{ji}^\numtypes ,v_{ji}^\numtypes\}
%\end{align*}
%
The desired $\THP$ $(H_{\q, \Tmc}, \l_{\q, \Tmc})$ is based upon $T'$
and contains the following hyperedges:
\begin{itemize}
\item $E_i^k = \la u_{ij_1}^k, \ldots, u_{i j_n}^k\ra $, if $\vec{w}_k \in W_d^t$ is compatible with $N_i$ and $N_{j_1}, \ldots, N_{j_n}$ are the neighbours of $N_i$;
\item $E_{ij}^{k m} = \la v_{ij}^k, v_{ji}^m\ra $, if $\{N_i, N_j\}$ is an edge in $T$ and $(\vec{w}_k, \vec{w}_m)$ is compatible with $(N_i, N_j)$.
\end{itemize}
Intuitively, $E^k_{i}$ corresponds to assigning (compatible) tuple $\vec{w}_k$ to node $N_i$, and
hyperedges of the form $E_{ij}^{k m}$ are used to ensure compatibility of choices at neighbouring nodes. 
Vertices of $H_{\q, \Tmc}$ % the hypergraph 
(i.e.\ the edges in $T'$) are labeled by $\l_{\q, \Tmc}$ as follows:
edges of the forms $\{N_i, u_{ij}^1\}$, $\{v_{ij}^\ell, u_{ij}^{\ell +1}\}$, and $\{v_{ij}^\numtypes, v_{ji}^\numtypes\}$ are labelled $0$,
and every edge $\{u_{ij}^\ell, v_{ij}^\ell\}$ %with $\xi_\ell = \vec{w}$ 
is labelled by the conjunction of the following variables:
% to ensure that the assignment $()$ is compatible with neighbouring nodes $$.  
%  %
%Vertices of $H_{\q, \Tmc}$ % the hypergraph 
%(i.e.\ the edges in $T'$) are labeled by $\l_{\q, \Tmc}$ as follows:
%every edge of the form $\{N_i, u_{ij}^1\}$, $\{v_{ij}^\ell, u_{ij}^{\ell +1}\}$, or $\{v_{ij}^\numtypes, v_{ji}^\numtypes\}$  is labelled $0$,
%and every edge $\{u_{ij}^\ell, v_{ij}^\ell\}$ %with $\xi_\ell = \vec{w}$ 
%is labelled by the conjunction of the following variables:
%   %
%\begin{itemize}
%\item $p_\atom$, if  %$\alpha \in \q$,
%$\vars(\atom) \subseteq \lambda(N_i)$
%and $\lambda_g(N_i) \in \vars(\atom)$ implies $\vec{w}[g]= \emptyword$
%\item $p_z^R$, if $\vars(\atom) = \{z\} \subseteq \lambda(N_i)$, $z = \lambda_g(N_i)$, and $\vec{w}[g]= R w'$
%\item $p_z^R$, $p_{z'}^R$, and $p_{z=z'}$, if $\vars(\atom) = \{z, z'\} \subseteq \lambda(N_i)$, $z = \lambda_g(N_i)$, $z' = \lambda_{g'}(N_i)$,
%and either $\vec{w}[g]= R w'$ or $\vec{w}[g']= R w'$
%\begin{itemize}
%\item every edge of the form $\{N_i, u_{ij}^1\}$, $\{v_{ij}^\ell, u_{ij}^{\ell +1}\}$, or $\{v_{ij}^\numtypes, v_{ji}^\numtypes\}$  is labelled $0$
%\item every edge $\{u_{ij}^\ell, v_{ij}^\ell\}$ with $\xi_\ell = \vec{w}$ is labelled by the
%conjunction of: % the following variables}:
\begin{itemize}
\item $p_\atom$, if  $\atom \in \q$,
$\vars(\atom) \subseteq \lambda(N_i)$,
and $\lambda_g(N_i) \in \vars(\atom)$ implies $\vec{w}_\ell[g]= \emptyword$;
\item $p_z^\varrho$, if $\vars(\atom) = \{z\}$, % \subseteq \lambda(N_i)$, 
$z = \lambda_g(N_i)$, and $\vec{w}_\ell[g]= \varrho w'$;
\item $p_z^\varrho$, $p_{z'}^\varrho$, and $p_{z=z'}$, if $\vars(\atom) = \{z, z'\}$, % \subseteq \lambda(N_i)$, 
$z = \lambda_g(N_i)$, $z' = \lambda_{g'}(N_i)$,
and either $\vec{w}_\ell[g]= \varrho w'$ or $\vec{w}_\ell[g']= \varrho w'$.
%\item $p_{h=h'}$, if $x_h = \lambda_g(N_i)$, $x_{h'} = \lambda_{g'}(N_i)$, and $g \sim g'$
\end{itemize}
We prove in the appendix that the $\THP$ $(H_{\q, \Tmc}, \l_{\q, \Tmc})$ computes $\homfn$,
which allows us to establish the following result:

%\begin{theorem}\label{DL2THP}
%For every ontology $\Tmc$ and CQ $\q$,  the $\THP$ $(H_{\q, \Tmc}, \l_{\q, \Tmc})$ computes
%$\homfn$.
%% There is a polynomial $p$ such that $|H_{\q, \Tmc}| \leq 2^{2\twidth^2+3} \cdot  |q|^2 \cdot |\tqwords|^{2 \twidth}$,
%%where $\twidth$ is the treewidth of $\q$.
%If $\q$ has treewidth $\twidth$, then $|H_{\q, \Tmc}| \leq 8 \cdot  |q|^2 \cdot |\tqwords|^{2 \twidth}$.
%%
%\end{theorem}
%
%Use next version if skip details of construction:
\vspace{-0.7mm}
\begin{theorem}\label{DL2THP}
Fix $\twidth \geq 1$ and $d \geq 0$. For every ontology $\Tmc$ of depth $\leq d$ 
and CQ $\q$ of treewidth $\leq \twidth$, $\homfn$ is computed by 
a monotone $\THP$ of size polynomial in $|\Tmc| + |\q|$. 
\end{theorem}
\vspace{-0.7mm}

To characterize % the expressive power of 
tree hypergraph programs, we
consider \emph{semi-unbounded fan-in} circuits
in which $\NOT$ gates are applied only to the inputs, $\AND$ gates have fan-in 2, and $\OR$ gates have unbounded fan-in.
The complexity class $\SAC$~\cite{vollmer99} is defined by considering circuits of this type having polynomial size and logarithmic depth.
$\SAC$ is the non-uniform analog of the class \LOGCFL \cite{DBLP:journals/jcss/Venkateswaran91}, which will play a central role %  of all languages logspace-reducible to context-free languages \cite{DBLP:journals/jcss/Venkateswaran91}.
%(\LOGCFL\ plays a central role 
in our complexity analysis in Section \ref{sec:complexity}. 

%The circuits of this type of polynomial size and logarithmic depth are used to define the complexity class
%$\SAC$, the nonuniform analog of the class $LOGCFL$ of all languages logspace-reducible to context-free languages \todo{add ref}.

We consider semi-unbounded fan-in circuits of size $\size$ and depth $\log \size$, where $\size$ is a parameter,
and show that they are polynomially equivalent to %capture the power of
$\THP$s by providing
%For this, we give %the following two lemmas giving the
reductions in both directions (details can be found in the appendix).
\vspace{-0.7mm}
\begin{theorem} \label{thm:thp_vs_sac}
There exist  %n explicit
polynomials %functions 
$p, p'$ such that:
\begin{itemize}
\item Every function computed by a semi-unbounded fan-in %(monotone)
  circuit of size at most $\size$
and depth at most $\log \size$ is computable by a %(monotone) 
$\THP$ of size $p(\size)$.
\item Every function computed by a %(monotone) 
$\THP$ of size $\size$ is computable by a semi-unbounded fan-in %(monotone) 
circuit of size at most $p'(\size)$
and depth at most $\log p'(\size)$.
\end{itemize}
Both reductions preserve monotonicity. 
\end{theorem}
%\begin{proof}
%Observe that if $G_\q$ is a tree, % with $\ell$ leaves, % $\q$ is tree-shaped, 
%then $H^{\q}_\Tmc$ is a tree hypergraph based upon $G_\q$. 
%By bounding the number of leaves, % Next observe that bounding the number of leaves in 
%%a tree-shaped query 
%we ensure there are only polynomially many tree witnesses, and hence that $|H^{\q}_\Tmc|$ is of polynomial size. %there are polynomially many hyperedges in $H^{\q}_\Tmc$. % tree witnesses (hence hyperedges in $H^{\q}_\Tmc$). 
%It then suffices to apply Proposition \ref{HF-HGP} to obtain a polysize \THP\ computing $\twfn$. \todo{Hmm... we never say that the transformation preserves structure of hypergraph} \end{proof}
%For arbitrary tree-shaped queries and bounded treewidth queries, the number of tree witnesses may be exponential,
%and we must therefore use $\homfn$ in place of $\twfn$. 
\vspace{-2mm}
\subsection{Bounded Leaf Queries, Linear \THP s, \& NBPs}
\label{ssec:linear}

%\subsection{\THP s and Non-uniform Models of Computation}
%Our next objective is to establish tight connections between \THP s and %two well-known 
%non-uniform models of computation, in order to be able to exploit results from complexity theory. 
%%non-deterministic branching programs and semi-unbounded fan-in circuits. 

\begin{figure*}[t]\label{fig-NL-Hypergraph}
\setlength{\abovecaptionskip}{3pt plus 0pt minus 0pt}
\scalebox{.97}{
\includegraphics{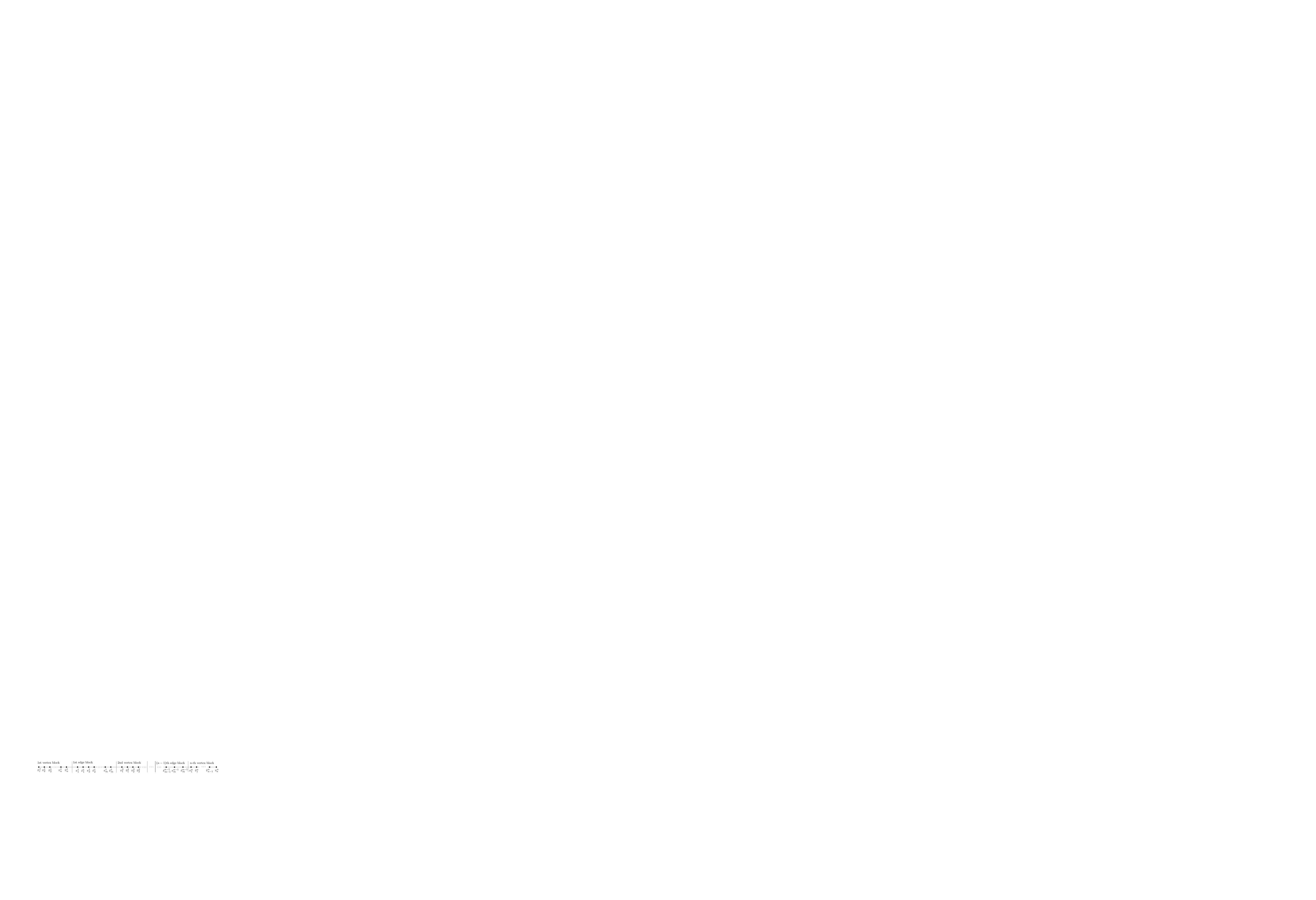}
}
\caption{The graph $G$ underlying the interval hypergraph program from the proof of Theorem \ref{nbp-to-ihp}. %hypergraph program computing reachability. 
}
\label{fig:interval-hypergraph}
\vspace*{-.25cm}
\end{figure*}

For bounded leaf queries, we establish a tight connection to 
\emph{non-deterministic branching programs} (NBPs), a well-known %model for the 
representation of Boolean functions situated between Boolean formulas and Boolean circuits \cite{Razborov91,Jukna12}. 
We recall that an %\emph{non-deterministic branching program} (NBP) % \emph{non-deterministic branching program} (NBP)
NBP
is defined as a tuple $P = (V_P, E_P, s, t, \l_P)$, where $(V_P, E_P)$ is a directed graph,
$s,t \in V_P$, and $\l_P$ is a function that labels every edge  $e \in E_P$ with 0, 1, or a conjunction
of propositional literals built from $L_P$. 
%Given a valuation $\vec{\alpha}$ of the propositional variables in $P$, the label on
%every edge is evaluated into 0 or 1.
%Such a program
%The program
The NBP $P$ induces the function
$f_P$ defined as follows:  for every valuation $\valpha$ of the variables %the propositional variables in 
$L_P$,
$f_P(\vec{\alpha}) = 1$ iff there is a path from $s$ to $t$ in the graph $(V_P,E_P)$ 
such that all labels along the path evaluate to 1 under $\valpha$. The size $|P|$ of $P$ 
is $|V_P|+|E_P|+|L_P|$. An NBP is \emph{monotone} if neither of its labels contains negation.

The next theorem shows that tree witness functions of bounded leaf queries 
can be captured by polysize NBPs. %to NBPs, then NBPs
%to interval \THP s. 
%provide polynomial translations between  
%NBPs, linear \HP s, and bounded-leaf \THP s. % with a bounded number of leaves.
\vspace{-0.7mm}
\begin{theorem}\label{bbthp-to-nbp}
Fix $\ell \geq 2$. 
For every  ontology $\Tmc$ and tree-shaped CQ $\q$ with at most $\ell$ leaves, 
the function $\twfn$ is 
%The tree witness function of an ontology $\Tmc$ and CQ $\q$ with at most $\ell$ leaves is 
computable by a monotone NBP of size polynomial in $|\q|$ and $|\Tmc|$. % for a fixed constant $\ell$.
\end{theorem}
\vspace{-0.7mm}
\begin{proof}
%Fix a constant $\ell \geq 2$. 
Consider an ontology $\Tmc$, a tree-shaped CQ $\q$ with $\ell$ leaves, and its associated graph $G_\q= (V_\q, E_\q)$.
%Let $G_\q = (V_\q, E_\q)$ be the graph associated with $\q$, and 
For every tree witness $\t \in \twset$, let $(V_\t, E_\t)$ be the graph associated with $\q_\t$, 
and for every subset $\Theta \subseteq \twset$, let $V_\Theta = \bigcup_{\t \in \Theta} V_\t$ and $E_\Theta = \bigcup_{\t \in \Theta} E_\t$. 
Pick some vertex $v_0 \in V_\q$ and call an independent subset $\Theta \subseteq \twset$ \emph{flat}
if every simple path in $G_\q$ with endpoint $v_0$ intersects at most one of the sets $E_\t$, $\t \in \Theta$. 
Note that every flat subset of $\twset$ can contain at most $\ell$ tree witnesses, 
so the number of flat subsets is polynomially bounded in $|\q|$, when $\ell$ is a fixed constant.
Flat subsets can be partially ordered as follows: $\Theta \prec\Theta'$ if every simple path between $v_0$ 
and a vertex $v' \in V_{\Theta'}$ intersects $E_\Theta$. 

The required NBP $P$ 
is based upon the graph $G_P$ with vertices %=(V_P,E_P)$, where 
% = (V_P, E_P, s, t, \l_P)$ by taking  
$V_P = \{u_\Theta, v_\Theta \mid \Theta \subseteq \twset \text{ is flat}\} \cup\{s,t\}$ and 
%edges 
$E_P=\{(s, u_\Theta), (v_\Theta, t), (u_\Theta, v_\Theta) \mid \Theta % \subseteq \twset
 \text{ flat}\} 
\cup \{(v_\Theta, u_{\Theta'}) \mid \text{ flat } 	\Theta \prec \Theta'\}$. 
We label $(u_\Theta, v_\Theta)$ with $\bigwedge_{\t\in\Theta}p_\t$ and label 
edges $(s, u_\Theta)$, $(v_\Theta, t)$, and $(v_\Theta, u_\Theta)$ % are labelled
by conjunctions of variables $p_\atom$ ($\atom \in \q$) corresponding respectively to the 
atoms in $\q$ that occur `before' $\Theta$, `after' $\Theta$,
and `between' $\Theta$ and~$\Theta'$. % (see appendix for formal details). 
In the appendix, we detail the construction and show 
that %in the appendix
%We prove in the appendix that for every valuation $\alpha$, % of the variables in $\twfn$,
$f^{\mathsf{tw}}_{\q,\Tmc}(\alpha) = 1$ iff %just in the case that 
there is a path from $s$ to 
$t$ in $G_P$ all of whose %edge
labels evaluate to $1$ under~$\alpha$.  % (see appendix for details). 
\end{proof}

NBPs in turn can be translated into polysize interval \HP s. %A polynomial translation from NBPs to 
%interval hypergraph programs.

\begin{theorem}\label{nbp-to-ihp}
Every function %$f$ 
that is computed by a %(monotone) 
NBP $P$ is %also 
computed by an %(monotone) 
interval \HP\ of size %whose %hypergraph program %based on an interval hypergraph $H=(V_H, E_H)$
%size is 
polynomial in $|P|$.
The reduction preserves monotonicity.
% (and so by a primitive evaluation function of a linear query).
\end{theorem}

\begin{proof}
Consider an NBP $P = (V_P, E_P, v_1, v_n, \l_P)$, where $V_P= \{v_1, \ldots, v_n\}$
and $E_P= \{e_1, \ldots, e_m\}$. We may assume w.l.o.g.\ that $e_m=(v_n,v_n)$ and $\l_P(e_m)=1$. 
%It follows 
This assumption ensures that if there is a path from $v_1$ to $v_n$ whose labels evaluate to $1$, then there is % we can construct 
a (possibly non-simple) path with the same properties %and
 whose length is exactly $n-1$. 

We now construct an interval \HP\ $(H, \l_H)$ %based on the hypergraph $H=(V_H, E_H)$ %with hypergraph $H=(V_H, E_H)$ and labelling function $\l_H$ 
that computes the function $f_P$. In Figure~\ref{fig:interval-hypergraph}, we display the graph %We start by describing the %path 
%graph 
$G=(V_G,E_G)$ that underlies the interval hypergraph $H$. %the construction. 
Its vertices are arranged into $n$ \emph{vertex blocks} and $n-1$ \emph{edge blocks} which alternate. 
The $\ell$th vertex block (resp.\ edge block) contains two copies, $v_i^\ell, \bar v_i^\ell$ (resp.\ $e_i^\ell, \bar e_i^\ell$), 
of every vertex $v_i \in V_P$
(resp.\ edge $e_i \in E_P$). %, except for the first and last vertex blocks from which we remove $v^1_1$ and $\bar v_n^n$.
We remove the first and last vertices $v^1_1$ and $\bar v_n^n$ and 
connect the remaining vertices as shown in Figure~\ref{fig:interval-hypergraph}. 
%The remaining points are linearly ordered as in 
The hypergraph $H=(V_H,E_H)$ is defined by setting $V_H=E_G$ and
letting $E_H$ be the set of all hyperedges 
%including in $E_H$ the hyperedge 
$\zeta_{i,\ell} = \langle \bar v^\ell_j, e^\ell_i \rangle$ and $\zeta'_{i,\ell} = \langle  {\bar e}^\ell_i, v^{\ell+1}_k \rangle$
%for every
where $e_i =(v_j,v_k) \in E_P$ and $1 \leq \ell < n$. 
%for every $e_i =(v_j,v_k)$ and $1 \leq \ell < n$ the hyperedge 
%$\epsilon_{i,\ell} = \langle v^\ell_j, e^\ell_i \rangle \cup \langle  {\bar e}^\ell_i, v^{\ell+1}_k \rangle$. %has hyperedges 
%$$E_H= \{\langle v^h_j, e^h_i \rangle \cup \langle  {\bar e}^h_i, v^{h+1}_k \rangle \mid e_i =(v_j,v_k), 1 \leq h < n\}.$$
The function $\l_H$ labels $\{e^\ell_i,\bar e^\ell_i\}$ with $\l_P(e_i)$ and all other vertices of $H$ (i.e.\ edges of $G$)
 with $0$.

We claim that %the \HP\
 $(H,\l_H)$ computes $f_P$. Indeed, if $f_P(\alpha)=1$,
then there is a path $e_{j_1}, e_{j_2}, \ldots, e_{j_{n-1}}$
from $v_1$ to $v_n$ whose labels evaluate to $1$ under $\alpha$. 
It follows that $E' = \{\zeta_{j_\ell, \ell}, \zeta'_{j_\ell, \ell} \mid 1 \leq \ell< n\}$ 
is an independent subset of $E_H$ that covers all zeros. 
Conversely, if $E' \subseteq E_H$ is independent and covers all zeros under $\alpha$,
then it must contain exactly one pair of hyperedges $\zeta_{j_\ell,\ell}$ and $\zeta'_{j_\ell,\ell}$ 
for every $1 \leq \ell <n$,
and the corresponding sequence of edges $e_{j_1}, \ldots, e_{j_{n-1}}$ defines a path from $v_1$ to $v_n$. 
Moreover, since $E'$ does not cover $\{e^\ell_{j_\ell},\bar e^\ell_{j_\ell}\}$, we know that 
$\l_H(\{e^\ell_{j_\ell},\bar e^\ell_{j_\ell}\}) =\l_P(e_{j_\ell})$ evaluates to $1$ under $\alpha$, for every $1 \leq \ell <n$. 
\end{proof}
%
%\begin{remark}
%%Using a similar argument to Theorem \ref{bbthp-to-nbp}, one can provide a polynomial translation from 
%%bounded leaf \THP s to NBPs, and hence also to linear \THP s (via Theorem \ref{nbp-to-ihp}). 
%We gave a direct translation from tree witness functions to NBPs, without passing by \THP s. 
%However, using a similar argument to Theorem \ref{bbthp-to-nbp}, one can provide a polynomial translation from 
%bounded leaf \THP s to NBPs, and hence also to linear \THP s (via Theorem \ref{nbp-to-ihp}). 
%\end{remark}

\subsection{Succinctness Results}
We now combine the correspondences from the preceding subsections with results from circuit complexity 
to derive upper and lower bounds on rewriting size for tree-like queries. 

\smallskip %\todo{Text, organization}
We start with what is probably our most surprising result: a super-polynomial lower bound on the size of 
PE-rewritings of linear queries and depth-2 ontologies. %To obtain this strong negative result,
This result significantly improves upon earlier negative results for PE-rewritings \cite{DBLP:conf/icalp/KikotKPZ12,lics14-KKPZ}, 
which required either arbitrary queries or arbitrary ontologies. 
The proof utilizes Theorems~ \ref{rew2prim}, \ref{tree-hg-to-query}  and \ref{nbp-to-ihp} 
and the well-known circuit complexity result 
that there is a sequence $f_n$ of monotone Boolean functions that are 
computable by polynomial-size monotone NBPs, but all monotone Boolean formulas computing $f_n$ are of size $n^{\Omega(\log n)}$
\cite{DBLP:conf/stoc/KarchmerW88}.

%It is obtained 
%%This strong negative result is obtained 
%by combining 
%%we combine 
%Theorems~\ref{TW2rew}, \ref{rew2prim}, \ref{nbp-to-ihp}, and \ref{tree-hg-to-query} 
%%with the fact (well known in the circuit complexity)  
%with the well-known circuit complexity result 
%that there is a sequence $f_n$ of monotone Boolean functions that are 
%computable by polynomial-size monotone NBPs, but all monotone Boolean formulas computing $f_n$ are of size $n^{\Omega(\log n)}$
%\cite{DBLP:conf/stoc/KarchmerW88}. 

%Probably most surprising is  strong negative result for PE-rewritings. 
%It is known that there is a sequence $f_n$ of monotone Boolean functions that are computable by polynomial-size 
%monotone NBPs, but all monotone Boolean formulas computing $f_n$ are of size $n^{\Omega(\log n)}$
%\cite{DBLP:conf/stoc/KarchmerW88}. % \cite{GrigniSipser95}.
%Using this fact, together with Theorems~\ref{TW2rew}, \ref{rew2prim}, \ref{nbp-to-ihp}, and \ref{tree-hg-to-query},  
%we obtain a strong negative result for PE-rewritings. 

\begin{theorem}\label{linear-lower}
There is a sequence of linear CQs $\q_n$ and ontologies $\Tmc_n$ of depth 2, both of polysize in $n$, 
such that any PE-rewriting of $\q_n$ and $\Tmc_n$ is of size $n^{\Omega(\log n)}$. 
\end{theorem}

%Theorem \ref{linear-lower} significantly improves upon earlier negative results for PE-rewritings \cite{DBLP:conf/icalp/KikotKPZ12,lics14-KKPZ}, 
%which required either arbitrary queries or arbitrary ontologies. 
%
%\smallskip

We obtain a positive result for NDL-rewritings of bounded-leaf queries using Theorems \ref{TW2rew} and \ref{bbthp-to-nbp} and the fact that 
NBPs are representable as polynomial-size monotone circuits \cite{Razborov91}. 

\begin{theorem}\label{bbcq-ndl}
Fix a constant $\ell \geq 2$.  Then 
all tree-shaped CQs with at most $\ell$ leaves and arbitrary ontologies have polynomial-size NDL-rewritings.
%any tree-shaped CQ $\q$ with at most $\ell$ leaves and OWL 2 QL TBox $\Tmc$ has a polynomial-size NDL-rewriting. 
\end{theorem}

As for FO-rewritings, % of bounded-leaf queries, 
we can use Theorems~\ref{TW2rew}, \ref{rew2prim}, \ref{nbp-to-ihp}, and \ref{tree-hg-to-query} to show that 
the existence of polysize FO-rewritings is equivalent to the open problem % in complexity theory 
of whether $\nlpoly\subseteq \ncone$. 

\begin{theorem}\label{nbps-conditional}
The following are equivalent:
\begin{enumerate}
\item There exist polysize FO-rewritings for all linear CQs and depth $2$ ontologies;
\item There exist polysize FO-rewritings for all tree-shaped CQs with at most $\ell$ leaves and arbitrary ontologies (for any fixed $\ell$);
\item There exists a polynomial function $p$ such that every NBP of size at most $s$ is
computable by a formula of size $p(s)$. Equivalently,
%All functions in \nlpoly\ are computed by polysize Boolean formulas, that is, 
$\nlpoly\subseteq \ncone$.
\end{enumerate}
\end{theorem}

Turning next to bounded treewidth queries and bounded depth ontologies, 
Theorems \ref{DL2THP} and \ref{thm:thp_vs_sac} together provide a means of constructing a polysize monotone
 \sac\ circuit that computes $\homfn$.
%Thus, by 
Applying Theorem~\ref{Hom2rew}, we obtain:
\begin{theorem}\label{btw-ndl}
Fix %constants 
$\twidth >0$ and $d > 0$.
Then all CQs of %$\q$ of
treewidth $\leq t$ % at most $\twidth$
 and ontologies of depth $\leq d$ %at most $d$ 
have polysize %nomial-size 
NDL-rewritings.
\end{theorem}

In the case of FO-rewritings, we can show that the existence of polysize rewritings corresponds to
the open question of whether $\sac\subseteq \ncone$.

\begin{theorem}\label{btw-fo}
The following are equivalent:
\begin{enumerate}
\item There exist polysize FO-rewritings for all tree-shaped CQs and depth 2 ontologies;
\item There exist polysize FO-rewritings for all CQs of treewidth at most $\twidth$ and ontologies of depth at most $d$ 
(for fixed constants $\twidth> 0$ and $d>0$); % and class $\tclass$ with the \pip);
\item There exists a polynomial function $p$ such that every semi-unbounded fan-in circuit of size at most $\size$ and depth at most $\log \size$ is
computable by a formula of size $p(\size)$. Equivalently,
%All functions in \sac\ are computed by polynomial-size Boolean formulas, that is,
$\sac\subseteq \ncone$.
\end{enumerate}
\end{theorem}

To complete the succinctness landscape, we generalize the result of \cite{lics14-KKPZ} that says that all tree-shaped queries
and depth 1 ontologies 
have polysize PE-rewritings by showing that this is also true for %property applies also to 
the wider class of bounded treewidth queries.

\begin{theorem} \label{depth-one-btw}
Fix %a constant 
$t>0$. Then there exist polysize PE-rewritings for all CQs of treewidth $\leq t$ %at most $t$ 
and depth 1 ontologies. % of depth 1.
%Then the class of all CQs of treewidth at most $d$ enjoys polynomial
%PE rewriting with respect to ontologies of depth 1.
\end{theorem}

\section{Complexity Results for Query Answering }\label{sec:complexity}
To complement our succinctness results and to gain a better understanding of the inherent 
difficulty of query answering, we analyze the computational complexity of 
answering tree-like queries in the presence of OWL 2 QL ontologies.

\subsection{Bounded Depth Ontologies}
We begin by showing that the \LOGCFL\ upper bound for bounded treewidth queries from \cite{DBLP:conf/icalp/GottlobLS99}
remains applicable in the presence of ontologies of bounded depth.

\begin{theorem}\label{logcfl-btw}
CQ answering is in \LOGCFL\ for bounded treewidth queries and bounded depth ontologies. 
%CQ answering is \LOGCFL-complete for bounded treewidth queries and bounded depth ontologies. 
%The lower bound holds already for tree-shaped queries and depth 0 ontologies. 
\end{theorem}
\begin{proof}
By %Equation 
(\ref{eq:canmod}), 
%we have that for every CQ $\q$ and tuple $\vec{a}$ 
%of individuals from $\Amc$: 
$\Tmc, \Amc \models \q(\vec{a})$ just in the case 
that $\canmod \models \q(\vec{a})$. When $\Tmc$ has finite depth, $\canmod$ is a finite relational structure, 
so the latter problem is nothing other than
standard conjunctive query evaluation over databases, which is \LOGCFL-complete when restricted to bounded treewidth queries \cite{DBLP:journals/jacm/GottlobLS01}. As %the class \
LOGCFL\ 
is closed under \llred\ reductions \cite{DBLP:conf/icalp/GottlobLS99}, 
it suffices to show that %we can compute 
$\canmod$ can be computed % from $(\Tmc,\Amc)$ 
by means of an \llred-transducer 
(that is, a deterministic logspace Turing machine with access to an \LOGCFL\ oracle). 

We briefly describe the \llred-transducer that generates $\canmod$ when given a KB $\Kmc=(\Tmc,\Amc)$ whose ontology $\Tmc$ has depth at most $k$. 
%into the corresponding canonical model $\canmod$. 
First note that we need only logarithmically many bits 
to represent a predicate name or individual constant from $\Kmc$. Moreover, as $\Tmc$ has depth at most $k$, %we know that 
the domain of $\canmod$ is contained in the set $U = \{aw \mid aw \in \Delta^{\canmod}, |w| \leq k\}$. Since $k$ is a fixed
constant, each element in $U$ can be stored using logarithmic space in $|\Kmc|$. 
Finally, we observe that each of the following operations can be performed by making a call to an \NL\ (hence \LOGCFL) oracle:
\begin{itemize}
\item Decide whether %a given element 
$aw \in U$ belongs to $\Delta^{\canmod}$.
\item Decide whether %an element 
$u \in \Delta^{\canmod}$ belongs to $A^{\canmod}$.
\item Decide whether %a pair %of elements 
$(u,u') \in (\Delta^{\canmod})^2$ % \times \Delta^{\canmod}$ 
belongs to $r^{\canmod}$.
\end{itemize}
Indeed, all three %of these 
problems can be decided %by performing 
%requires
using %only 
%a constant number of
constantly many entailment checks, 
%which can be done in  
and entailment is in  %known to be feasible in 
\NL \cite{CDLLR07}. 
%We can thus iterate over all predicate names in $\Kmc$, and for each, % predicate name, 
%we can further iterate over all (pairs of) elements of $U$ to identify (and output) those that belong to $P^{\canmod}$. 
\end{proof}

If we restrict the number of leaves in tree-shaped queries, then we can improve the preceding upper bound to \NL:
%\textbf{ADD ALSO DB LINEAR LOWER?}

\begin{figure}[t!]
\setlength{\abovecaptionskip}{4pt plus 0pt minus 0pt}
%\begin{boxedminipage}{15cm}
\hrule \smallskip
\noindent \textbf{Procedure} \bbbdalgo 
\smallskip \hrule \smallskip
\textbf{Input:} KB $(\Tmc, \Amc)$, tree-shaped query $\q$ with $\avars(\q)=(z_1, \ldots,z_n)$, tuple $\vec{b}=(b_1, \ldots, b_n) \in \ainds(\Amc)^n$
%\textbf{Output:} yes$\Tmc, \Amc \models q$
%\setlist{leftmargin=*,labelindent=0cm}

\begin{enumerate}[leftmargin=*,labelindent=0cm,label=\small{\arabic*:}]
 \item Fix a directed tree $T$ compatible with $G_\q$. Let $v_0$ be the root variable.
Set $U = \{aw \mid  |w| \leq 2|\Tmc|+|\q|\}$.  %|w| \leq k\}$. %a \in \ainds(\Amc), w \in (\rni)^*, 
% |w| \leq 2|\Tmc|+|\q|\}$. 
 \item Guess $u_0 \in U$ and % = \{aw \in \Delta^{\canmod} \mid 
 %a \in \ainds(\Amc), %aw \in \Delta^{\canmod},  %|w| \leq k\}$. %a \in \ainds(\Amc), 
 %w \in (\rni)^*, 
% |w| \leq 2|\Tmc|+|\q|\}$. 
return \no\ if either:
 \begin{itemize}[labelindent=0cm]
 \item $u_0\in\ainds(\Amc)$ and  \compnode($\Tmc, \Amc, \q, \vec{b}, v_0, u_0)\!=\,$\false
 \item  $u_0 = a_0 w_0 R$ and \compnodeanon($\Tmc,  \q,  v_0, R)\!=\,$\false
 \end{itemize}
 \item Initialize $\frontier$ to $\{(v_0,u_0)\}$. 
 \item While $\frontier \neq \emptyset$
 \begin{enumerate}[leftmargin=3mm,labelindent=0cm,label=\small{\alph*:}]
 \item Remove $(v_1,u_1)$ from $\frontier$.
 %\item Return \textbf{no} if \compnode($\Tmc, \Amc, \q, \vec{b}, v^*, u^*$)= \false. 
% one of the following holds:
%  \begin{itemize}
%  \item $v^* \in \avars(\q)$ and $u^* \not \in \ainds(\Amc)$
%  \item $\q$ contains $A(v^*) \in \q$ and $u^* \not \in A^{\canmod}$.
% \end{itemize}
\item For every child $v_2$ of $v_1$
\begin{enumerate}[leftmargin=3mm]
\item Guess %an element 
$u_2 \in U$. Return \textbf{no} if one of the following holds:
\begin{itemize}[leftmargin=*]
%\item \compedge($\Tmc, \Amc, \q, v_1, v_2, u_1, u_2)\!=\,$\false
\item $\q$ contains $P(v_1,v_2)$ ($P \in \rni$), % and 
$\canmod \not \models P(u_1,u_2)$
 \item $u_2\!\in\!\ainds(\Amc)$ and  \compnode$(\!\Tmc, \Amc, \q, \vec{b}, v_2, u_2)\!=$\false
 \item $u_2 = a_2 w_2 R$ and \compnodeanon($\Tmc,  \q,  v_2, R)\!=\,$\false
\end{itemize}
\item Add $(v_2,u_2)$ to $\frontier$.
\end{enumerate}
 \end{enumerate}
 \item Return \textbf{yes}.
\end{enumerate}
% \hrule
\medskip
\hrule \smallskip
\noindent %\textbf{Procedure} 
\compnode($\Tmc, \Amc, \q, \vec{b}, v, u$)
\smallskip \hrule \smallskip
%\compnode($\Tmc, \Amc, \q, \vec{b}, v, u$) = \false\ if one of the following holds (otherwise, output \true):
Return \false\ iff one of the following holds:
\begin{itemize}
% \item $u \not \in \Delta^{\canmod}$
  \item $v=z_i$ and $u \neq b_i$ for some $1 \leq i \leq n$
  \item $\q$ contains $A(v)$ and $\Tmc, \Amc \not \models A(u)$ %$u \not \in A^{\canmod}$,
    \item $\q$ contains $P(v,v)$ and $\Tmc, \Amc \not \models P(u,u)$. % \not \in r^{\canmod}$.
 \end{itemize}
 
 \medskip
 
\hrule \smallskip
\noindent %\textbf{Procedure} 
\compnodeanon($\Tmc, \q, v, R$)
\smallskip \hrule \smallskip
Return \false\ iff one of the following holds:
\begin{itemize}
\item  $v \in \avars(\q)$ 
\item $\q$ contains $A(v)$ and $\Tmc \not \models \exists y R(y,x) \rightarrow A(x) $
\item $\q$ contains some atom of the form $S(v,v)$.
 \end{itemize}
\caption{Non-deterministic procedure % $\simpleeval$ 
for answering tree-shaped queries}% (left) and its sub-procedures (right)} % w.r.t.\ bounded depth TBoxes}
\label{algo:tree-entail}
\vspace*{-.35cm}
\end{figure}

\begin{theorem}\label{nl-bb}
CQ answering is \NL-complete for bounded leaf queries and bounded depth ontologies. 
%CQ answering is \NL-complete for bounded leaf queries and bounded depth ontologies. 
%The lower bound holds already for atomic queries and depth 0 ontologies. 
\end{theorem}
\begin{proof}
The lower bound is an immediate consequence of the \NL-hardness of answering atomic queries in OWL 2 QL
%easily %straightforward 
%shown by reduction from directed reachability using only axioms $A(x) \rightarrow B(x)$ (with $A,B \in \cn$)
 %inclusions of the form $A \sqsubseteq B$ (with $A,B$ concept names) 
 \cite{CDLLR07}. 

For the upper bound, we introduce in Figure \ref{algo:tree-entail} a non-deterministic procedure \bbbdalgo\
for deciding whether a tuple is a certain answer of a tree-shaped query. % w.r.t.\ a given OWL 2 QL KB. 
The procedure views the input query as a directed tree and constructs a homomorphism on-the-fly by traversing the tree from root to leaves. 
The set $\frontier$ is initialized with a single pair $(v_0,u_0)$, which represents the choice of where to map the root variable $v_0$.
The possible choices for $u_0$ include all individuals from $\Amc$ as well as all elements $aw$ that belong to the domain of the 
canonical model $\canmod$ and have $|w| \leq 2|\Tmc|+|q|$. The latter bound is justified by the well-known fact that if 
there is a homomorphism of $\q$ into $\canmod$, then there is one %there exist a homomorphism 
whose image only involves elements $aw$ 
with $|w| \leq 2|\Tmc|+|q|$. 
We use the sub-procedures \compnode\ or \compnodeanon\ to check that the guessed element $u_0$ is compatible with the variable $v_0$.
%(and return \no\ if not). 
If $u_0  \in \ainds(\Amc)$, then we use the first sub-procedure \compnode, which verifies that (i) if $v_0$ is an 
answer variable, then $u_0$ is the individual corresponding to $v_0$ in the tuple $\vec{b}$, and (ii) $u_0$
satisfies all atoms in $\q$ that involve only $v_0$. If $u_0  \not \in \ainds(\Amc)$, then $u_0$ must take the form $a_0 w_0 R$.
In this case, \compnodeanon\ is called and checks that $v_0$ is not an answer variable, $\q$ does not contain a reflexive loop at $v_0$, 
and $\Tmc \models \exists y R(y,x) \rightarrow A(x)$ (equivalently, $a_0 w_0 R \in A^\canmod$) for every $A(v_0) \in \q$. % just in the case that $\Tmc \models \exists R^- \sqsubseteq A$). 
The remainder of the procedure consists of a while loop, in which we remove a pair $(v_1,u_1)$ from $\frontier$, and 
%and check that the choice of $u^*$ for $v^*$ is compatible with $\q$. Then, 
if $v_1$ is not a leaf node, we guess where to map the children of $v_1$. 
We must then check %for each child $v_2$ 
that the guessed element $u_2$ for child $v_2$ is compatible with the role 
assertions linking $v_1$ to $v_2$ %(this is handled by the sub-procedure \compedge) 
and the unary atoms concerning $v_2$ 
(using %the sub-procedures 
\compnode\ or \compnodeanon\ described earlier). 
If some check fails, we return \no, and otherwise we add %the pair
 $(v_2, u_2)$ to $\frontier$, for each child $v_2$ of $v_1$. 
We exit the while loop when $\frontier$ is empty, i.e. when an element of $\canmod$ has been assigned to every variable in $\q$. 

Correctness and termination %of  \bbbdalgo\ 
are straightforward to show and hold for arbitrary tree-shaped queries 
and OWL 2 QL ontologies. 
Membership in \NL\ for bounded depth ontologies and bounded leaf queries
relies upon the following observations:
\begin{itemize}
\item if $\Tmc$ has depth $k$ and $aw \in U$, then $|w| \leq k$
\item if $\q$ has $\ell$ leaves, then $|\frontier|$ never exceeds $\ell$ % contains more than $\ell$ elements. 
%\item also: checking whether $aw$ belongs to $U$ can be checked by making at most $|\Tmc| + k$ entailment checks, 
%and such entailment checks are known to be feasible in \NL. 
%\item also: checking if $u_0$ satisfies conditions in \NL. 
%\item also: checking conditions in 4(d) likewise feasible in \NL. 
\end{itemize}
which % together 
ensure %that 
$\frontier$ can be stored in %using only 
logarithmic space. %only logarithmic space is needed to store $\frontier$. 
%It follows that %each element in $U$ can be stored using only logarithmic space,
%%and so 
%we need only logarithmic space to store the (at most $\ell$) elements in $\frontier$. 
%As observed in the proof of Theorem \ref{logcfl-btw}, it can be decided in \NL\ whether an element in $U$ belongs to $\Delta^{\canmod}$,
%and whether an element (resp.\ pair of elements) from $\Delta^{\canmod}$ belong to the interpretation of a given predicate. %concept (resp.\ role). 
%This implies that the various checks %the sub-procedures 
%%\compnode, \compnodeanon, and \compedge\ 
%can all be implemented 
%in non-deterministic logarithmic space. 
%Finally, we note that it is not necessary to store the representation the directed tree 
%$T$ from Step 1, but merely to be able to 
%decide % given variables $v,v'$ 
%whether $v$ is the parent of $v'$ in $T$, and the 
%latter problem clearly belongs to \NL. 
\end{proof}

\subsection{Bounded Leaf Queries \& Arbitrary Ontologies}
The only remaining case is that of bounded leaf queries and arbitrary ontologies, 
for which neither the upper bounds from the preceding subsection, nor the \np\ lower bound from \cite{DBLP:conf/icalp/KikotKPZ12}
can be straightforwardly adapted. We settle the question by showing \LOGCFL-completeness. % for bounded branching queries and arbitrary TBoxes. 
%As the upper and lower bounds are both non-trivial, we prove them separately. 

%It remains to consider what happens if we consider bounded branching queries with arbitrary ontologies. 
%We have seen with the preceding result that the restriction to bounded branching queries leads to a drop in complexity, 
%but our \NL\ membership proof crucially depends on the assumption that the TBox has bounded depth. 
%Likewise, the \np\ lower bound that was shown in \cite{DBLP:conf/icalp/KikotKPZ12} for arbitrary OWL 2 QL TBoxes 
%requires (arbitrary) tree-shaped queries, 
%and it is unclear whether it can be adapted to work for queries with a bounded number of leaves. 
%We settle this question by showing \LOGCFL-completeness. % for bounded branching queries and arbitrary TBoxes. 
%As the upper and lower bounds are both non-trivial, we prove them separately. 

\begin{theorem}\label{logcfl-c-arb}
CQ answering is \LOGCFL-complete for bounded leaf queries and arbitrary ontologies. 
The lower bound holds already for linear queries. % and arbitrary ontologies. 
\end{theorem} 

\subsubsection{\LOGCFL\ upper bound}%Proof of \LOGCFL\ membership}
%For the upper bound, we utilize 
The upper bound relies on a characterization of the class \LOGCFL\ in 
terms of non-deterministic auxiliary pushdown automata (NAuxPDAs). 
We recall that an NAuxPDA \cite{DBLP:journals/jacm/Cook71} is a non-deterministic Turing machine that has an additional work tape that is constrained
to operate as a pushdown store. 
Sudborough \cite{sudborough78} proved that \LOGCFL\ can be characterized as the class of problems that can be solved by 
NAuxPDAs that run in logarithmic space and in polynomial time (note that the space on the pushdown tape is not subject to the logarithmic space bound). 
Thus, to show membership in \LOGCFL, it suffices to define a procedure for answering bounded leaf queries 
that can be implemented by such an NAuxPDA. We present such a procedure in Figure \ref{algo:bbqueries}. 
The input query is assumed to be connected; this is w.l.o.g.\ since 
the connected components can be treated separately. 
%Note that the procedure assumes the input query is connected (this is w.l.o.g.\ %without loss of generality 
%since %we can always treat 
%its connected components %of the query  
%can be treated separately). 

\begin{figure*}[t!]
\setlength{\abovecaptionskip}{4pt plus 0pt minus 0pt}
%\begin{boxedminipage}{15cm}
\hrule \smallskip
\noindent \textbf{Procedure} \bbarbalgo 
\smallskip \hrule \smallskip
\textbf{Input:} KB $(\Tmc,\Amc)$, connected tree-shaped query $\q$ with $\avars(\q)=(z_1, \ldots,z_n)$, tuple $\vec{b}=(b_1, \ldots, b_n) \in \ainds(\Amc)^n$
%\textbf{Output:} yes$\Tmc, \Amc \models q$
%\setlist{leftmargin=*,labelindent=0cm}
\vspace*{-2mm}
\begin{multicols}{2}
\begin{enumerate}[leftmargin=*,labelindent=0cm,label=\small{\arabic*:}]
 \item Fix a directed tree $T$, with root $v_0$, compatible with $G_\q$. %Let $v_0$ be the root variable.
% \item Set $U = \{aw \mid aw \in \Delta^{\canmod}, %a \in \ainds(\Amc), w \in (\rni)^*, 
% |w| \leq 2|\Tmc|+|\q|\}$. 
 \item Guess %an individual 
 $a_0 \in \ainds(\Amc)$ and %word 
 $w_0 \in (\rni)^*$ with $|w_0| \leq 2|\Tmc|+|\q|$ and $a_0w_0 \in \Delta^{\canmod}$. 
Return \no\ if either:
 \begin{itemize}
 \item $w_0=\varepsilon$ and \compnode($\Tmc, \Amc, \q, \vec{b}, v_0, a$)= \false;
 \item  $w_0 = w_0' R$ and \compnodeanon($\Tmc,  \q,  v_0, R$)= \false. 
 \end{itemize}
% Return \no\ if \compnode($\Tmc, \Amc, \q, \vec{b}, v_0, aw_0$)= \false. %$aw_0$ \textbf{NOT COMPATIBLE} with $v_0$.
 \item Initialize $\stack$ to $w_0$, $\stackheight$ to $|w_0|$, and 
% Initialize 
$\frontier$ to $\{(v_0, v_i, a_0, \stackheight) \mid v_i \text{ is a child of } v_0 \}$.
% $\{(v_0,v_1,a_0, \stackheight), \ldots, (v_0,v_p,a_0,\stackheight)\}$, where 
 %$n_0 = |w_0|$ and 
 %$v_1, \ldots, v_p$ are the children of $v_0$. % in $T$. 
 \item While $\frontier \neq \emptyset$, do one of the following:
\begin{description}
\item[Option 1] // \emph{Take one step in the core} 
\begin{enumerate}[leftmargin=1mm,labelindent=0cm,label=\small{\alph*:}]
\item Remove $(v_1, v_2, c, 0)$ from $\frontier$.
%\item Return \no\ if \compnode($\Tmc, \Amc, \q, \vec{b}, v_1, c$)= \false. %$c$ \textbf{NOT COMPATIBLE} with $v_1$.
%\item G
\item Guess $d \in \ainds(\Amc)$.  Return \no\ if either
\begin{itemize}[leftmargin=3mm]
%\item \compedge($\Tmc, \Amc, \q, v_1, v_2, c, d$)= \false, or 
\item $\q$ contains $P(v_1,v_2)$ ($P \in \rni$), % and 
$\canmod \not \models P(u_1,u_2)$
\item \compnode($\Tmc, \Amc, \q, \vec{b}, v_2, d$)= \false. %$(c,d)$ \textbf{NOT COMPATIBLE} with $(v_1,v_2)$.
\end{itemize}
\item For every child $v_3$ of $v_2$,  
add $(v_2,v_3,d,0)$ to $\frontier$. 
\end{enumerate}
\item[Option 2] // \emph{Take one step `forward' in anonymous part %(away from the core)
}
\begin{enumerate}[leftmargin=1mm,labelindent=0cm,label=\small{\alph*:}]
\setcounter{enumii}{3}
\item If $\stackheight = 2|\Tmc|+|\q|$, return \no. Otherwise, remove $(v_1, v_2, c, \stackheight)$ from $\frontier$. 
%Return \no\ if %$m < \stackheight$ or $m = 2|\Tmc|+|\q|$.
%$\stackheight = 2|\Tmc|+|\q|$.
%\item Return \no\ if either:
%\begin{itemize}%%%%%%%
%\item $n=0$ and \compnode($\Tmc, \Amc, \q, \vec{b}, v_1, c$)= \false.
%\item $n > 0$ and \compnodeanon($\Tmc, \Amc, \q, \vec{b}, v_1, R$)= \false, where $R$ is the top symbol of $\stack$.
%\end{itemize}
%$c$ / top of stack \textbf{NOT COMPATIBLE} with $v_1$. 
\item Guess $S \in \rni$.  Return \no\ if one of the following holds:%\compedge($\Tmc, \Amc, \q, \vec{b}, v_1, v_2, u_1, u_2$)= \false
\begin{itemize}[leftmargin=3mm]
\item $\stackheight=0$ and $\Tmc,\Amc \not \models \exists x \, S(c,x)$ 
\item $\stackheight> 0$ and $\Tmc \not \models \exists y R(y,x) \rightarrow \exists y S(x,y)$, where $R$ is the top symbol of $\stack$
\item $\q$ contains $P(v_1,v_2)$ and $\Tmc \not \models S(x,y) \rightarrow P(x,y)$
\item \compnodeanon($\Tmc,  \q,  v_2, S$)= \false
\end{itemize}
%S \textbf{NOT COMPATIBLE} with $(v_1,v_2)$ \textbf{and} stack. 
\item If $v_2$ has at least one child in $T$, then
\begin{itemize}[leftmargin=3mm]
\item Push $S$ onto $\stack$, and increment $\stackheight$.
\item For every child $v_3$ of $v_2$ in $T$, add $(v_2,v_3,c,\stackheight)$ to $\frontier$. 
\end{itemize}
Else, %be the maximum of $\{m-\ell \mid (v,v',d,\ell) \in \frontier\}$. 
pop $\delta = \stackheight - \mathbf{max} \{\ell \mid (v,v',d,\ell) \in \frontier\}$ symbols from $\stack$ and decrement $\stackheight$ by $\delta$.
\end{enumerate}
\item [Option 3] // \emph{Take one step `backward' in anonymous part %(towards the core)
}
\begin{enumerate}[leftmargin=1mm,labelindent=0cm,label=\small{\alph*:}]
\setcounter{enumii}{6}
\item If $\stackheight = 0$, return \no. Else, remove $\deepest = \{(v_1,v_2,c,n) \in \frontier \mid n = \stackheight\}$ from $\frontier$,
%\item Pop $R$ from $\stack$ and decrement $\stackheight$.
%\item 
 pop $R$ from $\stack$, and decrement $\stackheight$.  
\item Return \no\ if for some $(v_1,v_2,c,n) \in \deepest$, one of the following holds:
\begin{itemize}[leftmargin=3mm]
\item $\stackheight=0$ and  \compnode($\Tmc, \Amc, \q, \vec{b}, v_2, c$)= \false
\item $\stackheight>0$ and  \compnodeanon($\Tmc,  \q,  v_2, S$)= \false, where $S$ is the top symbol of $\stack$
\item $\q$ contains $P(v_1,v_2)$ and $\Tmc \not \models R(y,x) \rightarrow P(x,y)$
\end{itemize}
%\compnodeanon($\Tmc, \q,  v_1, R$)= \false, %$\exists R^-$ \textbf{NOT COMPATIBLE} with $v_1$, 
%for some $(v_1,v_2,c,n) \in \deepest$. 
\item If there is some $(v_1,v_2,c,n) \in \deepest$ such that $v_2$ is a non-leaf node in $T$:
%$\children = \{(v_2,v_3) \mid (v_1, v_2, c,n) \in \deepest, v_2 \text{ has child } v_3,\}$ is non-empty:%$\children \neq \emptyset$:
\begin{itemize}
\item For every $(v_1,v_2,c,n) \in \deepest$ and child $v_3$ of $v_2$ in $T$, %$(v_2, v_3) \in \children$, 
add $(v_2, v_3, c, \stackheight)$ to $\frontier$.
\end{itemize}
Else, pop $\delta = \stackheight - \mathbf{max} \{\ell \mid (v,v',d,\ell) \in \frontier\}$ %be the maximum of $\{m-\ell \mid (v,v',d,\ell) \in \frontier\}$. 
 symbols from $\stack$ and decrement $\stackheight$ by $\delta$.
%\begin{enumerate}
%\item Return \no\ if $\q$ contains $r(v_1,v_2)$ and $\Tmc \not \models R^- \sqsubseteq r$
%%$R^-$ \textbf{NOT COMPATIBLE} with $(v_1,v_2)$.TO DO!!!
%\item For every child $v_3$ of $v_2$, add $(v_2,v_3,c,n-1)$ to $\frontier$.
%\end{enumerate}
\end{enumerate}
\end{description}
 \item Return \textbf{yes}.
\end{enumerate}
\end{multicols}
\vspace*{-5mm}
%\hrule
\caption{Non-deterministic procedure % $\simpleeval$ 
for answering bounded leaf queries. Refer to Fig.\ \ref{algo:tree-entail} for the definitions of %sub-procedures 
\compnode\ and \compnodeanon. }
\label{algo:bbqueries}
\vspace*{-.25cm}
\end{figure*}

We start by giving an overview of the procedure \bbarbalgo. Like %the earlier procedure 
\bbbdalgo,
the idea is to view the input query $\q$ as a tree and iteratively construct a 
homomorphism of the query into the canonical model $\canmod$,
working from root to leaves. % the root towards the leaves. 
At the start of the procedure, we guess an element $a_0w_0$
to which the root variable $v_0$ is mapped and check that the guessed element is compatible with $v_0$.
However, instead of storing directly $a_0w_0$ on $\frontier$, 
we push the word $w_0$ onto the stack ($\stack$) and record the height of the stack $(|w_0|)$ in $\stackheight$.
We then initialize $\frontier$ to the set of all $4$-tuples $(v_0,v_i,a_0, \stackheight)$ with $v_i$ a child of $v_0$. 
Intuitively, a tuple $(v,v',c,n)$ % in $\frontier$ 
records that the variable $v$ is mapped to the element $c \,\stack[n]$ and 
that the child $v'$ of $v$ remains to be mapped (we use $\stack[m]$ to denote the 
word consisting of the first $m$ symbols of $\stack$). % starting from the bottom). 

In Step 4, we will remove one or more tuples from $\frontier$, choose where to map the variable(s) in the second component,
and update $\frontier$, $\stack$, and $\stackheight$ accordingly. There are three options depending on how 
we map the variable. Option 1 will be used for tuples $(v,v',c,0)$ in which both $v$ and $v'$ are mapped to named constants, 
while Option 2 (resp.\ Option 3) is used for tuples $(v,v',c,n)$ in which we wish to map $v'$ to a child (resp.\ parent) of $v$. 
Crucially, however, the order in which tuples are treated matters, due to the fact that several tuples are `sharing' the single stack. 
Indeed, when applying Option 3, we pop a symbol from $\stack$, and may therefore lose some information that is needed for 
the processing of other tuples. To prevent this, %such a situation, 
Option 3 may only be applied to tuples whose last component is maximal 
(i.e.\ equals $\stackheight$), and it must be applied to \emph{all} such tuples. 
For Option 2, we will also impose that the selected tuple $(v,v',c,n)$ is such that $n=\stackheight$. 
This is needed because Option 2 corresponds to mapping $v'$ to an element $c \, \stack[n] \,S$, and 
we need to access the $n$th symbol in $\stack$ %both 
to determine the possible choices for $S$ 
and %to be able 
to record the symbol chosen (by pushing it onto $\stack$). % which relation was chosen. 

The procedure terminates and returns \yes\ when $\frontier$ is empty, meaning that we have successfully constructed 
a homomorphism of the input query into the canonical model that witnesses that the input tuple is an answer. 
Conversely, given such a homomorphism, we can % use it to 
define a successful execution of \bbarbalgo, % of the procedure, 
as illustrated by the following example.  

\begin{example}\label{ex-second-algo}
Reconsider  the KB $(\Tmc_0, \Amc_0)$, CQ $\q_0$, and homomorphism
$\q_0(c,a) \rightarrow \Cmc_{\Tmc_0, \Amc_0}$ from Figure \ref{ex-fig}.
We show in what follows how $h_0$ can be used to define an execution of \bbarbalgo\ that outputs \yes\ 
on input $(\Tmc, \Amc, \q, (c,a))$. 

In Step 1, we will fix some variable, say $y_1$, as root. Since we wish to map $y_1$ to $aP$, we will guess in Step 2 the constant $a$ and the word $P$
and verify using \compnodeanon\ that our choice is compatible with $y_1$. 
% by calling 
%\compnodeanon\ on input ($\Tmc,  \q,  y_1, P$);
%it returns \yes\ since $y_1 \not \in \avars(\q)$, $\Tmc \models \exists y\, P(y,x) \rightarrow B(x)$, and aside from $B(y_1)$, $\q$ contains no unary atoms or loops involving $y_1$. % aside from $B(y_1)$. 
As the check succeeds, we proceed to Step~3, where we initialize $\stack$ to $P$, $\stackheight$ to $1$, and $\frontier$ to $\{(y_1,y_2, a, 1),(y_1, y_3, a, 1)\}$. 
Here the tuple $(y_1,y_2,a,1)$ % is used to 
records that $y_1$ has been mapped to % the object 
$a\, \stack[1] = aP$ and %that 
the edge between $y_1$ and $y_2$ remains to be mapped. 

%Each iteration of Step 6 involves removing one or more tuples from $\frontier$ using one of the three options. 
%Since we aim to produce a successful execution of \bbarbalgo, we will use the homomorphism from Figure \ref{fig:second-algo} 
%as our guide in deciding which of the three options should be used for a given tuple in $\frontier$. 

At the beginning of Step 4, $\frontier$ contains $2$ tuples: $(y_1,y_2, a, 1)$ and $(y_1, y_3, a, 1)$. 
Since $y_1$, $y_2$, and $y_3$ are mapped to $aP$, $a$, and $aPS$ respectively, 
we will use Option~3 (`step backward') for $(y_1,y_2, a, 1)$ and Option~2 (`step forward') for $(y_1, y_3, a, 1)$.
If we were to apply Option 3 at this stage, then we would be forced to treat both tuples together, 
and the check in Step 4(h) would fail for $(y_1, y_3, a, 1)$ since $S(y_1,y_3) \in \q$ but $\Tmc \not \models R(y,x) \rightarrow S(x,y)$.
%Note that we cannot apply Option 3 at this stage since this would require the tuple $(y_1, y_3, a, 1)$ to be processed as well, 
%which would lead to \no\ being returned in Step 4(h) (since $S(y_1,y_3) \in \q$ but $\Tmc \not \models R(y,x) \rightarrow S(x,y)$). 
%As Option 3 involves popping $r$ from $\stack$, which would make it unavailable for processing the tuple $(y_1, y_3, a, 1)$,
%We will therefore start by performing Option 2.
% Indeed, this is the only possible choice at this stage since Option 3 would require us to process 
%the tuple $(y_1, y_3, a, 1)$ as well, and this would lead to \no\ being returned (since $s(y_1,y_3) \in \q$ but $\Tmc \not \models r^- \sqsubseteq s$). 
%$\q$ contains the atom $s(y_1,y_3) \in \q$ but $\Tmc \not \models r^- \sqsubseteq s$ 
We will therefore choose to perform Option 2, removing %the tuple 
$(y_1, y_3, a, 1)$ from $\frontier$ in Step 4(d) and %since $y_3$ maps to $aPS$,
%we will 
guessing $S$ in Step 4(e). As the check succeeds, we will proceed to 4(f), 
where we push $S$ onto $\stack$, set $\stackheight = 2$, and add tuples $(y_3, y_4,a, 2)$ and $(y_3,y_5,a,2)$ to $\frontier$. 
Observe that from the tuples in $\frontier$, 
we can read off the elements $a \stack[1]$ and $a \stack[2]$ to which %the 
variables $y_1$ and $y_3$ are mapped. % respectively. 

At the start of the second iteration of the while loop,  we have % begins with 
%We return to the start of Step 4 with %At the start of the second iteration, % of Step 4, 
%we have 
$\frontier = \{(y_1,y_2, a, 1),(y_3,y_4,a,2),(y_3,y_5,a,2)\}$, $\stack= PS$, and  $\stackheight=2$. 
Note that since $h_0$ %the considered homomorphism 
maps $y_4$ to $aP$ and $y_5$ to $aPST^-$, 
we will use Option~3 to treat $(y_3,y_4,a,2)$ and Option~2 for $(y_3,y_5,a,2)$. 
It will again be necessary to start with Option 2. We will thus remove $(y_3,y_5,a,2)$ from $\frontier$, and guess the relation $T^-$ (which satisfies the required conditions). 
Since $y_5$ does not have any children and $\stackheight - \mathbf{max} \{\ell \mid (v,v',d,\ell) \in \frontier\} = 2 - 2 = 0$, 
we leave $\frontier$, $\stack$, and $\stackheight$ unchanged. 

At the start of the third iteration, we have $\frontier = \{(y_1,y_2, a, 1),(y_3,y_4,a,2)\}$, $\stack= PS$, and $\stackheight=2$. %, and $\frontier = \{(y_1,y_2, a, 1),(y_3,y_4,a,2)\}$. 
We have already mentioned % above 
that both tuples should be handled using Option 3. We will start by applying Option 3 to 
tuple $(y_3,y_4,a,2)$ since its last component is maximal. % is the only tuple with last component equal to $\stackheight$. 
We will thus remove $(y_3,y_4,a,2)$ %this tuple from 
from
$\frontier$, pop $S$ from $\stack$, and decrement $\stackheight$. As the checks succeed for $S$, %relation $S$ verifies the required conditions in Step 6(h), 
we will %$\children=\{(y_4,x_2)\}$ %in Step 6(i) 
%and 
add the tuple $(y_4, x_2,a,1)$ to $\frontier$ in Step 4(i). 

At the start of the fourth iteration, we have $\frontier = \{(y_1,y_2, a, 1),(y_4,x_2,a,1)\}$, $\stack = P$, and $\stackheight=1$. %, and $\frontier = \{(y_1,y_2, a, 1),(y_4,x_2,a,1)\}$. 
Since $y_4$ and $x_2$ are mapped respectively to $aP$ and $a$, we should use Option 3 to handle the second tuple. 
We will thus apply Option 3 with $\deepest = \{(y_1,y_2, a, 1),(y_4,x_2,a,1)\}$. This will lead to both tuples being removed
from $\frontier$, $P$ being popped from $\stack$, and $\stackheight$ being decremented. We next perform the required checks in
Step 4(h), and in particular, we verify that the choice of where to map the answer variable $x_2$ agrees with the input vector $\vec{b}$ 
(which is indeed the case). In Step 4(i), we %We will then %proceed to set $\children = \{(y_2, x_1)\}$ and 
add $(y_2, x_1,a,0)$ to $\frontier$. 

The final iteration of the while loop begins with $\frontier = \{(y_2,x_1,a,0)\}$, $\stack= \epsilon$, and $\stackheight=0$. %, and $\frontier = \{(y_2,x_1,a,0)\}$. 
Since %the considered homomorphism 
$h_0$ maps $x_1$ to the constant $c$, we will choose Option 1.  
%(which is possible since the last component of $(y_2,x_1,a,0)$ is $0$). 
We thus remove $(y_2,x_1,a,0)$ from $\frontier$, guess the constant $c$, and perform the required 
compatibility checks. 
As $x_1$ is a leaf, %has no children, 
no new tuples are added to $\frontier$. We are thus left with $\frontier= \emptyset$, and so we %will
 continue on
 to Step 7, where we output \yes. 
\end{example}

In the appendix, we argue that  \bbarbalgo\ can be implemented by an NAuxPDA, and we
prove its correctness:
%and provide a detailed proof of correctness:

\begin{proposition} \label{logcfl-upper-prop}
Every execution of \bbarbalgo\ terminates. 
There exists an execution of \bbarbalgo\ that returns \yes\ on input $(\Tmc,\Amc, \q, \vec{b})$ if and only if $\Tmc, \Amc \models \q(\vec{b})$. 
% If $\Tmc, \Amc \models \q(\vec{b})$, then some execution of \bbarbalgo($\Tmc, \Amc, \q, \vec{b})$ returns \yes.
\end{proposition}

\subsubsection{\LOGCFL\ lower bound}%Proof of \LOGCFL-hardness}
%The \LOGCFL-hardness result for evaluating tree-shaped queries over relational databases \cite{}
%was proven by reduction from the problem of deciding whether an 
%input of length $n$ is accepted by the $n$th circuit of a \emph{logspace-uniform} family of 
%\sac\ circuits \cite{venkateswaran1991properties}. 
%We adopt a similar proof strategy as \cite{}, but with one crucial difference:
%the power of OWL 2 QL ontologies allows us to `unravel' the circuit into a tree and thus use linear 
%queries instead of tree-shaped ones.
\begin{figure*}[t]
\setlength{\abovecaptionskip}{4pt plus 0pt minus 0pt}
\scalebox{.75}{
\begin{tikzpicture}%\footnotesize%
\scriptsize
\node at (-3.2,4) {{\small (a)}};
\node[fill=gray!40,input,label=left:$g_{11}$] (in1) at (-2.7,0) {$x_1$}; 
\node[input,label=left:$g_{12}\!$] (in2) at (-1.6,0) {$x_2$}; 
\node[fill=gray!40,input,label=left:$g_{13}\!$] (in3) at (-.5,0) {$\!\!\neg x_3$}; 
\node[input,label=left:$g_{14}\!$] (in4) at (.6,0) {$x_4$}; 
\node[input,label=left:$g_{15}\!$] (in5) at (1.7,0) {$x_5$}; 
\node[input,label=left:$g_{16}\!$] (in6) at (2.8,0) {$\!\!\neg x_1$}; 
%\node[fill=gray!40,input,label=left:$g_7$] (in7) at (3.3,0) {$\!\!\neg x_2$}; 
%
\node at (0,-0.7) {\textsc{Input}: $\quad x_1=1 \quad x_2=0 \quad x_3=0\quad x_4=0 \quad x_5=0$};
\node[fill=gray!40,or-gate,label=left:$g_7$] (or1-1) at (-2.2,1) {OR}; 
\node[fill=gray!40,or-gate,label=left:$g_8$] (or1-2) at (-.75,1) {OR}; 
%\node[or-gate,label=left:$g_{10}$] (or1-3) at (0,1) {OR}; 
\node[fill=gray!40,or-gate,label=left:$g_{9}$] (or1-4) at (.75,1) {OR}; 
\node[or-gate,label=left:$g_{10}$] (or1-5) at (2.2,1) {OR}; 
\node[fill=gray!40,and-gate,label=left:$g_{4}$] (and1-1) at (-1.5,2) {AND};
\node[fill=gray!40,and-gate,label=left:$g_{5}$] (and1-2) at (0,2) {AND};
\node[and-gate,label=left:$\quad g_{6}$] (and1-3) at (1.5,2) {AND};
%\node[fill=gray!40,and-gate,label=left:$g_{16}$] (and1-4) at (2.2,2) {AND};
%
\node[fill=gray!40,or-gate,label=left:$g_{2}$] (or2-1) at (-.8,3) {OR}; 
\node[fill=gray!40,or-gate,label=left:$g_3$] (or2-2) at (.8,3) {OR}; 
\node[fill=gray!40,and-gate,label=left:$g_{1}$] (and2) at (0,4) {AND}; 
\draw[->] (in1) to (or1-1);
\draw[->] (in1) to (or1-2);
\draw[->] (in2) to (or1-1);
\draw[->] (in2) to (or1-4);
\draw[->] (in3) to (or1-2);
\draw[->] (in3) to (or1-4);
\draw[->] (in4) to (or1-4);
\draw[->] (in4) to (or1-5);
%%\draw[->] (in5) to (or1-3);
%\draw[->] (in5) to (or1-4);
%%\draw[->] (in5) to (or1-5);
\draw[->] (in5) to (or1-5);
\draw[->] (in6) to (or1-5);
%\draw[->] (in7) to (or1-5);
%
\draw[->] (or1-1) to (and1-1);
\draw[->] (or1-2) to (and1-1);
\draw[->] (or1-2) to (and1-3);
\draw[->] (or1-2) to (and1-2);
%\draw[->] (or1-3) to (and1-2);
%\draw[->] (or1-3) to (and1-3);
\draw[->] (or1-4) to (and1-2);
%\draw[->] (or1-4) to (and1-4);
\draw[->] (or1-5) to (and1-3);
\draw[->] (and1-1) to (or2-1);
\draw[->] (and1-2) to (or2-1);
\draw[->] (and1-2) to (or2-2);
\draw[->] (and1-3) to (or2-2);
%\draw[->] (and1-4) to (or2-2);
%
\draw[->] (or2-1) to (and2);
\draw[->] (or2-2) to (and2);
\end{tikzpicture}
\quad
\begin{tikzpicture}
\scriptsize
\node at (4.7,4) {{\small (b)}};
\node[round-rect] (v10) at (5,0)  {$g_{11}$}; %[label=left:$a^1_3$]{};
\node[round-rect] (v6) at (5,1)    {$g_7$}; %[label=left:$b^1_2$]{};
\node[round-rect] (v11) at (6,0)   {$g_{13}$}; % [label=left:$a^2_3$]{};
\node[round-rect] (v7) at (6,1)    {$g_8$}; %[label=left:$b^2_2$]{};
\node[round-rect] (v4) at (5.5,2)  {$g_4$}; % [label=left:$a^1_2$]{};
\node[round-rect] (v2) at (5.5,3)  {$g_2$}; %[label=left:$b^1_1$]{};
\node[round-rect] (v1) at (6.5,4)  {$g_1$}; %[label=left:$a^1_1$]{};
\node[round-rect] (v3) at (7.5,3)  {$g_3$}; %[label=left:$b^2_1$]{};
\node[round-rect] (v5) at (7.5,2)  {$g_5$}; % [label=left:$a^2_2$]{};
\node[round-rect] (v8) at (7,1)   {$g_8$}; % [label=left:$b^3_2$]{};
\node[round-rect] (v9) at (8,1)   {$g_9$}; % [label=left:$b^4_2$]{};
\node[round-rect] (v12) at (7,0)  {$g_{13}$}; %  [label=left:$a^3_3$]{};
\node[round-rect] (v13) at (8,0)   {$g_{13}$}; % [label=left:$a^4_3$]{};
\draw[-] (v10) to (v6);
\draw[-] (v11) to (v7);
\draw[-] (v6) to (v4);
\draw[-] (v7) to (v4);
\draw[-] (v4) to (v2);
\draw[-] (v2) to (v1);
\draw[-] (v12) to (v8);
\draw[-] (v13) to (v9);
\draw[-] (v8) to (v5);
\draw[-] (v9) to (v5);
\draw[-] (v5) to (v3);
\draw[-] (v3) to (v1);
\node at (5,-.75) {};
\end{tikzpicture}}
\quad
\scalebox{.75}{
\begin{tikzpicture}[>=latex, point/.style={circle,draw=black,thick,minimum size=1.5mm,inner sep=0pt,fill=white},
spoint/.style={rectangle,draw=black,thick,minimum size=1.5mm,inner sep=0pt,fill=white},
ipoint/.style={circle,draw=black,thick,minimum size=1.5mm,inner sep=0pt,fill=lightgray},
wiggly/.style={thick,decorate,decoration={snake,amplitude=0.3mm,segment length=2mm,post length=1mm}},
query/.style={thick},yscale=1,xscale=1]
\footnotesize
\node at (4.5,4) {{\small (c)}};
%\node[fill=black] (a) at (1,3) [point, label=above:{$A$}, label=left:{$a$}]{};
%\node[fill=black] (c) at (2.5,3) [point, label=right:{$c$}]{};
%\node (d1) at (1,2) [point, fill=white, label=left:{$aP$}, label=right:{$B$}]{};
\node[ipoint] (u1) at (6,4) {};
\node[ipoint] (u2) at (5.25,3) {};
\node[ipoint] (u3) at (5,2) {};
\node[ipoint] (u4) at (4.75,1) {};
\node[ipoint,label=below:{$\trueleaf$}] (u5) at (4.75,0) {};
\node[ipoint] (u6) at (5.15,1) {};
\node[ipoint] (u7) at (5.5,2) {};
\node[ipoint] (u8) at (5.85,1) {};
\node[ipoint,label=below:{$\trueleaf$}] (u9) at (6.25,0) {};
\node[ipoint] (u10) at (6.25,1) {};
\node[ipoint] (u11) at (6,2) {};
\node[ipoint] (u12) at (5.75,3) {};
\node[ipoint] (u13) at (6.5,4) {};
\node[ipoint] (u14) at (7.25,3) {};
\node[ipoint] (u15) at (7,2) {};
\node[ipoint] (u16) at (6.75,1) {};
\node[ipoint,label=below:{$\trueleaf$}] (u17) at (6.75,0) {};
\node[ipoint] (u18) at (7.15,1) {};
\node[ipoint] (u19) at (7.5,2) {};
\node[ipoint] (u20) at (7.85,1) {};
\node[ipoint,label=below:{$\trueleaf$}] (u21) at (8.25,0) {};
\node[ipoint] (u22) at (8.25,1) {};
\node[ipoint] (u23) at (8,2) {};
\node[ipoint] (u24) at (7.75,3) {};
\node[ipoint] (u25) at (7,4) {};
\draw[<-] (u1) to node[left] {} %{$\leftand$} 
(u2);
\draw[<-] (u2) to node[right] {$\,\,\,\,\aor$} 
(u3);
\draw[<-] (u3) to node[left] {} (u4);
\draw[<-] (u4) to node[left] {$\aor$} (u5);
\draw[->] (u5) to node[right] {$\aor$} 
(u6);
\draw[->] (u6) to node[left] {$\!\!\!\!\leftand$} (u7);
\draw[<-] (u7) to node[right] {$\,\rightand$} 
(u8);
\draw[<-] (u8) to node[left] {$\aor$} (u9);
\draw[->] (u9) to node[right] {$\,\aor$} (u10);
\draw[->] (u10) to node {}%{$\!\!\!\!\!\!\!\rightand$} 
(u11);
\draw[->] (u11) to node {} %{$\!\!\!\!\!\!\!\!\!\!\!\aor$} 
(u12);
\draw[->] (u12) to node[left] {$\leftand\,\,$}(u13);
\draw[<-] (u13) to node {$\qquad\rightand$}(u14);
\draw[<-] (u14) to node[right] {$\,\,\,\,\aor$} 
(u15);
\draw[<-] (u15) to node {} % {$\leftand$}
(u16);
\draw[<-] (u16) to node {}% {$\aor$}
(u17);
\draw[->] (u17) to node[right] {$\,\aor$}(u18);
\draw[->] (u18) to node[left] {$\!\!\!\!\leftand$} % {$\leftand$}
(u19);
\draw[<-] (u19) to node[right] {$\,\rightand$} % {$\rightand$}
(u20);
\draw[->] (u21) to node[left] {$\aor$}(u20);
\draw[->] (u21) to node[right] {$\aor$}(u22);
\draw[->] (u22) to node {} % {$\rightand $}
(u23);
\draw[->] (u23) to node {}%{$\!\!\!\!\!\!\!\!\aor$}
(u24);
\draw[->] (u24) to node {} %{$\,\,\rightand$}
(u25);
\node at (7,-.75) {};
\end{tikzpicture}}
\quad
%\caption{An \sac circuit $\Circuit$ accepts some input iff it has a subgraph which is isomorphic to the tree in the middle with 1's on its leaves iff there is a homomorphism from the query in Figure~\ref{fig:5a} 
%to the unravelling of $\Circuit$.
%}
%\label{fig:5}
%\end{figure}
%\begin{figure}[t]
%\centering
\scalebox{.75}{
\begin{tikzpicture}[>=latex, point/.style={circle,draw=black,thick,minimum size=1.5mm,inner sep=0pt,fill=white},
spoint/.style={rectangle,draw=black,thick,minimum size=1.5mm,inner sep=0pt,fill=white},
ipoint/.style={circle,draw=black,thick,minimum size=1.5mm,inner sep=0pt,fill=lightgray},
wiggly/.style={thick,decorate,decoration={snake,amplitude=0.3mm,segment length=2mm,post length=1mm}},
query/.style={thick},yscale=1,xscale=1]
\footnotesize
\node at (-3,4) {{\small (d)}};
%\node[fill=gray!40,input,label=left:$g_{11}$] (in1) at (-2.7,0) {$x_1$}; 
%\node[input,label=left:$g_{12}\!$] (in2) at (-1.6,0) {$x_2$}; 
%\node[fill=gray!40,input,label=left:$g_{13}\!$] (in3) at (-.5,0) {$\!\!\neg x_3$}; 
%\node[input,label=left:$g_{14}\!$] (in4) at (.6,0) {$x_4$}; 
%\node[input,label=left:$g_{15}\!$] (in5) at (1.7,0) {$x_5$}; 
%\node[input,label=left:$g_{16}\!$] (in6) at (2.8,0) {$\!\!\neg x_1$}; 
%%\node[fill=gray!40,input,label=left:$g_7$] (in7) at (3.3,0) {$\!\!\neg x_2$}; 
%%
%\node at (0,-0.7) {\textsc{Input}: $\quad x_1=1 \quad x_2=0 \quad x_3=0\quad x_4=0 \quad x_5=0$};
%%\node (in1a) at (-3,-0.5) {$1$}; 
%%\node (in2a) at (-2,-0.5) {$0$}; 
%%\node (in3a) at (-1,-0.5) {$1$}; 
%%\node (in4a) at (0,-0.5) {$0$}; 
%%\node (in5a) at (1,-0.5) {$1$}; 
%%\node (in6a) at (2,-0.5) {$0$}; 
%%\node (in7a) at (3,-0.5) {$1$}; 
%%
\node[point,fill=white,label=left:$G_{11}\!$,label=below:{$\trueleaf$}] (in1) at (-2.6,0) {}; 
\node[point,fill=white,label=left:$G_{12}\!$] (in2) at (-1.8,0) {}; 
\node[point,fill=white,label=left:$G_{11}\!$,label=below:{$\trueleaf$}] (in3) at (-.4,0) {}; 
\node[point,fill=white,label=left:$G_{13}\!$,label=below:{$\trueleaf$}] (in4) at (.4,0) {}; 
\node[point,fill=white,label=left:$G_{14}\!$] (in5) at (1.4,0) {}; 
\node[point,fill=white,label=left:$G_{15}\!$] (in6) at (2.2,0) {}; 
\node[point,fill=white,label=left:$G_{16}\!$] (in7) at (3.1,0) {}; 
\node[point,fill=white,label=left:$G_7\!$] (or1-1) at (-2.2,1) {}; 
\node[point,fill=white,label=left:$G_8\!$] (or1-2) at (-1.3,1) {};
\node (dots1) at (-1.3, .7) {$\vdots$};
\node[point,fill=white,label=left:$G_8\!$] (or1-3) at (0,1) {}; 
\node[point,fill=white,label=left:$G_9\!$] (or1-4) at (.85,1) {}; 
\node (dots1) at (.85, .7) {$\vdots$};
\node[point,fill=white,label=left:$G_8\!$] (or1-5) at (1.55,1) {}; 
\node (dots1) at (1.55, .7) {$\vdots$};
\node[point,fill=white,label=left:$G_{10}\!$] (or1-6) at (2.4,1) {}; 
%%\node[or-gate,label=left:$g_{10}$] (or1-3) at (0,1) {OR}; 
%\node[fill=gray!40,or-gate,label=left:$g_{9}$] (or1-4) at (.75,1) {OR}; 
%\node[or-gate,label=left:$g_{10}$] (or1-5) at (2.2,1) {OR}; 
%%
\node[point,fill=white,label=left:$G_{4}\!$] (and1-1) at (-1.5,2) {};
\node[point,fill=white,label=left:$G_{5}\!$] (and1-2) at (-.5,2) {};
\node (dots1) at (-.5, 1.7) {$\vdots$};
\node[point,fill=white,label=left:$G_{5}\!$] (and1-3) at (.5,2) {};
\node[point,fill=white,label=left:$G_{6}\!$] (and1-4) at (1.5,2) {};
\node[point,fill=white,label=left:$G_{2}$] (g2) at (-.65,3) {}; 
\node[point,fill=white,label=right:$G_3$] (g3) at (.65,3) {}; 
\node[point,fill=black,label=left:$G_{1}$,label=above:$a$] (a) at (0,4) {}; 
\draw[->,wiggly] (g2)  to node [left]{\scriptsize $L\,\,$} (a);
\draw[->,wiggly] (g3)  to node [right]{\scriptsize $\,R$} (a);
\draw[->,wiggly] (and1-1)  to node [left]{\scriptsize $U\,$} (g2);
\draw[->,wiggly] (and1-2)  to node [right]{\scriptsize $\,U$} (g2);
\draw[->,wiggly] (and1-3)  to node [left]{\scriptsize $U\,$} (g3);
\draw[->,wiggly] (and1-4)  to node [right]{\scriptsize $\,U$} (g3);
\draw[->,wiggly] (or1-1)  to node [left]{\scriptsize $L\,$} (and1-1);
\draw[->,wiggly] (or1-2)  to node [right]{\scriptsize $R$} (and1-1);
\draw[->,wiggly] (or1-3)  to node [left]{\scriptsize $L$} (and1-3);
\draw[->,wiggly] (or1-4)  to node [right]{\scriptsize $R$} (and1-3);
\draw[->,wiggly] (or1-5)  to node [left]{\scriptsize $L$} (and1-4);
\draw[->,wiggly] (or1-6)  to node [right]{\scriptsize $\,R$} (and1-4);
\draw[->,wiggly] (in1)  to node [left]{\scriptsize $U\,$} (or1-1);
\draw[->,wiggly] (in2)  to node [right]{\scriptsize $\,U$} (or1-1);
\draw[->,wiggly] (in3)  to node [left]{\scriptsize $U\,$}  (or1-3);
\draw[->,wiggly] (in4)  to node [right]{\scriptsize $\,U$} (or1-3);
\draw[->,wiggly] (in5)  to node [right]{\scriptsize $\,U$} (or1-6);
\draw[->,wiggly] (in6)  to node [right]{\scriptsize $\,U$} (or1-6);
\draw[->,wiggly] (in7)  to node [right]{\scriptsize $\,U$} (or1-6);
\node at (1,-.75) {};
\end{tikzpicture}}
%
%\node[fill=black] (a) at (1,3) [point, label=above:{$A$}, label=left:{$a$}]{};
%\node[fill=black] (c) at (2.5,3) [point, label=right:{$c$}]{};
%\node (d1) at (1,2) [point, fill=white, label=left:{$aP$}, label=right:{$B$}]{};
%\node (d2) at (1,1) [point, fill=white, label=left:{$aPS$}]{};
%\node (d31) at (0.5,0) [point, fill=white, label=left:{$aPSR$}]{};
%\node (d32) at (1.5,0) [point, fill=white, label=right:{$aPST^-$}]{};
%%
%\draw[->,query] (a) to node[above] {$R$} (c);
%\draw[->,wiggly] (a)  to node [left]{\scriptsize $P, R, U^-$} (d1);
%
%\node[ipoint,label=above:{$y_1$}] (y1) at (1,2) {};
%\node[ipoint,label=left:{$y_{12}$}] (y12) at (1,1) {};
%\node at (-.3, 0){$\q_e$};
%\node at (4.7, 0){$\q_P$};
%%%\node[ipoint,label=right:{$y_6$}] (y6) at (2.5,0) {};
%%
%\draw[ultra thin,dashed,rounded corners=8] (-0.7,-0.25) rectangle +(4.15,2.75);
%%
%\draw[->,query] (y1) to node[left] {\scriptsize$S_{12}$} (y12);
%\draw[->,query] (y12) to node[left,near start] {\scriptsize$S_{12}'$} (y2);
%\draw[->,query] (y2) to node[left] {\scriptsize$S_{26}$} (y26);
%\draw[->,query] (y2) to node[right, near start] {\scriptsize$S_{23}$} (y23);
%
\caption{
%From left to right:
(a) Example circuit $\Cir^*$ with input $\vec{x}^*$ (b) proof tree for $\Cir^*$ and $\vec{x}^*$
(c) query $\q_{\mbox{\tiny $\Cir^*$}}^\mathsf{lin}$
(d) canonical model for KB $(\Tmc_{\mbox{\tiny $\Cir^*$}}^{\vec{x}^*}, \Amc_{\Cir^*})$. }
\label{fig:5a}
\vspace*{-.35cm}
\end{figure*}
The proof is by reduction from the 
%We show how to reduce the 
problem of deciding whether an 
input of length $l$ is accepted by the $l$th circuit of a \emph{logspace-uniform} family of 
\sac\ circuits (proven \LOGCFL-hard in \cite{DBLP:journals/jcss/Venkateswaran91}).  
%to the problem of answering linear queries over OWL 2 QL KBs. 
This problem was used in \cite{DBLP:journals/jacm/GottlobLS01} to establish the \LOGCFL-hardness 
of evaluating tree-shaped queries over databases.
We follow a similar approach, but with one crucial difference:
%the power of 
using an OWL 2 QL ontology, we can `unravel' the circuit into a tree, allowing us to replace 
tree-shaped queries by linear ones. 
% linear queries
%instead of tree-shaped ones.

%Alike the authors of 
As in \cite{DBLP:journals/jacm/GottlobLS01}, we assume w.l.o.g.\ that the considered \sac\ circuits
adhere to the following \emph{normal form}: %i.e., 
\begin{itemize}
\item fan-in of all AND gates is 2;
\item nodes are assigned to levels, with gates on level $i$ only receiving inputs from gates on level $i+1$;
\item there are an odd number of levels with the output \AND\ gate on level 1 and the input gates on the
greatest level;
\item all even-level gates are \OR\ gates, and all odd-level gates (excepting the circuit inputs)  are \AND\ gates.
\end{itemize}
It is well known (cf.\ \cite{DBLP:journals/jacm/GottlobLS01,DBLP:journals/jcss/Venkateswaran91}) and easy to see that a circuit %$\Cir$ 
in normal form 
accepts an input $\vec{x}$ iff there is a labelled rooted tree (called a \emph{proof tree})
with the following properties:
\begin{itemize}
\item the root node is labelled with the output \AND\ gate;
\item if a node is labelled by an \AND\ gate $g_i$,
then it has two children labelled by the two predecessor nodes of $g_i$;
\item if a node is labelled by an \OR\ gate $g_i$,
then it has a unique child that is labelled by a predecessor of $g_i$;
\item every leaf node is labelled by an input gate whose corresponding literal evaluates into 1 under $\vec x$.
\end{itemize}
%there is a subset of its gates $\PT$ (sometimes called a \emph{proof tree}) such that:
%\begin{itemize}
%\item the output \AND\ gate belongs to $\PT$;
%\item if an \AND\ gate is in $\PT$, then its two predecessor gates of its input gates are also in $\PT$;
%\item if an \OR\ gate is in $\PT$, then at least one of its predecessor gates is in $\PT$;
%\item if an input gate appears in $\PT$, then the corresponding literal evaluates into 1 under $\vec x$.
%\end{itemize}
% displayed in gray. \todo{update last}
%$\{x_1 = 1, x_2 = 0, x_3 = 0, x_4 = 0, x_5 = 1\}$, 
%and we can take $\PT$ to be the set of all gates coloured in gray.
For example, the circuit $\Cir^*$ in Fig.\ 6(a) accepts input $\vec{x}^*=(1,0,0,0,1)$,
%and this is 
as witnessed by the proof tree in Fig.\ 6(b).

Importantly, while a circuit-input pair may admit multiple proof trees, 
they are all isomorphic modulo the labelling. % differing only with respect to the labelling. 
Thus, with every circuit $\Cir$, we can associate a \emph{skeleton proof tree} $T_{\Cir}$
such that $\Cir$ accepts input $\vec{x}$ iff some labelling of $T_{\Cir}$ is a proof tree for 
$\Cir$ and~$\vec{x}$. % \cite{DBLP:journals/jacm/GottlobLS01}.
The reduction in \cite{DBLP:journals/jacm/GottlobLS01} encodes the circuit $\Cir$ 
and input $\vec{x}$ in the database
and uses a Boolean tree-shaped query based upon the %derived from the 
skeleton proof tree. %based upon $T_{\Cir}$. 
More precisely, 
the database $D_{\Cir}^\vec{x}$
uses the gates of $\Cir$ as constants and contains the following facts\footnote{For presentation purposes, we use a minor variant of the reduction in~\cite{DBLP:journals/jacm/GottlobLS01}.}:
\begin{itemize}
\item $\aor(g_j,g_i)$, for every OR gate $g_i$ with predecessor gate $g_j$;
\item $\leftand(g_j, g_i)$ (resp.\ $\rightand(g_j,g_i)$), for every AND gate $g_i$ with left (resp.\ right) predecessor~$g_j$;
%\item $R(g_i,g_j)$, for every AND gate $g_i$ with right predecessor $g_j$;
\item $A(g_i)$, for every input gate $g_i$ whose value is $1$ under $\vec{x}$. 
\end{itemize}
The query $\q_{\Cir}$ uses the nodes of $T_{\Cir}$ as variables, 
has an atom $\aor(n_j,n_i)$ (resp.\ $\leftand(n_j,n_i)$, $\rightand(n_j,n_i)$) 
for every node $n_i$ with unique (resp.\ left, right) child $n_j$,
and has an atom $\trueleaf(n_i)$ for every leaf node $n_i$.
It is proven in \cite{DBLP:journals/jacm/GottlobLS01} that $D_{\Cir}^\vec{x}\models \q_{\Cir}$ 
%just in the case
%that  
if and only if $\Cir$ accepts $\vec{x}$. Moreover, both $\q_{\Cir}$ and $D_{\Cir}^\vec{x}$
can be constructed by means of logspace transducers. 

To adapt the preceding reduction to our setting, 
% to work with linear queries. 
we will replace the tree-shaped query $\q_{\Cir}$ by a linear query 
$\q_{\Cir}^{\mathsf{lin}}$ that is obtained, intuitively, by performing
an ordered depth-first traversal of $\q_{\Cir}$. The new query 
$\q_{\Cir}^{\mathsf{lin}}$ may give a different answer than $\q_{\Cir}$
when evaluated on $D_{\Cir}^\vec{x}$, but the two queries coincide if evaluated on the 
\emph{unraveling of $D_{\Cir}^\vec{x}$ into a tree}. Thus, we will define 
a KB $(\Tmc_{\Cir}^{\vec{x}}, \Amc_{\Cir})$ whose canonical model 
induces a tree that is isomorphic to the tree-unravelling of $D_{\Cir}^\vec{x}$. 

To formally define the query $\q_{\Cir}^{\mathsf{lin}}$,
%and KB $(\Tmc_{\Cir}^{\vec{x}}, \Amc_{\Cir})$. 
%To formally define $\q_{\Cir}^\mathsf{lin}$, %suppose $\Cir$ has $2n+1$ levels 
consider the sequence of words %over $\{\aor, \leftand, \rightand, \aor^-, \leftand^-, \rightand^-\}^*$ 
inductively defined as follows:
$w_0 = \epsilon$ %(empty word) 
and  $w_{j+1} = \leftand^-\,\aor^-\,w_j\,\aor\,\leftand\,\rightand^-\,\aor^-\,w_j\,\aor\,\rightand$.
Every word $w =  \varrho_1 \varrho_2 \dots \varrho_k$ naturally gives rise to a linear query % is naturally associated with a linear query 
$\q_w %(y_0,y_1, \dots, y_k) 
=  \bigwedge_{i=1}^{k} \varrho_i(y_{i-1}, y_i).$
We then take 
$$\q_{\Cir}^\mathsf{lin} = \exists y_1 \dots \exists y_k(\q_{w_d}\land\bigwedge_{
%\substack{%R_iR_{i+1} = \aor^-\,\aor \mbox{\footnotesize in } w_n}
w_n[i,i+1] =\aor^-\,\aor} \trueleaf(y_i)).$$
where $k=|w_d|$ and $d$ is such that $\Cir$ has $2d+1$ levels. 
The query $\q_{\Cir^*}^\mathsf{lin}$ for our example circuit $\Cir^*$ is given in Fig.\ \ref{fig:5a}(c). 
%To illustrate,
%we display the query $\q_{\Cir^*}^\mathsf{lin}$ in Fig.\ \ref{fig:5a}(c). 

%To check for the existence of such a labelling, we create a KB %$(\Tmc_{\Cir}^{\vec{x}}, \Amc_{\Cir})$ %$(\Tmc_\Cir, \Amc$
%whose canonical model corresponds to the unravelling of the circuit $\Cir$ 
%and a linear query %$q_{\Cir}$ 
%that corresponds to an ordered depth-first traversal of the skeleton proof tree $T_{\Cir}$. 

We now proceed to the definition of the KB $(\Tmc_{\Cir}^{\vec{x}}, \Amc_{\Cir})$.
Suppose %that the circuit 
$\Cir$ has gates $g_1, g_2, \dots, g_m$, with $g_1$ %being
the output gate. %circuit output.
In addition to the predicates $\aor,\leftand,\rightand,\trueleaf$ from earlier, we introduce  
a unary predicate $G_i$ for each gate $g_i$
and a binary predicate $P_{ij}$ for each gate $g_i$ with predecessor $g_j$.
We set $\Amc_{\Cir} = \{G_1(a)\}$ and include in $\Tmc_{\Cir}^{\vec{x}}$ the following axioms:
\begin{itemize}
\item $G_{i}(x) \rightarrow \exists y P_{i j}(y,x)$ and $\exists y P_{i j}(x,y) \rightarrow G_j(x)$ for every gate $g_i$ with predecessor $g_j$;
\item $P_{i j}(x,y) \rightarrow S(x,y)$ for every $S \in \{\aor, \leftand,\rightand\}$ such that $S(g_j,g_i) \in D_{\Cir}^\vec{x}$;
%\item $P_{i j}(x,y) \rightarrow \aor(x,y)$ whenever $U(g_j,g_i) \in D_{\Cir}^\vec{x}$;
%%for all \OR\ gates $g_i$;
%\item $P_{i j}(x,y) \rightarrow \leftand(x,y)$ whenever $\leftand(g_j,g_i) \in D_{\Cir}^\vec{x}$;% for all \AND\ gates $g_i$ with left predecessor $g_j$;
%\item $P_{i j}(x,y) \rightarrow \rightand(x,y)$ whenever $\rightand(g_j,g_i) \in D_{\Cir}^\vec{x}$; % for all \AND\ gates $g_i$ with right predecessor $g_j$;
\item $G_i(x) \rightarrow \trueleaf(x)$ whenever $\trueleaf(g_i) \in D_{\Cir}^\vec{x}$.% for every input gate $g_i$ whose value is $1$ under $\vec{x}$.
\end{itemize}
%By construction, the canonical model of the KB $(\Tmc_{\mbox{\tiny $\Cir^*$}}^{\vec{x}^*}, \Amc_{\Cir^*})$
In Fig.\  \ref{fig:5a}(d), we display (a portion of) the canonical model of the KB associated with circuit $\Cir^*$ and input $\vec{x}^*$. 
% $(\Tmc_{\mbox{\tiny $\Cir^*$}}^{\vec{x}^*}, \Amc_{\Cir^*})$.
%for our example circuit $\Cir^*$ and input $\vec{x}^*$. 
Observe that, when restricted to the predicates $\aor,\leftand,\rightand,\trueleaf$,  %the canonical model
it is isomorphic to the unravelling of $D_{\Cir}^\vec{x}$ into a tree starting from $g_1$.

In the appendix, we argue $\q_{\Cir}^\mathsf{lin}$ and $(\Tmc_{\Cir}^{\vec{x}}, \Amc_{\Cir})$ can be constructed by logspace transducers,
and we prove the following proposition that establishes the correctness of the reduction.

\begin{proposition}\label{logcfl-lower-prop}
$\Cir$ accepts input $\vec{x}$ iff $\Tmc_{\Cir}^{\vec{x}}, \Amc_{\Cir} \models \q_{\Cir}^\mathsf{lin}(a)$.
\end{proposition}

\section{Conclusion}
In this paper, we have clarified the impact of query topology and ontology depth on the worst-case size of query rewritings and the complexity of query answering 
in OWL 2 QL. %, considering a range of tree-like queries and ontologies of varying depths. 
Our results close an open question from \cite{lics14-KKPZ} and yield 
a complete picture of the succinctness and complexity landscapes 
for the considered classes of queries and ontologies.

On the theoretical side, our results demonstrate the utility of 
using non-uniform complexity as a tool for studying succinctness.  
In future work, we plan to utilize the developed machinery to investigate additional dimensions
of the succinctness landscape, with the hope of identifying other natural restrictions 
%on queries and ontologies
 that guarantee small rewritings. 
%We speculate that our techniques can be fruitfully applied to studying succinctness in other logical settings. 

Our results also have practical implications for querying OWL 2 QL KBs. 
Indeed, our succinctness analysis provides strong evidence in favour of %suggest the interest of 
adopting NDL as the target language for rewritings, since we have identified a range of query-ontology 
pairs for which polysize NDL-rewritings are guaranteed, but PE-rewritings 
may be of superpolynomial size. Interestingly, we have proved that for 
these same classes of queries and ontologies, query answering is tractable 
(either in \NL\ or in \LOGCFL). 
We plan to marry these positive succinctness and complexity
results by developing concrete NDL-rewriting algorithms for OWL 2 QL for 
which both the rewriting and evaluation phases
run in polynomial time (as was done in \cite{DBLP:conf/ijcai/BienvenuOSX13} for DL-Lite$_{\mathsf{core}}$). 
Moreover, since \NL\ and \LOGCFL\ are considered highly parallelizable,
it would also be interesting to explore parallel query answering algorithms.%\vspace*{1mm}
%\smallskip 

%\noindent{\bf Acknowledgements}
\section*{Acknowledgment}
\noindent Theorems 13 and 14 were obtained at the IITP RAS
at the expense of the Russian Foundation for Sciences (project No. 14-50-00150).
The rest of the paper was partially funded % supported 
by ANR grant 12-JS02-007-01, the Russian Foundation for Basic Research
and the programme ``Leading Scientific Schools''.
%\vspace*{-3mm}
%\section*{Acknowledgment}
%HARDCODED BIB!


%UNCOMMENT TO RECOMPILE BIB!!!
%\bibliographystyle{IEEEtran}
%\bibliography{Bibliography}
%\end{document}

\newpage
\onecolumn

\section*{Proofs for Section \ref{sec:prelims}}

\medskip

\noindent\textbf{Theorem \ref{Hom2rew}.} Thm.\ \ref{TW2rew} remains true
 if $\twfn$ is replaced by $\homfn$: 
\begin{equation*}
\homfn = \!\!\!\!\bigvee_{\substack{\Theta \subseteq \twset\\ \text{ independent}}}
\bigg(\bigwedge_{\atom \in \q \setminus \q_\Theta} \!\!p_\atom
 \wedge  \bigwedge_{\t \in \Theta} \big(\!\!\bigwedge_{z,z'\in\t} \!\!p_{z=z'} \wedge\!\! \bigvee_{\substack{\varrho \in \rni,\\ \t \in \twset[\varrho]}}\!\! \bigwedge_{z \in \t} p_{z}^\varrho\big)\bigg)
\end{equation*}
\noindent{\bf Remark.}  In fact, Theorem~\ref{TW2rew} was proved in \cite{lics14-KKPZ} 
only for \emph{consistent} KBs. However, it is known that it is possible to define a  
short PE-query $\q_{\Tmc}^\bot$ that when evaluated on $\Imc_\Amc$
returns all $k$-tuples of individual constants on $\Imc_\Amc$ if 
the KB $(\Tmc,\Amc)$ is inconsistent, and returns no answers otherwise, cf.\ \cite{ACKZ09}. 
It follows that if %we have a query 
$\q'$ %that
 is a rewriting for $\q$ and $\Tmc$ for all 
data instances $\Amc$ that are consistent with $\Tmc$, then we can obtain a rewriting for $\q$  and $\Tmc$
(that works for all data instances) by taking the disjunction of $\q'$ and $\q_{\Tmc}^\bot$. 
Therefore, to prove Theorem \ref{Hom2rew}, it sufficient to show how to construct such ``consistent rewritings".

\begin{proof}
Let $\Tmc$ be an OWL 2 QL ontology and $\q= \exists \vec{y}\,\varphi(\vec{x},\vec{y})$ be a CQ with answer variables $\vec{x}$ and existential variables $\vec{y}$. 
(we will use $z$ and $z'$
when referring to variables of either type). We begin by recalling that every atom %and instance 
$\atom(\vec{u})$ has the following simple PE-rewriting: %of $\q$ 
%possesses the following  ``standard rewriting"
$$\rho_{\atom} %(\vec{u})
 = \bigvee_{\Tmc \models \xi(\vec{u}) \to \atom(\vec{u})} \xi(\vec{u})$$
where $\xi(\vec{u})$ ranges over $\varrho(\vec{u})$ ($\varrho \in \rni$) when $|\vec{u}|=2$
and over 
\begin{align*}
\tau(u) \ &::= \ \ A(u) \quad (A \in \cn) \quad \mid \quad \exists v\,\varrho(u,v) \quad (\varrho \in \rni) % \qquad (\text{where } A \in \cn) \\
%\varrho(x,y) \ & ::= \ \top \ \mid \ P(x,y) \ \mid \ P(y,x) \qquad (\text{where } P \in \rn). 
\end{align*}
when $\vec{u}$ consists of the single variable $u$. 
%
%over $\tau(\vec{u})$ (cf.\ Section \ref{prelims:kb}) when $|\vec{u}|=1$
%and over $\varrho(\vec{u})$ ($\varrho \in \rni$) when $|\vec{u}|=2$.

\medskip
To show the first statement, 
consider a Boolean formula $\chi$ that computes $\homfn$, and let $\q'$ be the FO-formula obtained from $\chi$ as follows:
\begin{itemize}
\item replace $p_{z=z'}$ by the equality $z = z'$;
\item replace $p_\atom$ by its PE-rewriting $\rho_\atom$; % of $\atom$ and $\Tmc$;
\item replace $p^{\varrho}_{z}$ by the PE-rewriting $\rho_\varrho(z)$ of $\exists y \varrho(z,y)$; % and $\Tmc$; and
\item existentially quantify the variables $\vec{y}$.
\end{itemize}
Note that $\q'$ has the same answer variables as $\q$, and if $\chi$ is a monotone formula, then $\q'$ is a PE-formula.

\smallskip

We wish to show that $\q'$ is a consistent rewriting of $\q$ and~$\Tmc$ (cf.\ preceding remark).
To do so, we let $\q''$ be the PE-formula obtained by applying the above transformation
to the original monotone Boolean formula $\homfn$:
$$\q'' = \exists \vec{y} \!\!\!\bigvee_{\substack{\Theta \subseteq \twset\\ \text{ independent}}}\,\,
\bigg(\bigwedge_{\atom \in \q \setminus \q_\Theta}\!\!\! \rho_\atom \,\,\wedge \bigwedge_{\t \in \Theta} (\bigwedge_{z,z'\in\t}\hspace{0cm} z=z' \wedge \bigvee_{\substack{\varrho \in \rni, \\Ê\t \in \twset[\varrho]}}  \bigwedge_{z \in \t} \rho_\varrho(z))\bigg).$$
We know that $\chi$ and $\homfn$
compute the same Boolean function.
It follows that $\q'$ and $\q''$ are equivalent FO-formulas.
It thus suffices to show that $\q''$ is a consistent rewriting of $\q$ and~$\Tmc$.
This is easily seen by comparing
$\q''$ to the following query 
$$\q'''=\exists \vec{y}' \!\!\!\bigvee_{\substack{\Theta \subseteq \twset\\ \text{ independent}}}\,\,
\bigg(\bigwedge_{\atom \in \q \setminus \q_\Theta}\!\!\! \rho_\atom \,\,\wedge \bigwedge_{\t \in \Theta}
(\!\!\bigvee_{\substack{\varrho \in \rni, \\ \t \in \twset[\varrho]}} \exists z \,(\,\rho_\varrho(z) \wedge \bigwedge_{z' \in \t_r} z'=z))\bigg)  $$
which was proven in \cite{lics14-KKPZ} to be a consistent FO-rewriting of $\q$ and $\Tmc$ (here $\vec{y}'$ is the restriction of $\vec{y}$ to the variables in $\q'''$).
%%
%Clearly, we have $\q'' \models \q'''$. For the other implication, 
%suppose that an interpretation $\Imc$ satisfies $\q'''$. Then there exists an independent subset $\Theta \subseteq \twset$
%such that:
%\begin{itemize}
%\item[(a)] $\Imc \models \rho_\atom$ for every $\atom \in \q $, and
%\item[(b)] for every $\t \in \Theta$, there exists $\varrho^\t \in \rni$ such that $\t \in \twset[\varrho^\t]$ and $\Imc \models \exists z \,(\,\rho_{\varrho^\t}(z) \wedge \bigwedge_{z' \in \t_r} z'=z))$
%\end{itemize}
%By choosing the same subset $\Theta$ and for each $\t$, the same 
%
%we observe that the 
%extra equalities in the $\q''$ contain only existentially quantified variables that do not appear in any atoms 
%(other than equalities) in the corresponding disjuncts and
%the extra $\rho_\varrho(z'')$ atoms are redundant because of equalities. \todo{check!!!}
%$$\bigwedge_{z,z' \in \t} z=z' \land \bigwedge_{z \in \t} \rho_\varrho(z)$$
%is equivalent to \todo{update this part of argument!!!}
%$$\exists z \,(\,\rho_R(z) \wedge \bigwedge_{z' \in \t_r} z'=z) =
%\begin{cases}
%\rho_\varrho(z'')\land \bigwedge_{z,z' \in \t_r} z=z'\mbox{ for some $z''\in\t_r$ if $\t_r \neq  \emptyset$}\\
%\exists z'' \rho_\varrho(z'') \mbox{if $\t_r = \emptyset$}
%\end{cases}
%$$
%as subformulas of $\q''$ and $\q'''$ since extra equalities in the former formula contain only existentially quantified variables that do not appear in any atoms (other than equalities) in the corresponding disjuncts and
%extra $\rho_\varrho(z'')$ atoms are redundant because of equalities.

\medskip

The proof of the second statement concerning NDL-rewritings closely follows the proof of Theorem \ref{TW2rew} from \cite{}, 
but we include it for the sake of completeness.
First, we define a unary predicate $D_0$ that contains all individual constants of the given data instance. This is done by taking the rules
\begin{equation}\label{d1}
\varrho(u) \rightarrow  D_0(u),
\end{equation}
where $\varrho(u)$ is of the form $S(u)$,
$S(u,v)$ and $S(v,u)$, %NOTE: typically don't use quantifiers in Datalog rule bodies
for some predicate $S \in \sig(\Tmc) \cup \sig(\q)$. 
% occurring in $\Tmc$ or $\q$ (we say that $\varrho$ and $S$ are \emph{in the signature of $\Tmc$ and $\q$}). 
%Intuitively, the interpretation of $D_1$ contains all the individual constants of the given data instance. 
Next, we let $\vec{z}=\vec{x} \cup \vec{y}$ and define a $|\vec{z}|$-ary predicate $D$ using the following rule:
\begin{equation}
\bigwedge_{z \in \vec{z}}D_0(z) \rightarrow D(\vec{z}).
\end{equation}
We need the predicate $D$ to ensure that all the rules in our \NDL\ program are safe, i.e.\ every variable that appears in the head of a rule also occurs in the body. 
%We will use $D(\vec{z})$ as a shorthand for $D_{|\vec{z}|}(|\vec{z}|)$.

\smallskip

Now let $\Cir$ be a monotone circuit for $\twfn$ whose gates are $g_1, \dots, g_n$, with $g_n$ the output gate. % be the gates of $\Cir$ ordered in such a way that $g_n$ is the output gate. 
For input gates $g_i$ whose variable is $p_{z=z'}$,  we take the rule\footnote{For ease of notation, we use equality atoms in rule bodies, but these can be removed using standard (equality-preserving) transformations.}
\begin{equation}\label{d3}
 z = z' \wedge D(\vec{z})  \rightarrow G_i(\vec{z}). 
\end{equation}
For every input gate $g_i$ whose variable is $p_{\atom}$, 
we include the rule 
\begin{equation}\label{d2}
 \xi \land D(\vec{z}) \rightarrow G_i(\vec{z}), 
\end{equation}
%NOTE: put \xi, not \xi(\vec{z}), since \xi can contain other variables than \vec{z} (which are existentially quantified in the PE-rewriting, but not here in the rule)
%ANOTHER NOTE: don't need to guard body with D since all head variables are guaranteed to appear in the body
for every disjunct  $\xi$ of the rewriting $\rho_{\atom}(\vec{z})$ of $\atom$ and $\Tmc$ (here we assume w.lo.g.\ that any variable 
in $\xi$ that does not appear in $\nu$ does not belong to $\vec{z}$).
If instead $g_i$ is associated with variable $p_z^\varrho$, then we use the rules
\begin{equation}\label{d2}
\xi \land D(\vec{z})  \rightarrow  G_i(\vec{z}), %\land D(\vec{z}), 
\end{equation}
where $\xi$ is a disjunct of the rewriting $\rho_\varrho$ of $\varrho(z)$ and $\Tmc$ (here again we assume that every variable that appears both in $\xi$ and $\vec{z}$ also appears in the atom $\varrho$). 
The remaining (AND and OR) gates are encoded using the following rules:
\begin{align}\label{d5}
G_{j_1}(\vec{z})  \land G_{j_2}(\vec{z})  \land  D(\vec{z}) \rightarrow G_i(\vec{z})  & \quad\text{ if } g_i = g_{j_1} \land g_{j_2};\\[3pt]
\label{d6}
\left.\begin{array}{ll}%
G_{j_1}(\vec{z}) \land D(\vec{z}) \rightarrow G_i(\vec{z}) \\[3pt]
G_{j_2}(\vec{z}) \land D(\vec{z}) \rightarrow G_i(\vec{z}) 
\end{array}\right\}
 &\quad \text{ if } g_i = g_{j_1} \lor g_{j_2}.
\end{align}
%where $\vec{z}$ is appropriately chosen to include those and only those variables on which
%$G_i$ actually depends. \todo{what does this mean?}
%

%\smallskip

Denote the resulting set of rules~\eqref{d1}--\eqref{d6} by $\Pi$. We note that $\Pi$ is of size $O(|\Cir| \cdot |\Tmc|)$ and further claim that $(\Pi, G_n)$ is an \NDL-rewriting of $\q$ and $\Tmc$. % of size $O(|\Cir| \cdot |\Tmc|)$.
To see why, observe that by ``unfolding" these rules in the standard way, we can transform $(\Pi, G_n)$ into an equivalent \PE-formula of the form
\begin{equation*}
\exists \vec{y} \, \Big[\psi(\vec{x}, \vec{y}) \land \bigwedge_{z \in Z} \  \Bigl(
\bigvee_{ \varrho(u) \rightarrow  D_0(u) \in \Pi
%\varrho \in \sig(\Tmc) \cup \sig(\q) 
} \hspace*{-1em}\varrho(z)
\Bigr)
 \Big],
\end{equation*}
where $Z \subseteq \vec{x} \cup \vec{y}$ and $\exists \vec{y} \,\psi(\vec{x}, \vec{y})$ can be constructed by taking the Boolean formula representing $\Cir$ and replacing $p_\atom$ with $\rho_\atom$, $p_{z=z'}$ with $z=z'$ and $p_z^\varrho$ with $\rho_\varrho$.
We have already shown that $\exists \vec{y} \,\psi(\vec{x}, \vec{y})$ is a rewriting of $\q$ and $\Tmc$ in the first half of the proof, and the additional conjuncts asserting that the variables in $\vec{z}$ appear in some predicate are trivially satisfied. \end{proof}

\bigskip

\noindent{\bf Theorem \ref{rew2prim}}
 If $\q'$ is a \textup{(}\PE-\textup{)} \FO-rewriting of $\q$ and $\Tmc$, then
there is a \textup{(}monotone\textup{)} Boolean formula $\chi$ of size $O(|\q'|)$ which computes
$\primfn$.
If $(\Pi, G)$ is an \NDL-rewriting of $\q$ and $\Tmc$, then $\primfn$ is computed by a
monotone Boolean circuit $\Cir$ of size $O(|\Pi|)$.

\begin{proof}[Proof (implicit in \cite{lics14-KKPZ}).]
Given a \PE-, \FO- or \NDL-rewriting $\q'$ of $\q$ and $\Tmc$, we show how to construct, respectively, a monotone Boolean formula, a Boolean formula or a monotone Boolean circuit for
the function $\primfn$ of size $|\q'|$.

\smallskip

Suppose $\q'$ is a \PE-rewriting of $\q$ and $\Tmc$. We eliminate the quantifiers in $\q'$ by first replacing
every subformula of the form $\exists x\, \psi(x)$ in $\q'$ with $\psi(a)$, and then
replacing each atom of the form $A(a)$ and $P(a,a)$ with the corresponding propositional variable.
One can verify that the resulting propositional monotone Boolean formula computes $\primfn$.
If  $\q'$ is an \FO-rewriting of $\q$, then we eliminate the quantifiers by replacing $\exists x\, \psi(x)$ and $\forall x\, \psi(x)$ in $\q'$ with $\psi(a)$.
We then proceed as before, replacing atoms $A(a)$ and $P(a,a)$ by the corresponding propositional variables, to obtain a %propositional 
Boolean formula
computing $\primfn$.

\smallskip

If  $(\Pi, G)$ is an \NDL-rewriting of $\q$, then
we replace all the % individual
variables in $\Pi$ with $a$ and then perform the replacement described above.
Denote the resulting propositional \NDL-program by $\Pi'$.
The program $\Pi'$ can now be transformed into a monotone Boolean circuit computing $\primfn$.
For every (propositional) variable $p$ occurring in the head of a rule in $\Pi'$,
we introduce an $\OR$-gate whose output is $p$ and inputs are the bodies of the rules with head $p$;
for each such body, we introduce an $\AND$-gate whose inputs are the propositional variables in the body.
\end{proof}

\section*{Proofs for Section \ref{sec:succ}}

\medskip

\noindent{\bf Theorem \ref{tree-hg-to-query}}.
Let $P=(H_P, \l_P)$ be a \THP.
For every input $\alpha$ for $P$,
 $f_P(\alpha) = 1$ iff \mbox{$f_{\q_P,\Tmc_P}^\primsuper(\vgamma)= 1$},
where $\vgamma$ is defined as follows:
$\vgamma(B_e) = 1$, $\vgamma(R_e)=\vgamma(R_e')=0$,
and $\vgamma(S_{ij})=\vgamma(S_{ij}')= \valpha(\l_P(\{v_i,v_j\}))$.

\begin{proof}
Consider a \THP\ $P=(H_P, \l_P)$ whose underlying tree $T$ has vertices  %based upon the tree hypergraph %program %$P$
%whose hypergraph
%$H= (V,E)$. Let $T$ be the underlying  that is
%based upon the tree $T$ whose vertices are 
$v_1, \ldots, v_n$, and 
%Let $v_1, \ldots, v_n$ be the vertices of $T$, and
 let $T^{\downarrow}$ be the directed tree obtained from $T$ by fixing one of its leaves $v_1$ as the root and orienting
edges away from~$v_1$. In what follows, we will say that a vertex $v \in V_T$ is an \emph{internal vertex in $e \in E_P$ (w.r.t.\ $T$)} if
it appears in $e$ and is neither a leaf nor a boundary vertex of $e$ w.r.t.\ $T$. 
Note that because we chose a leaf as root of $T^\downarrow$, we know that for every hyperedge $e$, the highest vertex in $e$ (according to $T^\downarrow$)
must be either a leaf or a boundary vertex of $e$. 
 
%Let $P=(H_P, \l_P)$ be a \THP\ based upon the tree $T= (V_T, E_T)$.
Take some $\valpha: L_P \to \{0,1\}$
and let $\vgamma$ be as defined in the theorem statement.
Define the corresponding data instance: $$\Amc_{\vgamma} =
 \{ B_e(a) \mid e \in E_P\} \cup \{ S_{ij}(a,a), S_{ij}'(a,a) \mid \vgamma(S_{ij})=\vgamma(S_{ij}')= \valpha(\l_P(\{v_i,v_j\}))= 1\}.$$

\smallskip

For the first direction, suppose that $f_P(\valpha ) = 1$.
Then we know that there exists $E' \subseteq E_P$ that is independent  and covers all zeros of $\valpha$.
To show $\primfnP(\vgamma)= 1$, we must show that $\Tmc_P, \Amc_{\vgamma} \models  \q_P$.
%We begin by observing that the additional requirement on the degree of boundary vertices in the definition of THGPs 
%implies that the vertices $V_T$ of $T$ can be partitioned into three groups based upon their position relative to $E'$:
%\begin{description}
%\item[Interior vertices:] those $v_i$ for which there exists exactly one $e \in E'$ such that
%$v_i$ is an interior vertex of $e$. %\tocheck{In this case we denote this $e$ by $e_i$.}
%\item[Boundary vertices:] those $v_i$ which are boundary vertices of one or many $e \in E'$.
%\item[\tocheck{Exterior vertices}:] those $v_i$ which do not appear in %belong or border 
%any $ e \in E'$.
%\end{description}
Define a mapping $h$ as follows:
\begin{itemize}
\item $h(y_i)=a R_e R_e'$ if $v_i$ is an internal vertex of $e \in E'$. % (i.e.\ it appears in $e$ and is neither a leaf nor a boundary vertex). 
Otherwise, $h(y_i)=a$. \smallskip
\item $h(y_{ij})= a R_e$ if %there is some $e \in E'$ such that
$\{v_i, v_j\} \in e$ and $e \in E'$. Otherwise, $h(y_{ij})=a$.
\end{itemize}
Note that $h$ is well-defined: since $E'$ is independent, different hyperedges in $E'$ cannot share internal vertices, and there can be at most one hyperedge $e \in E'$ that contains a given edge $\{v_i, v_j\}$.

It remains to show that $h$ is a homomorphism from $\q_P$ to   $\Cmc_{\Tmc_P, \Amc_{\vgamma}}$.
Consider a pair of atoms $S_{ij}(y_i, y_{ij}), S'_{ij}(y_{ij}, y_{j})$ in $\q_P$.
Then $(v_i, v_j) \in T^{\downarrow}$, so either $\valpha(\{v_i, v_j\}) = 1$ or
there is some $e \in E'$ such that $\{v_i, v_j\} \in e$.

In the former case, we have $\vgamma(S_{ij})=\vgamma(S_{ij}')= 1$,
so $\Amc_{\vgamma}$ contains $S_{ij}(a,a)$ and $S_{ij}'(a,a)$.
If there is no $e \in E'$ such that $\{v_i, v_j\} \in e$, % and $\vbeta(e)=1$, 
then $h(y_i)=h(y_{ij})=h(y_j)=a$, so the atoms $S_{ij}(y_i, y_{ij}), S'_{ij}(y_{ij}, y_{j})$ are satisfied by $h$.

Now consider the alternative in which $e = \la v_{k_1}, \ldots, v_{k_m} \ra \in E'$ is such that $\{v_i, v_j\} \in e$ and $e \in E'$. 
%Note  $a R_e$ and $a R_e R_e'$  are present in $\Imc^{\Tmc_H, \Amc_{\vgamma}}$. 
%We may assume without loss of generality that none of the vertices $v_{k_\ell}$ is internal % either a boundary vertex of $e$ w.r.t.\ $T$ or a leaf of $T$ 
Note that because boundary vertices must have degree 2 (recall that this condition is part of the definition of \THP s), we know that all boundary and leaf vertices of $e$
must be among $v_{k_1}, \ldots, v_{k_m}$. Moreover, we may assume without loss of generality that $v_{k_1}, \ldots, v_{k_m}$ are all either boundary vertices or leaves of $T$
(since any internal vertex $v_{k_\ell}$ can be dropped without changing the meaning of $e$).  
Note that this ensures that for all $v_i \in e$, $h(y_i) = a$ iff  $v_i \in \{ v_{k_1}, \ldots, v_{k_m} \}.$
% vertices from the specification of $e$ without changing the hyperedge). 
%We may assume that $v_{k_1}$ is the vertex in $e$ that is highest in $T^{\downarrow}$. 
%
%\tocheck{ Note that because the root of $T^{\downarrow}$ was chosen to be a leaf of $T$ and because of our requirement that all boundary vertices in  we can be sure that the ``top vertex" $v_{k_1}$ in $T$
%
%Since the root of $T^{\downarrow}$ is its leaf, it follows that for any hyperedge $e \in E'$ its top vertex is a boundary vertex. } %% LEAVES DON'T MATCH DEF OF BOUNDARY VERTEX
%
There are four possibilities to consider:
\begin{itemize}
\item Case 1: $\{v_i, v_j\} \subseteq \{v_{k_1}, \ldots, v_{k_m}\}$ \emph{(i.e.\ neither of $v_i$ and $v_j$ is internal)}. We know that the boundary vertices of $e$ have degree 2, so the only possibility is that $e=\{\{ v_i, v_j \}\}$. 
We therefore have $h(v_i)=h(v_j)=a$ and $h(v_{ij})=a R_e$, and the ontology $\Tmc_P$ contains $R_e(x,y) \to S_{ij}(x,y)$ and  $R_e(y,x) \to S_{ij}'(x,y)$.
\smallskip
\item  Case 2: $v_i \in \{v_{k_1}, \ldots, v_{k_m}\}$ but $v_j \not \in \{v_{k_1}, \ldots, v_{k_m}\}$.  \emph{ (i.e.\ $v_i$ is a boundary vertex or leaf and $v_j$ is internal)}\smallskip\\
We have $h(v_i)=a$, $h(v_{ij})=a R_e$, and $h(v_{j})=a R_e R_e'$, and the ontology contains
$R_e(x,y) \to S_{ij}(x,y)$ and  $R_e'(x,y) \to S_{ij}'(x,y)$. \smallskip
\item Case 3: $v_j \in \{v_{k_1}, \ldots, v_{k_m}\}$ but $v_i \not \in \{v_{k_1}, \ldots, v_{k_m}\}$\emph{ (i.e.\ $v_j$ is a boundary vertex or leaf and $v_i$ is internal)}\smallskip\\
Then we have $h(v_j)=a$, $h(v_{ij})=a R_e$, and $h(v_{i})=a R_e R_e'$, and the ontology contains
$R_e(y,x) \to S_{ij}'(x,y)$ and  $R_e'(y,x) \to S_{ij}(x,y)$. \smallskip
\item Case 4:  $\{v_i, v_j\} \cap \{v_{k_1}, \ldots, v_{k_m}\} = \emptyset$ \emph{(i.e.\ both are internal vertices)}. %\smallskip\\
Then we have $h(v_i)=h(v_j)=a R_e R_e'$ and $h(v_{ij})=a R_e$, and the ontology contains
$R_e(y,x) \to S_{ij}(x,y)$ and  $R_e'(x,y) \to S_{ij}'(x,y)$.
\end{itemize}
In all cases, we find that $h$ satisfies the atoms $S_{ij}(y_i, y_{ij}), S'_{ij}(y_{ij}, y_{j})$.
We can thus conclude that $h$ is indeed a homomorphism.

\medskip

For the other direction, suppose that $\primfnP(\valpha) = 1$. Then we have
$\Tmc_P, \Amc_{\vgamma} \models \q_P$, so there is a homomorphism
$h: \q_P \rightarrow  \Cmc_{\Tmc_P, \Amc_{\vgamma}}$. We wish to show that there exists
a subset of $E_P$ that is independent and covers all zeros of $\valpha$.
Let us define $E' $ as the set of all $e \in E$ such that $h^{-1}(a R_e) \neq \emptyset$ (that is, $a R_e$ is in the image of $h$).

To show that $E'$ is independent, we start by establishing the following claim:\medskip\\
\noindent\textbf{Claim}. If $h^{-1}(a R_e) \neq \emptyset$, $y_{ij} \in \vars(\q_H)$, and $\{v_i, v_j\} \in e$,
then $h(y_{ij})= a R_e$. \smallskip\\
\noindent\emph{Proof of claim}. Suppose that $h^{-1}(a R_e) \neq \emptyset$,
where $e = \la v_{k_1}, \ldots, v_{k_m} \ra \in E'$. We may assume w.l.o.g.\
that $v_{k_1}$ is the highest vertex in $e$ according to $T'$, and that none of $v_{k_1}, \ldots, v_{k_m}$ is an internal vertex.
Now pick some variable $z \in h^{-1}(a R_e)$
such that there is no $z' \in h^{-1}(a R_e)$ that is higher than $z$ in $\q_P$ (here we use the ordering
of variables induced by the tree $T^{\downarrow}$).
We first note that $z$ cannot be of the form $y_j$, since then there is an atom in $\q_P$
of the form $S_{j \ell}(y_j, y_{j\ell})$ or $S_{\ell j}'(y_{\ell j}, y_j)$, and $a R_e$ does not have
any outgoing $S_{j \ell}$ or $S_{\ell j}'^-$ arcs in $\Cmc_{\Tmc_P, \Amc_{\vgamma}}$. It follows that $z=y_{j \ell}$. By again considering
the available arcs leaving $a R_e$, we can further see that $\{v_j, v_\ell\} \in e$. We next wish to
show that $j= k_1$. Suppose that this is not the case. Then, we know that % \tocheck{since the root of $T^\downarrow$ was selected among the leaves of $T$, ...}
%we know that the highest vertex of $e$ is not internal, and hence must appear among the   and boundary vertices are required to have degree 2, 
there must exist some edge $\{v_p, v_j\} \in e$ such that $(v_p, v_j) \in T^\downarrow$. A simple examination of the axioms in $\Tmc_P$ shows that the only
way for $h$ to satisfy the atom $S_{j \ell}(y_j, y_{j \ell})$ is to map $y_j$ to $a R_e R_e'$.
It follows that to satisfy that atom $S_{p j}'(y_{p j}, y_j)$, we must have $h(y_{p j})=a R_e$. This contradicts our earlier assumption that $z=y_{j \ell}$ was a highest vertex in $h^{-1}(a R_e)$. We thus have $j= k_1$.
Now using a simple inductive argument on the distance from $y_{k_1}$, and considering the possible ways of
mapping the atoms of $\q_H$, we can show that
$h(y_{ij})= a R_e$ for every $\{v_i, v_j\} \in e$. (\emph{end proof of claim})
\medskip\\
\noindent Suppose that there are two distinct hyperedges $e,e' \in E'$ that have a non-empty intersection:
$\{v_i, v_j\} \in e \cap e'$. We know that either $y_{ij}$ or $y_{ji}$ belongs to $\vars(\q_P)$, and we can suppose w.l.o.g.\
that it is the former. We can thus apply the preceding claim to obtain $h(y_{ij})= a R_e = a R_{e'}$, a contradiction.
We have thus shown
that $E'$ is independent, and so it only remains to show it covers all zeros. To this end,
let $\{v_i,v_j\}$ be such that $\valpha(\{v_i,v_j\})=0$ and again suppose w.l.o.g.
that $y_{ij} \in \vars(\q_P)$. Then $\Amc_{\vgamma}$ does not contain
$S_{ij}(a,a)$, so the only way $h$ can satisfy the query atom $S_{ij}(y_i, y_{ij})$
is by mapping $y_{ij}$ to some element $a R_e$ such that $\{v_i,v_j\} \in e$.
It follows that there is some $e \in E'$ such that
$\{v_i,v_j\} \in e$, so all zeros of $\valpha$ are covered by $E'$.
We have thus shown that $E'$ is an independent subset of $E_P$
that covers all zeros of $\valpha$, and hence  we conclude that $f_P(\valpha) = 1$.
\end{proof}

\bigskip

\noindent{\bf Theorem~\ref{DL2THP}}.
Fix $\twidth \geq 1$ and $d \geq 0$. For every ontology $\Tmc$ of depth $\leq d$ %at most $d$
and CQ $\q$ of treewidth $\leq \twidth$,
there is a monotone $\THP$ that computes $\homfn$
of size
polynomial in $|\Tmc| + |\q|$.

\medskip

More specifically, we show the following:

\medskip

\noindent{\bf Proposition.}
For every ontology $\Tmc$ and CQ $\q$,  the $\THP$ $(H_{\q, \Tmc}, \l_{\q, \Tmc})$ defined in Section III.C
computes $\homfn$.
If $\Tmc$ has depth $d$ and $\q$ has treewidth $\twidth$, then $H_{\q, \Tmc}$ 
\begin{itemize}
\item contains at most $(2M +1)L$ vertices;
\item contains at most $L(M+M^2)$ hyperedges;
\item has labels with at most $(2|\Tmc|+ |\q| + 1) |\q|$ conjuncts.
\end{itemize}
where $L = (2|\q|-1)^2$ and $M= |W_d^t| \leq (2|\Tmc|)^d$. 
\begin{proof}
 Let $(T, \lambda)$ be the tree decomposition of $G_\q$ 
 of width~$\twidth$ that was used to construct the $\THP$ $(H_{\q, \Tmc}, \l_{\q, \Tmc})$.
We may assume  w.l.o.g.\ that $T$ contains
at most $(2|\q|-1)^2$ nodes, cf.\ \cite{books/sp/Kloks94}.
%Consider a TBox $\Tmc$, % of depth $d$,
%a CQ $\q$, and a tree decomposition $(T, \lambda)$
%of $G_\q$ of width $\twidth$.
Recall that to more easily refer to the variables in $\lambda(N)$,
we make use of functions $\lambda_1, \ldots, \lambda_\twidth$
such that $\lambda_i(N) \in \lambda(N)$ and $\lambda(N) = \cup_i \lambda_i(N)$.
Further recall that the formula $\homfn$ is defined as follows:
$$\homfn = \bigvee_{\substack{\Theta \subseteq \twset\\ \text{ independent}}}\,\,
\bigg(\bigwedge_{\atom \in \q \setminus \q_\Theta} p_\atom
 \wedge  \bigwedge_{\t \in \Theta} \big(\bigwedge_{z,z'\in\t} p_{z=z'} \,\, \wedge \bigvee_{\substack{\varrho \in \rni,\\ \t \in \twset[\varrho]}} \,\, \bigwedge_{z \in \t} p_{z}^\varrho\big)\bigg)  $$
where $\q_\Theta = \bigcup_\t \q_\t$.
Throughout the proof, we use $f_P$ to denote the function computed by the $\THP$ $(H_{\q, \Tmc}, \l_{\q, \Tmc})$.
Note that by definition $f_P$ uses exactly the same set of propositional variables as $\homfn$.
\bigskip

To show the first direction of the first statement, let $\vec{v}$ be a valuation of the variables
in $\homfn$ such that $\homfn(\vec{v})=1$.
Then we can find an independent subset $\Theta \subseteq \twset$ such that $\vec{v}$ satisfies the
corresponding disjunct of $\homfn$:
\begin{equation}\label{disjunct}
\bigwedge_{\atom \in \q \setminus \q_\Theta} p_\atom
 \wedge  \bigwedge_{\t \in \Theta} \big(\bigwedge_{z,z'\in\t} p_{z=z'} \,\, \wedge \bigvee_{\substack{\varrho \in \rni,\\ \t \in \twset[\varrho]}} \,\, \bigwedge_{z \in \t} p_{z}^\varrho\big)
\end{equation}
For every $\t \in \Theta$, we let $\varrho_\t$ be a role that makes the final disjunction hold.
Furthermore,
we choose some homomorphism $h_\t: \q_\t \rightarrow \Cmc_{\Tmc_{\varrho_\t}, \Amc_{\varrho_\t}}$, where
$\Tmc_{\varrho_\t}= \Tmc \cup \{A_{\varrho_\t}(x) \rightarrow \exists y \varrho_\t(x,y) \}$ and $\Amc_{\varrho_\t}= \{A_{\varrho_\t}(a)\}$.
Such homomorphisms are guaranteed to exist by the definition of tree witnesses.

Now for each node $N$ in the tree decomposition $T$,
we define $\vec{w}_N$ %, \sim_N) \in \bchoices$
by setting:
\begin{itemize}
\item $\vec{w}_N[j] = \emptyword$ if $\lambda_j(N) = z$ and either $z$ appears in an atom $\atom$ such that
$\vec{v}(p_\atom)=1$ or there is some $\t \in \Theta$ such that $z \in \t_r$.
\item $\vec{w}_N[j] = w$ if $\lambda_j(N) = z$ and there is some $\t \in \Theta$ such that $z \in \t_i$ and $h_\t(z) = a w$.
%\item $i \sim_N j$ whenever $\lambda_i(N)=x_k$, $\lambda_j(N)=x_\ell$, and $k \sim_\Theta \ell$
%the smallest equivalence relation on $\{1, \ldots, t\}$ such that $i \sim_N j$ whenever either
%(a) $\lambda_i(N)=\lambda_j(N)$, or (b) $\lambda_i(N)=x_k$, $\lambda_j(N)=x_\ell$, and $\vec{v}(p_{k=\ell}) = 1$
\end{itemize}
First note that $\vec{w}_N$ is well-defined since the independence of $\Theta$ guarantees that every variable in $\q$ can appear in $\t_i$ for
at most one tree witness $\t \in \Theta$. Moreover, every variable in $\q$ must either belong to an atom $\atom$ such that
$\vec{v}(p_\atom)=1$ or to an atom that is contained in $\q_\t$ for some $\t \in \Theta$.
%Finally, we note that $\sim_N$ is an equivalence
%relation because $\sim_\Theta$ is.

\smallskip

Next we show that $\vec{w}_N$ is consistent with the node $N$.
To show that the first condition holds, consider some atom $A(\lambda_i(N)) \in \q$
such that $\vec{w}_N[i] \neq \emptyword$. Then there must be a tree witness $\t \in \Theta$
such that $\lambda_i(N) \in \t_i$, in which case we have that %$\vec{w}_N[i] \neq \emptyword$
$h_\t (\lambda_i(N))=a \vec{w}_N[i]$.
Let $\varsigma \in \rni$ be the final symbol in $\vec{w}_N[i]$. Then since $h_\t$ is a homomorphism from $\q_\t$ into $\Cmc_{\Tmc_{\varrho_\t}, \Amc_{\varrho_\t}}$,
it must be the case that $\Tmc \models \exists y\, \varsigma(y,x) \to A(x)$. 

To show the second condition holds,
consider some atom $R(\lambda_i(N),\lambda_j(N)) \in \q$ such that either $\vec{w}_N[i] \neq \emptyword$
or $\vec{w}_N[j] \neq \emptyword$. We suppose w.l.o.g.\ that $\vec{w}_N[i] \neq \emptyword$ (the other case is handled analogously).
It follows from the definition of  $\vec{w}_N$
that there must exist a tree witness $\t \in \Theta$
such that $\lambda_i(N) \in \t_i$ and $h_\t (\lambda_i(N))=a \vec{w}_N[i]$.
Since $\lambda_i(N) \in \t_i$ and $R(\lambda_i(N),\lambda_j(N)) \in \q$, the definition of tree witnesses
ensures that $\lambda_j(N) \in \t$.
Because $h_\t$ is a homomorphism from $\q_\t$ into $\Cmc_{\Tmc_{\varrho_\t}, \Amc_{\varrho_\t}}$,
we know that one of the following must  hold:
\begin{itemize}
\item $\lambda_j(N) \in t_r$, $\vec{w}_N[j]=\emptyword$, and $\vec{w}_N[i]=\varsigma$ for some $\varsigma\in \rni$
such that $\Tmc \models \varsigma(y,x) \to  R(x,y)$
\item $\lambda_j(N) \in t_i$ and $\vec{w}_N[i]=\vec{w}_N[j] \cdot \varsigma$ for some $\varsigma\in \rni$
such that $\Tmc \models \varsigma(y,x) \to R(x,y)$
\item $\lambda_j(N) \in t_i$ and $\vec{w}_N[j]=\vec{w}_N[i] \cdot \varsigma$ for some $\varsigma\in \rni$
such that $\Tmc \models \varsigma(x,y) \to R(x,y)$
\end{itemize}
This establishes the second  consistency condition.

We must also show that the pairs associated with different nodes in $T$ are compatible.
To this end, consider a pair of nodes $N_1$ and $N_2$ and the corresponding tuples of words
$\vec{w}_{N_1}$ and  $\vec{w}_{N_2}$.
It is clear from the way we defined $\vec{w}_{N_1}$ and  $\vec{w}_{N_2}$
that if $\lambda_i(N_1)= \lambda_j(N_2)$, then we must have $\vec{w}_{N_1}[i] = \vec{w}_{N_2}[j]$.
%Next suppose that we have $\lambda_i(N_1)= \lambda_j(N_2)$ and $\lambda_{i'}(N_1)= \lambda_{j'}(N_2)$.
%Then $i \sim_{N_1} i'$ iff  $\lambda_i(N_1) \sim_\Theta \lambda_{i'}(N_1)$ iff
%$\lambda_j(N_2) \sim_\Theta \lambda_{j'}(N_2)$ iff  $j \sim_{N_2} j'$.
%We have thus shown that $(\vec{w}_{N_1}, \sim_{N_1})$ and  $(\vec{w}_{N_2}, \sim_{N_2})$ are
%compatible with the nodes $(N_1, N_2)$.

\smallskip

Now consider the set $E'$ of hyperedges in $H_{\q,\Tmc}$ that contains:
\begin{itemize}
\item for every $N_i$ in $T$, the hyperedge $E_i^k = \la u_{ij_1}^k, \ldots, u_{i j_n}^k\ra $, where $k$ is such that
$\xi_k = \vec{w}_{N_i}$, and $N_{j_1}, \ldots, N_{j_n}$ are the neighbours of $N_i$;
\item for every pair of adjacent nodes $N_i, N_j$ in $T$, the hyperedge
$E_{ij}^{k m} = \la v_{ij}^k, v_{ji}^m\ra $, where $k$ and $m$ are such that $\xi_k=\vec{w}_{N_i}$ and $\xi_m=\vec{w}_{N_j}$.
\end{itemize}
Note that the aforementioned hyperedges all belong to $H_{\q,\Tmc}$ since we showed that
each $\vec{w}_{N_i},$ is consistent with node $N_i$, and that
$\vec{w}_{N_i}$ and  $\vec{w}_{N_j}$ are
compatible with $(N_i, N_j)$ for all pairs of  nodes  $(N_i, N_j)$ in $T$.
It is easy to see that $E'$ is independent, since whenever we include $E_i^k$ or $E_{ij}^{k m}$, we do not include any $E_i^{k'}$ or $E_{ij}^{k' m}$ for $k' \neq k$.
To see why $E'$ covers all zeros, consider a vertex $F$ of $H_{\q,\Tmc}$  (= an edge in $T'$) that evaluates to 0 under $\vec{v}$.
There are several cases to consider:
\begin{itemize}
\item $F = \{N_i, u_{ij}^1\}$: $F$ is covered by the hyperedge in $E'$ of the form $E_i^k$
\item $F= \{v_{ij}^\ell, u_{ij}^{\ell +1}\}$: then $F$ is either covered by the hyperedge in $E'$ of the form $E_i^k$ (if $k \leq \ell +1$) or by the hyperedge
$E_{ij}^{k m}$ (if $k > \ell +1$)
\item $F = \{v_{ij}^\numtypes, v_{ji}^\numtypes\}$: then $F$ is covered  by the hyperedge in $E'$ of the form $E_{ij}^{k m}$
\item $F= \{u_{ij}^\ell, v_{ij}^\ell\}$ with $\xi_\ell = \vec{w}$: then $F$ is either covered by the hyperedge in $E'$ of the form $E_i^k$ (if $\ell< k$) or by the hyperedge $E_{ij}^{k m}$ (if $k > \ell$), or we have $\vec{w} =\vec{w}_{N_i}$. In the latter case,
 we know that $F$ is labeled by the conjunction of the following variables:
\begin{itemize}
\item $p_\atom$, if  %$\alpha \in \q$,
$\vars(\atom) \subseteq \lambda(N_i)$
and $\lambda_g(N_i) \in \vars(\atom)$ implies $\vec{w}[g]= \emptyword$
\item $p_z^\varrho$, if $\vars(\atom) = \{z\} $, $z = \lambda_g(N_i)$, and $\vec{w}[g]= \varrho w'$
\item $p_z^\varrho$, $p_{z'}^\varrho$, and $p_{z=z'}$, if $\vars(\atom) = \{z, z'\}$, $z = \lambda_g(N_i)$, $z' = \lambda_{g'}(N_i)$,
and either $\vec{w}[g]= \varrho w'$ or $\vec{w}[g']= \varrho w'$
\end{itemize}
Since $F$ evaluates to false, one of these variables must be be assigned 0 under $\vec{v}$.
First suppose that $p_\atom$ is in the label and $\vec{v}(p_\atom)=0$. Then since Equation \eqref{disjunct}
is satisfied, it must be the case that $\atom$ belongs to some $\q_\t$, but the fact that
$\lambda_g(N_i) \in \vars(\atom)$ implies $\vec{w}_{N_i}[g]= \emptyword$ means that all variables in $\atom$
must belong to $\t_r$, contradicting the fact that $\q_\t$ contains only atoms that have
at least one variable in $\t_i$.
Next suppose that one of $p_z^\varrho$, $p_{z'}^\varrho$, and $p_{z=z'}$ is part of the label and
evaluates to $0$ under $\vec{v}$. We focus on the case where these variables came from a role atom with distinct
variables, but the proof is entirely similar if $p_z^\varrho$ is present because of a unary atom (item 2 above).
Then we know that there is some atom $\atom$ with
$\vars(\atom) = \{z, z'\} $, $z = \lambda_g(N_i)$, $z' = \lambda_{g'}(N_i)$,
and either $\vec{w}[g]= \varrho w'$ or $\vec{w}[g']= \varrho w'$.
It follows that there is a tree witness $\t \in \Theta$ such that $z \in \t$.
This means that the atom $p_{z=z'}$ must be a conjunct of Equation \eqref{disjunct}, and so it
must be satisfied under $\vec{v}$.
Moreover, the fact that $\vec{w}[g]= \varrho w'$ or $\vec{w}[g']= \varrho w'$ means that
$\varrho_\t = \varrho$, so
the variables $p_z^\varrho$ and $p_{z'}^\varrho$ are also satisfied under $\vec{v}$,
contradicting our earlier assumption to the contrary.
\end{itemize}
We have thus shown that $E'$ is independent and covers all zeros under $\vec{v}$,
which means that $f_P(\vec{v})=1$.
\medskip

For the other direction, suppose that $f_P(\vec{v})=1$, i.e.\ there is an independent subset $E'$
of the hyperedges in $H_{\q,\Tmc}$ that covers all vertices that evaluate to 0 under $\vec{v}$.
It is clear from the construction of $H_{\q,\Tmc}$ that the set $E'$ must contain exactly one hyperedge
of the form $E_i^k$ for every node $N_i$ in $T$, and exactly one hyperedge of the form
$E_{ij}^{k m}$ for every edge $\{N_i, N_j\}$ in $T$. Moreover, if we have hyperedges
$E_i^k$ and $E_{ij}^{k' m}$ (resp.\ $E_j^m$ and $E_{ij}^{k m'}$), then it must be the case that $k=k'$ (resp.\ $m=m'$).
We can thus associate with every node $N_i$ the tuple $\vec{w}_{N_i} = \type_k$.
Since all zeros are covered,
we know that for every node $N_i$,
the following variables are
assigned to $1$ by $\vec{v}$:
\begin{itemize}
\item $p_\atom$, if  %$\alpha \in \q$,
$\vars(\atom) \subseteq \lambda(N_i)$
and $\lambda_g(N_i) \in \vars(\atom)$ implies $\vec{w}[g]= \emptyword$
\item $p_z^\varrho$, if $\vars(\atom) = \{z\} $, $z = \lambda_g(N_i)$, and $\vec{w}[g]= \varrho w'$ \hfill ($\star$)
\item $p_z^\varrho$, $p_{z'}^\varrho$, and $p_{z=z'}$, if $\vars(\atom) = \{z, z'\} $,
 $z = \lambda_g(N_i)$, $z' = \lambda_{g'}(N_i)$,
and either $\vec{w}[g]= \varrho w'$ or $\vec{w}[g']= \varrho w'$
%\item $p_{h=h'}$, if $x_h = \lambda_g(N_i)$, $x_{h'} = \lambda_{g'}(N_i)$, and $g \sim g'$
\end{itemize}
We know from the definition of  the set of hyperedges in $H_{\q,\Tmc}$ that
every $\vec{w}_{N_i}$ is consistent with $N_i$, and for adjacent nodes $N_i,N_j$,
pairs $\vec{w}_{N_j}$ and $\vec{w}_{N_j}$ are compatible.
Using the consistency and compatibility properties, and the connectedness condition of tree
decompositions, we can infer that the pairs assigned to \emph{any two nodes} $N_i, N_j$
in $T$ are compatible. Since every variable must appear in at least one node label,
it follows that we can associate a \emph{unique} word $w_z$ to every variable $z$ in $\q$.
Now let $\equiv$ be the smallest equivalence relation on the atoms of $\q$ that satisfies
the following condition:
$$\text{If $y \in \vars(\q)$, $w_y \neq \emptyword$, $y \in \atom_1$, and $y \in \atom_2$, then $\atom_1 \equiv \atom_2. $} $$
Let $\q_1, \ldots, \q_m$ be the queries corresponding to the equivalence classes of $\equiv$. It is easily verified that the queries $\q_i$
are pairwise disjoint. Moreover, if $\q_i$ contains only variables $z$ with $w_z= \emptyword$, then $\q_i$ consists of a single atom.
We can show that the remaining $\q_i$ correspond to tree witnesses:
\smallskip

\noindent\textbf{Claim}. For every $\q_i$ that contains a variable $y$ with $w_y \neq \emptyword$:
\begin{enumerate}
\item there is a role $\varrho_{i}$ such that every $w_y \neq \emptyword$ begins by $\varrho_{i}$
\item there is a homomorphism $h_{i}$ from $\q_i$ into $C_{\Tmc_i, \Amc_i}$ where
$\Tmc_i = \Tmc \cup \{A_{\varrho_{i}}(x) \rightarrow \exists y \varrho_{i}(x,y)\}$ and $\Amc_i= \{A_{\varrho_{i}}(a)\}$ (with $A_{\varrho_{i}}$ fresh)
\item there is a tree witness $\t^i$ for $\q$ and $\Tmc$ generated by $\varrho_i$ such that $\q_i= \q_{\t^i}$
\end{enumerate}
\emph{Proof of claim}.
From the way we defined $\q_i$, we know that there exists a sequence $Q_0, \ldots, Q_m$ of subsets of $\q$
such that $Q_0 = \{\atom_0\} \subseteq \q_i$ contains a variable $y_0$ with $w_{y_0}\neq \emptyword$,
$Q_m= \q_i$, and for every $1 \leq \ell \leq m$,
$Q_{\ell +1}$ is obtained from $Q_\ell$ by adding an atom $\atom \in \q \setminus Q_\ell$
that contains a variable $y$ that appears in $Q_\ell$ and is such that $w_y \neq \emptyword$.
By construction, every atom in $\q_i$ contains a variable $y$ with $w_y \neq \emptyword$.
Let $\varrho_i$ be the first letter of the word $w_{y_0}$, and for every $0 \leq \ell \leq m$,
let $h_\ell$ be the function mapping every variable $z$ in $Q_\ell$ to $a w_z$.

Statements 1 and 2 can be shown by induction. The base case is trivial.
For the induction step, suppose that at stage $\ell$, we know that every variable $y$ in $Q_\ell$ with
$w_y \neq \emptyword$ begins by $\varrho_i$, and that $h_\ell$ is a homomorphism of $Q_\ell$ into
the canonical model $C_{\Tmc_i, \Amc_i}$.
We let $\atom$ be the unique atom in $Q_{\ell+1} \setminus Q_\ell$.
Then we know that $\atom$ contains a variable $y$ that appears in $Q_\ell$ and is such that
$w_y \neq \emptyword$. If  $\atom= B(y)$, then Statement 1 is immediate.
For Statement 2, we let $N$ be a node in $T$ such that $\vars(\atom) \subseteq \lambda(N)$ (such a node must
exist by the definition of tree decompositions), and let $j$ be such that $\lambda_j(N)=y$.
 We know that $\vec{w}_N$ is consistent with $N$, so
 $w_y = \vec{w}_N[j] $ must end by a role $\varsigma$ with $\Tmc \models \exists y \varsigma(y,x) \to B(x)$,
which proves Statement 2.
Next consider the other case in which $\atom$ contains a variable other than $y$.
Then $\atom$ must be a role atom of the form $\atom= R(y,z)$ or $\atom=R(z,y)$.
We give the argument for the case where $\atom= R(y,z)$ (the argument for
$\atom=R(z,y)$ is entirely similar).
Let $N$ be a node in $T$ such that $\vars(\atom) \subseteq \lambda(N)$,
and let $j,k$ be such that $\lambda_j(N)=y$
and $\lambda_k(N)=z$.
We know that $\vec{w}_N$ is consistent with $N$,
so one of the following must hold:
\begin{itemize}
\item $\vec{w}_N[k]= \vec{w}_N[j] \cdot \varsigma$ with $\Tmc \models \varsigma(x,y) \to R(x,y)$
\item $\vec{w}_N[j] = \vec{w}_N[k] \cdot \varsigma$ with $\Tmc \models \varsigma(x,y) \to R(y,x)$
\end{itemize}
By definition, we have $w_y = \vec{w}_N[j]$ and $w_{z} =\vec{w}_N[k]$.
Since $w_y$ begins with $\varrho_i$, it follows that the same holds for $w_z$ unless $w_z = \emptyword$,
which shows Statement 1. Moreover, we either have (i)
$w_z = w_y \varsigma$ and $\Tmc \models \varsigma(x,y) \to R(x,y) $, or (ii)
$w_y = w_z \varsigma$ and $\Tmc \models \varsigma(x,y) \to R(y,x)$. In both cases, it is clear from the way we defined
$h_{\ell+1}$ that it is homomorphism from $Q_{\ell+1}$ to $C_{\Tmc_i, \Amc_i}$, so Statement 2 holds.

Statement 3 now follows from Statements 1 and 2, the definition of $\q_i$, and the definition of tree witnesses.
\emph{(end proof of claim)}
\smallskip

Let $\Theta$ consist of all the tree witnesses $\t^i$ obtained from the preceding claim.
As the $\q_i$ are known to be disjoint, we have that the set $\{\q_{\t^i} \mid \t^i \in \Theta\}$ is independent.
We aim to show that $\vec{v}$ satisfies the disjunct of $\homfn$ that corresponds to $\Theta$ (cf.\ Equation \eqref{disjunct}).
First consider some $\atom \in \q \setminus \q_\Theta$. Then we know that for every variable $z$ in $\atom$,
we have $w_z = \emptyword$.
Let $N$ be a node such that $\vars(\atom) \subseteq \lambda(N)$.
Then we know that $\lambda_g(N) \in \vars(\atom)$ implies $\vec{w}_{N}[g]= \emptyword$.
It follows from $(\star)$ that $\vec{v}(p_\atom)=1$.
Next consider a variable $p_{z=z'}$ such that there is an atom $\atom \in \t^i$
with $\vars(\atom)=\{z,z'\}$ such that $z \neq z'$. Then since $\atom \in \q_i$,
we know that either $w_{z} \neq \emptyword$ or
$w_{z'} \neq \emptyword$.  It follows from $(\star)$ that $\vec{v}(p_{z=z'})=1$.
Finally, consider some $p_z^{\varrho_i}$ such that $z \in \t^i$.
First suppose that there is a unary atom $B(z) \in \q_{\t^i}$.
Then we know that $w_{z} \neq \emptyword$, and so by the above claim,
we must have $w_{z} = \varrho_i w'$. It follows that there is a node $N$ in $T$
such that $z = \lambda_g(N)$ and $\vec{w}_N[g]= \varrho_i w'$. From $(\star$),
we can infer that $p_z^{\varrho_i}$ evaluates to $1$ under $\vec{v}$.
The other possibility is that there exists a binary atom $\atom \in \q_{\t^i}$
such that $\vars(\atom)=\{z,z'\}$. Let $N$ be a node in $T$
such that $z= \lambda_g(N)$ and $z' = \lambda_{g'}(N)$.
Since $\q_{\t^i}=\q_i$, we know that either $w_{z} \neq \varepsilon$
or $w_{z'} \neq \varepsilon$. From the above claim,
this yields $\vec{w}_N[g]= \varrho_i w'$ or $\vec{w}_N[g']= \varrho_i w'$.
We can thus apply $(\star)$ to obtain $\vec{v}(p_z^{\varrho_i})=1$.
To conclude, we have shown that $\vec{v}$ satisfies one of the disjuncts of $\homfn$, so $\homfn(\vec{v})=1$.

\bigskip

For the second statement of the theorem, we recall that the tree $T$ in the tree decomposition of $\q$ 
has at most $(2|\q|-1)^2$ nodes and that the set $W_d^t$ consists of all tuples of words $\{(w_1, \ldots, w_\twidth)\mid w_i \in (\rni \cap \sig(\Tmc))^*, |w_i|\leq d \}$.
To simplify the counting, we let $L=(2|\q|-1)^2$ and $M= |W_d^t| \leq (2|\Tmc|)^d$. 
%The number $M$ of elements in $\bchoices$ cannot exceed $%2^{\twidth^2}
%|\tqwords|^\twidth$.
The vertices of the hypergraph $H_{\q, \Tmc}$ correspond to the edges of $T'$, and there can be at most $L \cdot (2M+1)$ of them, since there can be no more than $L$ edges in $T$, and each is replaced by $2 M + 1$ new edges. 
The hyperedges of $H_{\q, \Tmc}$ are of two types: $E^k_i$ (where $1 \leq i \leq L$ and $1 \leq k \leq M$) and $E^{k m}_{i j}$ (where
$1 \leq i \leq L$ and $1 \leq k \leq M$). It follows that the total number of hyperedges cannot exceed $L(M+M^2)$. Finally, a simple examination of the labelling function shows that there can be at most $(2|\Tmc|+ |\q| + 1) |\q|$ conjuncts in each label.  
\end{proof}

\bigskip

\noindent{\bf Theorem \ref{thm:thp_vs_sac}.}
There exist
polynomials
$p, p'$ such that:
\begin{itemize}
\item Every function computed by a semi-unbounded fan-in %(monotone)
  circuit of size at most $\size$
and depth at most $\log \size$ is computable by a %(monotone) 
$\THP$ of size $p(\size)$.
\item Every function computed by a %(monotone) 
$\THP$ of size $\size$ is computable by a semi-unbounded fan-in %(monotone) 
circuit of size at most $p'(\size)$
and depth at most $\log p'(\size)$.
\end{itemize}
Both reductions preserve monotonicity. 

\subsection*{Proof of the First Statement of Theorem \ref{thm:thp_vs_sac}: $\THP$s can simulate $\SAC$ circuits}

Consider a semi-unbounded fan-in circuit $\Cir$ of size at most $\size$ and depth at most $\log \size$.
Denote its gates by $g_1, \ldots, g_\size$, where $g_\size$ is the output gate.
We define the $\AND$-depth of gates of $\Cir$ inductively.
For the $\AND$ gate $g_i$, its $\AND$-depth, $d(g_i)$, equals $1$ if on each path from $g_i$ to the input
there are no $\AND$ gates.
If there are $\AND$ gates on the paths from $g_i$ to the input,
consider one such gate $g_j$ with maximal $\AND$-depth,  and let $d(g_i) = d(g_j) + 1$.
For an $\OR$ gate $g_i$, we let $d(g_i)$ to be equal to the largest $\AND$-depth of an $\AND$ gate
on some path from $g_i$ to the input.
If there are no such $\AND$ gates, then the $\AND$-depth of $g_i$ is~$0$.

We denote by $S_i$ the set of $\AND$ gates of the circuit of $\AND$-depth $i$.
Note that since the depth of $\Cir$ is at most $\log \size$, we have that the $\AND$-depth
of its gates is also at most $\log \size$.
For each $\AND$ gate $g_i$ of the circuit, we distinguish its first and its second input.
We denote by $\leftt(g_i)$ the subcircuit computing the first input of $g_i$, that is,
the subcircuit consisting of the left input $g_j$ of $g_i$
and of all gates such that there is a path from them to $g_j$.
Analogously, we use $\rightt(g_i)$ to denote the subcircuit that computes the second
input of $g_i$.

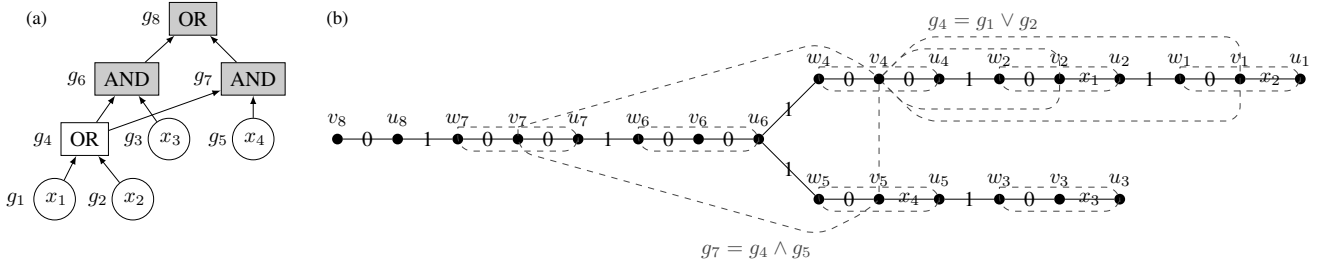
\begin{figure}[t]
\scalebox{.8}{
\begin{tikzpicture}[>=latex, point/.style={circle,draw=black,thick,minimum size=1.5mm,inner sep=0pt,fill=black}]
%\footnotesize
%\small
\node at (-3,3) {{\small (a)}};
\node at (2,3) {{\small (b)}};
\node[input,label=left:$g_{1}$] (g1) at (-2.7,0) {$x_1$};
\node[input,label=left:$g_{2}\!$] (g2) at (-1.4,0) {$x_2$};
\node[input,label=left:$g_{3}\!$] (g3) at (-.8,1) {$x_3$};
\node[input,label=left:$g_{5}\!$] (g5) at (.6,1) {$x_4$};
\node[or-gate,label=left:$g_4$] (g4) at (-2.2,1) {OR};
\node[fill=gray!40,and-gate,label=left:$g_{6}$] (g6) at (-1.5,2) {AND};
\node[fill=gray!40,and-gate,label=left:$g_{7}$] (g7) at (0.6,2) {AND};
\node[fill=gray!40,or-gate,label=left:$g_{8}$] (g8) at (-0.4,3) {OR};
\node[point,label=above: $v_{8}$] (v8) at (2,1) {};
\node[point,label=above: $u_{8}$] (u8) at (3,1) {};
\node[point,label=above: $w_{7}$] (w7) at (4,1) {};
\node[point,label=above: $v_{7}$] (v7) at (5,1) {};
\node[point,label=above: $u_{7}$] (u7) at (6,1) {};
\node[point,label=above: $w_{6}$] (w6) at (7,1) {};
\node[point,label=above: $v_{6}$] (v6) at (8,1) {};
\node[point,label=above: $u_{6}$] (u6) at (9,1) {};
\node[point,label=above: $w_{4}$] (w4) at (10,2) {};
\node[point,label=above: $v_{4}$] (v4) at (11,2) {};
\node[point,label=above: $u_{4}$] (u4) at (12,2) {};
\node[point,label=above: $w_{2}$] (w2) at (13,2) {};
\node[point,label=above: $v_{2}$] (v2) at (14,2) {};
\node[point,label=above: $u_{2}$] (u2) at (15,2) {};
\node[point,label=above: $w_{1}$] (w1) at (16,2) {};
\node[point,label=above: $v_{1}$] (v1) at (17,2) {};
\node[point,label=above: $u_{1}$] (u1) at (18,2) {};
\node[point,label=above: $w_{5}$] (w5) at (10,0) {};
\node[point,label=above: $v_{5}$] (v5) at (11,0) {};
\node[point,label=above: $u_{5}$] (u5) at (12,0) {};
\node[point,label=above: $w_{3}$] (w3) at (13,0) {};
\node[point,label=above: $v_{3}$] (v3) at (14,0) {};
\node[point,label=above: $u_{3}$] (u3) at (15,0) {};
\draw[->] (g1) to (g4);
\draw[->] (g2) to (g4);
\draw[->] (g4) to (g6);
\draw[->] (g3) to (g6);
\draw[->] (g4) to (g7);
\draw[->] (g5) to (g7);
\draw[->] (g6) to (g8);
\draw[->] (g7) to (g8);
\draw[-] (v8) to node {0} (u8);
\draw[-] (u8) to node {1} (w7);
\draw[-] (w7) to node {0} (v7);
\draw[-] (v7) to node {0} (u7);
\draw[-] (u7) to node {1} (w6);
\draw[-] (w6) to node {0} (v6);
\draw[-] (v6) to node {0} (u6);
\draw[-] (u6) to node {1} (w4);
\draw[-] (w4) to node {0} (v4);
\draw[-] (v4) to node {0} (u4);
\draw[-] (u4) to node {1} (w2);
\draw[-] (w2) to node {0} (v2);
\draw[-] (v2) to node {$x_1$} (u2);
\draw[-] (u2) to node {1} (w1);
\draw[-] (w1) to node {0} (v1);
\draw[-] (v1) to node {$x_2$} (u1);
\draw[-] (u6) to node {1} (w5);
\draw[-] (w5) to node {0} (v5);
\draw[-] (v5) to node {$x_4$} (u5);
\draw[-] (u5) to node {1} (w3);
\draw[-] (w3) to node {0} (v3);
\draw[-] (v3) to node {$x_3$} (u3);
\begin{scope}[rounded corners=7,ultra thin,dashed,gray!147]
\draw ($(v7)+(0,+0.2)$) -- ($(w4)+(0,0.6)$) -- ($(v4)+(-0.2,0.2)$) -- (v4) -- (v5) -- ($(v5)+(-0.2,-0.2)$) -- ($(w5)+(0,-0.6)$) node [anchor=north east] {$g_7 = g_4 \land g_5$} -- ($(v7)+(0,-0.2)$) -- ($(v7)+(0,+0.2)$);	
\foreach \x in {w1,w2,w3,w4,w5,w6,w7}
\draw ($(\x) +(0,+0.2)$) -- ($(\x) + (2,0.2)$) -- ($(\x) + (2,-0.2)$) --($(\x) + (0,-0.2)$) -- ($(\x) + (0,+0.2)$);	
\draw (v4) -- ($(v4)+(0.5,0.5)$) -- ($(v2)+(0,0.5)$) -- ($(v2)+(0,-0.5)$) --  ($(v4)+(0.5,-0.5)$) -- (v4);
\draw (v4) -- ($(v4)+(0.7,0.7)$) node [above right ] {$g_4 = g_1 \lor g_2$} -- ($(v1)+(0,0.7)$) -- ($(v1)+(0,-0.7)$) --  ($(v4)+(0.7,-0.7)$) -- (v4);
\end{scope}
\end{tikzpicture}
}
\caption{(a) Example circuit, with nodes from $S_1$ in white, nodes from $S_2$ in grey %the nodes from $S_1$ are white and from $S_2$ are grey; 
 (b) the tree underlying %skeleton of its 
 the corresponding THGP, together with labels and some hyperedges}
\label{fig:7}
\end{figure}

\begin{lemma}
Any semi-unbounded fan-in circuit $\Cir$ of size $\size$ and depth $d$ is equivalent to a semi-unbounded fan-in circuit of size $2^d \size$ and depth $d$ such that
for each~$i$
$$
\left(\bigcup_{g \in S_i} \leftt(g) \right) \bigcap \left(\bigcup_{g \in S_i} \rightt(g) \right) = \emptyset.
$$
\end{lemma}

The proof of this lemma is standard, but we include it for the sake of completeness.

\begin{proof}
We show by induction on $j$ that we can reconstruct the circuit in such a way that the property holds for all $i \leq j$,
the depth of the circuit does not change, and the size of the circuit increases at most by the factor of $2^j$.

Consider all $\AND$ gates in $S_j$, and consider a subcircuit $\bigcup_{g \in S_j} \leftt(g)$. Construct a copy $\Cir^{\prime}$ of this subcircuit separately and feed
its output as first inputs to $\AND$ gates in $S_j$. This at most doubles the size of the circuit and ensures the property for $S_j$.
Now for both circuits $\Cir^\prime$ and $\bigcup_{g \in S_j} \rightt(g)$ apply the induction hypothesis (note that the circuits do not intersect).
The size of both circuits will increase at most by the factor of $2^{j-1}$ and the property for $S_i$ for $i < j$ will be ensured.
\end{proof}

We can thus assume without loss of generality that the circuit $\Cir$ satisfies the property from the preceding lemma.

\smallskip

We now proceed to the construction of the $\THP$. We will begin by constructing a tree $T$ and then
afterwards define a hypergraph program based upon this tree. For each gate $g_i$ in $\Cir$, we introduce three vertices $w_i, v_i, u_i$,
and we arrange all these vertices into the tree from output gate to inputs. We construct the tree inductively from the root to leaves (see Figure~\ref{fig:7}). First we arrange vertices corresponding to the gates of $\AND$-depth $d$ into a path.
Vertices are ordered according to the order of gates in $\Cir$. In each triple of vertices, the $u$-vertex preceds the $v$-vertex, which precedes the $w$-vertex. `
% the $u$-vertex is below the $v$-vertex, which in turn %%  "below" doesn't match the picture which is horizontal
%is below the $w$-vertex.
%
Next we branch the tree into two branches at the last $u$-vertex and associate subcircuit $\bigcup_{g \in S_d} \leftt(g)$ to the left branch and the subcircuit of all other vertices
to the right branch. We repeat the process for each subcircuit.
This results in a tree, the number of vertices of which is $3\size$. We remove from this tree the vertex $w_\size$.

We now define a hypergraph program based upon this tree.
As before, the vertices of the hypergraph are the edges of the tree,
and the hyperedges will take the form of generalized intervals.
For each $i \neq \size$, we introduce a hyperedge $\langle w_i, u_i\rangle$.
For each ANG gate $g_i$ with $g_i = g_j \wedge g_k$, we add a hyperedge
$\langle v_j, v_k, v_i \rangle$.
For each OR gate $g_i = g_{k_1} \vee \ldots \vee g_{k_l}$,
 we add hyperedges $\langle v_{k_1}, v_i\rangle, \ldots, \langle v_{k_l}, v_i \rangle$.

For input gates, we label the corresponding $\{u,v\}$-edges by the corresponding literals
(recall that in the circuit $\Cir$ negations are applied only to the inputs,
so in this construction, we assume that the inputs are variables and their negations,
 and there are no $\NOT$ gates in the circuit).
We label all other $\{u,v\}$-edges and $\{v,w\}$-edges of the tree by $0$, and all
remaining edges are labelled by $1$.

\smallskip

The preceding construction clearly yields a $\THP$ of size polynomial in the original circuit, and the construction is monotonicity-preserving. 
To complete the proof of the first statement of
Theorem \ref{thm:thp_vs_sac}, we must show that the constructed $\THP$ computes the same function as the circuit. This is established by
the following claim:

\smallskip

\noindent\textbf{Claim}. For a given input $x$ and for any $i$, the gate $g_i$ outputs $1$ iff the subtree with the root $v_i$ can be covered (i.e.\ there is an independent subset of hyperedges that lies inside the subtree and  covers all of the zeros in the subtree).
\begin{proof}
We prove the claim by induction on $i$. For input gates, the claim is trivial.
If $g_i$ is an $\AND$ gate, then both its inputs output $1$. We cover both subtrees corresponding to the inputs (by induction hypothesis) and add a hyperedge
$\langle v_j, v_k, v_i \rangle$. This covers the subtree rooted in $v_i$.
If $g_i$ is an $\OR$ gate, then there exists an input $g_{k_j}$ of $g_i$ which outputs $1$.
By the induction hypothesis, we can find a cover of its subtree and then add a hyperedge $\langle v_{k_j}, v_i\rangle$. All other edges of the
subtree rooted in $v_i$ can be covered by hyperedges of the form $\{u_p, w_p\}$.
\end{proof}

\subsection*{Proof of  the Second Statement of Theorem \ref{thm:thp_vs_sac}: $\SAC$ circuits can simulate $\THP$s}

Now we proceed to the second part of Theorem~\ref{thm:thp_vs_sac}.
Suppose $P$ is a $\THP$ of size $\size$, and denote by $T$ its underlying tree.
We aim to construct a semi-unbounded fan-in circuit of size polynomial in $\size$.
We first describe the idea of the construction, then do some preliminary work, and finally,
detail the construction of the circuit. 

First of all, we note that it is not convenient to think about covering all zero vertices of the hypergraph, and it is more convenient to think about partitioning the set of all vertices into disjoint hyperedges. To switch to this setting, for each vertex $e$ of the hypergraph
(recall it is an edge of $T$), we introduce a hyperedge $\{e\}$. Thus we arrive at the following problem: 
\begin{description}
\item[(prob)]given a tree hypergraph $H=(V_H, E_H)$, a labelling of its hyperedges $\l$ and an input $\gamma$, decide if 
$V_H$ can be partitioned into disjoint hyperedges, whose labels are evaluated into 1 under $\gamma$.
\end{description}

Before we proceed, we need to introduce some notation related to the trees.
A vertex of a tree $T$ is called a \emph{branching point} if it has degree at least $3$.
A \emph{branch} of the tree $T$ is a simple path between two branching points which does not contain any other branching points.
If $v_1, v_2$ are vertices of $T$ we denote by $T_{v_1, v_2}$ the subtree of $T$ lying between the vertices $v_1$ and $v_2$.
If $v$ is a vertex of degree $k$ with adjacent edges $e_1, \ldots, e_k$ then it splits $T$ into $k$ vertex-disjoint subtrees
which we denote by $T_{v, e_1}, \ldots, T_{v,e_k}$.
We call a vertex of a subtree $T_1$ a \emph{boundary point} if it has a neighbour in $T$ outside of $T_1$.
The edges of $T_1$ adjacent to boundary points are called \emph{boundary edges} of $T_1$.
The \emph{degree} of a subtree $T_1$ is the number of its boundary points.
Note that there is only one subtree of $T$ of degree $0$ -- the tree $T$ itself.

Before we proceed with the proof of the theorem, we show the following technical lemma
concerning the structure of tree hypergraph programs.

\begin{lemma}
For any tree hypergraph program $H$ with underlying tree $T$, there is an equivalent tree hypergraph program
$H^\prime$ with underlying tree $T^\prime$ such that each hyperedge of $H^\prime$ covers at most one branching point of $T^\prime$,
and the size of $H^\prime$ is at most $p(|H|)$ for some explicit polynomial $p$.
\end{lemma}
%We sketch the proof of this lemma below.
\begin{proof}[Sketch]
Let $h_1, \ldots, h_l$ be the hyperedges of $H$ containing more than $2$ branching points.
Let $bp_1, \ldots, bp_l$ be the number of branching points in them.
We prove the lemma by induction on $bp = \sum_i bp_i$. The base case is when this sum is $0$.

For the induction step, consider $h_1$ and let $v$ be one of the branching points of $T$ in $h_1$.
Denote by $e_{1}, \ldots, e_{k}$ the edges adjacent to $v$ in $T$.
On each $e_i$ near vertex $v$, introduce two new adjacent edges $e_{i1}, e_{i2}$, edge $e_{i1}$ closer to $v$, and label them by $0$.
Let $v_i$ be the new vertex lying between $e_{i1}$ and $e_{i2}$. Break $h$ into $k+1$ hyperedges by vertices $v_i$
and substitute $h$ by these new hyperedges.
Add hyperedges $\{e_{i1}, e_{i2}\}$ for all $i$.
It is not hard to see that for each evaluation of variables there is a cover of all zeros in the original hypergraph iff
there is a cover of all zeros in the new hypergraph.
It is also not hard to see that $bp$ has decreased during this operation and the size of the hypergraph program has
increased by at most $2|T|$. So the lemma follows.
\end{proof}

Thus, in what follows, we can assume that each hyperedge of the hypergraph contains at most one branching point.

\medskip

Now we are ready to proceed with the proof of the theorem.
The main idea of the computation generalizes the polynomial size logarithmic depth semi-unbounded fan-in circuit
for directed connectivity problem (discussed below).

In what follows, we say that some subtree $T^\prime$ of $T$ can be partitioned into disjoint hyperedges if
there is a set of disjoint hyperedges $h_1, h_2, \ldots, h_k$ in $H$ such that they all lie in $T^\prime$,
$\l(h_i) = 1$ under $\gamma$ for $1 \le i \le k$, and their union contains all edges of $T^\prime$.
Fix $\gamma$.
Given the tree $T$ underlying the hypergraph,
we say that its vertices $v_1$ and $v_2$ are reachable from each other if
the subtree lying between them can be partitioned into disjoint hyperedges.
 In this case, we let $\Reach(v_1, v_2)=1$, otherwise we let $\Reach(v_1, v_2)=0$.
If $v$ is a vertex of $T$ and $e = \{v,u\}$ is an edge adjacent to it,
we say that $v$ is reachable from the side of $e$ if the subtree $T_{v,e}$ can be partition into disjoint hyperedges.
In this case, we let $\Reach(v, e)=1$, otherwise we let $\Reach(v, e) =0$.
Our circuit will gradually compute the reachability relation $\Reach$
for more and more vertices, and in the end, we will compute whether the whole tree can be partitioned into hyperedges.

First, our circuit will compute the reachability relation for vertices on each branch of the tree $T$.
If one of the endpoints of the branch is a leaf, we compute the reachability for the remaining vertex from the side containing the leaf.
This is done just like for the usual reachability problem.

Next we proceed to compute reachability between vertices on different branches of $T$.
For this, consider a tree $D$ those vertices are branching points and leaves of the original tree $T$ and those edges are branches of $T$.
In $D$, each vertex is either a leaf or a branching point.
We will consider subtrees of the tree $D$.

We describe a process of partitioning $D$ into subtrees. At the end of the process, all subtrees will be individual edges.
We have the following basic operation. Assume that we have already constructed a subtree $D^\prime$, consider some vertex $v \in D^\prime$
and assume that it has $k$ outgoing edges $e_1, \ldots, e_k$ within $D^\prime$, $e_i = \{v, v_i\}$.
By partitioning $D^\prime$ in the vertex $v$, we call a substitution of $D^\prime$ by a set of disjoint subtrees $D_1, \ldots, D_k \subseteq D^\prime$,
where for all $i$, we let $D_i = D_{v,e} \cap D^\prime$.

The following lemma helps us to apply our basic operation efficiently.
\begin{lemma}
Consider a subtree $D^\prime$ of size $m$. If its degree is $\leq 1$, then there is $v \in D^\prime$ partitioning it into subtrees of size at most $m/2 + 1$
and degree at most $2$ each.
If the degree of $D^\prime$ is $2$, then there is $v \in D^\prime$ partitioning it into subtrees of size at most $m/2+1$ and degree at most $2$ and possibly one subtree of size less than $m$ and degree~$1$.
\end{lemma}

\begin{proof}
If $D^\prime$ is of degree $\leq 1$, then consider its arbitrary vertex $v_1$ and subtrees into which this vertex divides $D^\prime$.
If among them there is a subtree $D_1$ larger than $m/2+1$, then consider the (unique) vertex $v_2$ in this subtree adjacent to $v_1$.
If we separate $D^\prime$ by $v_2$, then this partition will consist of the tree $D^\prime \setminus D_1 \cup \{v_1,v_2\}$ and other trees lying inside of $D_1$.
Thus, $D_1$ will be of size at most $m/2$ and other subtrees will be of size smaller than $|D_1|$.
Thus, the size of the largest subtree decreased, and we repeat the process until the size of the largest subtree becomes at most $m/2+1$.

If $D'$ has degree $2$, consider its boundary points $b_1$ and $b_2$.
Repeat the same process starting with $v_1=b_1$. Once in this process the current vertex $v$ tries to leave a path between $b_1$ and $b_2$, we stop.
For this vertex $v$, it is not hard to see that all the resulting trees are of degree at most $2$, and the only tree having
the size larger than $m/2$ is of degree $1$.
\end{proof}

With this lemma, the partitioning process works as follows.
We start with the partition $\{D\}$ consisting of the tree $D$ itself and repeatedly partition the current set of subtrees into smaller ones.
At each step, we repeat the described procedure for each subtree separately.
Note that after two steps the size of the largest subset decreases by the factor of $2$.
Thus in $O(\log \size)$ steps, we obtain the partition consisting of individual edges.

Now we are ready to describe the computational process of the circuit.
The circuit will consider tree partitions described above in the reversed order.
That is, we first have subtrees consisting of individual edges. Then on each step we merge some of them.
In the end, we obtain the whole tree.

The intuition is that along with the construction of a subtree $D_1$, we compute the reachability for its boundary edges,
that is, for example if the boundary edges of $D_1$ are $b_1$ and $b_2$ then we compute the reachability relation $\Reach(v_1, v_2)$
for all $v_1$ lying on the branch $b_1$ in $T$ and $v_2$ lying on the branch $b_2$ in $T$.

Now we are ready to describe the circuit. First for each branch of the tree the circuit computes
the reachability matrix for that branch. This is done by squaring the adjacency matrix $O(\log \size)$ times for all branches in parallel.
Note that squaring a matrix requires only bounded fan-in $\AND$-gates, and thus this step is readily computable by semi-unbounded circuit
of size polynomial in $\size$ and depth logarithmic in $\size$.

Thus the circuit computes the reachability matrix for the initial partition of $D$.
Next the circuit computes the reachability matrix for larger subtrees of $D$ following the process above.
More specifically, suppose we merge subtrees $D_1, \ldots, D_k$ meeting in the vertex $u$ to obtain a tree $D^\prime$.
For simplicity of notation, assume that there are two  subtrees $D_1$ and $D_2$ having degree $2$,
denote by $b, b^\prime$ the boundary edges of $D^\prime$
and by $b_1, \ldots, b_k$ the boundary edges of $D_1, \ldots, D_k$ respectively
adjacent to the vertex $u$.

It is not hard to see that for all vertices $v$ in $b$ and $v^\prime$ in $b^\prime$, it is true that
$$
\Reach(v, v^\prime) = \bigvee_{h \ni u} \left( \Reach(v,v_1) \wedge \Reach(v^\prime,v_2) \wedge \bigwedge_{i=3}^{k} \Reach(v_i, e_i)\land \l(h) \right),
$$
where $h$ ranges over all hyperedges of our hypergraph inside $D$, $v_1, \ldots, v_k$ are boundary vertices of $h$ lying in the branches $b_1, \ldots, b_k$ respectively, for each $i$ $e_i$ is an edge adjacent to $v_i$ and not contained in $h$.

The case when only one subtree among $D_1, \ldots, D_k$ has degree $2$ is analogous.

Thus we have described the circuit for solving (prob). Clearly that it is monotone provided that $P$ is monotone.
It is easy to see that its size is bounded by some fixed polynomial in $\size$.
It is only left to show that this circuit can be arranged in such a way that it has depth $O(\log \size)$.
This is not trivial since we have to show how to compute big $\AND$ in the formula above to make depth logarithmic.
For this, we will use the following lemma.
\begin{lemma}
Suppose the reachability relation for each branch is already computed.
Then the subtree $D^\prime$ with $m$ edges constructed on step $i$ can be computed in the $\AND$-depth at most $\log m + i$.
\end{lemma}

\begin{proof}
The proof proceeds by induction. Suppose that to construct $D^\prime$, we unite subtrees $D_1, \ldots, D_k$ of sizes $m_1, \ldots, m_k$ respectively.
Note that $m = m_1 + \ldots + m_k$.

By the induction hypothesis, we can compute each subtree $D_j$ having $\AND$-depth at most $\log m_j + i-1$.

Consider a $k$-letter alphabet $A = \{a_1, \ldots, a_k\}$ and assign to each letter $a_j$ the probability $m_j/m$.
It is well-known that there is a prefix binary code for this alphabet such that each letter $a_j$ is encoded
by a word of length $\lceil \log(m/m_j)\rceil$. This encoding can be represented by a rooted binary tree the leaves of which
are labeled by letters of $A$ and the length of the path from root to the leaf labeled by $a_j$ is equal to $\lceil \log(m/m_j)\rceil$.
Assigning the $\AND$ function to each vertex of the tree, we obtain the computation of $\AND$ in the formula above.
The depth of this computation is the maximum over $j$ of
$$
\log m_j + (i-1) + \lceil \log(m/m_j)\rceil \leq \log m_j + (i-1) + \log(m/m_j) + 1 = \log m + i.
$$
\end{proof}

From this lemma and the fact that the computation stops after $O(\log \size)$ steps, we obtain that
overall the $\AND$-depth of the circuit is $O(\log \size)$ and thus the overall depth is $O(\log \size)$.
This completes the proof of Theorem \ref{thm:thp_vs_sac}. 

\bigskip

\noindent{\bf Theorem~\ref{bbthp-to-nbp}.}
Fix $\ell \geq 2$.
For every  ontology $\Tmc$ and CQ $\q$ with at most $\ell$ leaves,
the function $\twfn$ is
computable by a monotone NBP of size polynomial in $|\q|$ and $|\Tmc|$.

\begin{proof}[Proof (continued).]
%Fix $\Tmc$ and CQ $\q$ with at most $\ell$ leaves.
Recall that in the main text we chose a root variable $v_0$ in the query $\q$. We then defined
flat sets $\Theta$ of tree witnesses, by requiring that every simple path
starting from $v_0$ intersect at most one tree witness of $\Theta$, and showed how flat sets could 
be ordered by the precedence relation $\prec$. This led us to construct the graph 
$G_P=(V_P, E_P)$ with
$V_P = \{u_\Theta, v_\Theta \mid \text{ flat } \Theta \subseteq \twset\} \cup\{s,t\}$ and
%edges
$E_P=\{(s, u_\Theta), (v_\Theta, t), (u_\Theta, v_\Theta) \mid  \text{ flat } \Theta \}
\cup \{(v_\Theta, u_{\Theta'}) \mid \text{ flat } 	\Theta \prec \Theta'\}$.

To formally define the labelling of the edges of $G_P$, we must first introduce some notation. % is more involved and requires some additional notation. 
For flat $\Theta \prec \Theta '$, we denote by $\mathsf{between}(\Theta, \Theta')$ the conjunction of $p_\atom$
for atoms $\atom$ 'between' $\Theta$ and $\Theta'$, that is those that lie outside of $\Theta$ and $\Theta'$
and are accessible from $\Theta$ via paths not passing through $\Theta'$ but are not accessible from $v_0$ via a path not passing through $\Theta$.
For flat $\Theta$, we denote by $\mathsf{before}(\Theta)$ the conjunction of $p_\atom$ for query
atoms $\atom$ which lie outside of $\Theta$ and are accessible from $v_0$ via paths not passing through $\Theta$.
By $\mathsf{after}(\Theta)$, we denote the conjunction of $p_\atom$  for atoms $\atom$
outside $\Theta$ which are accessible from $v_0$ only via paths passing through $\Theta$.
Now we are ready to define the labelling:
\begin{itemize}
\item edges of the form  $(u_\Theta, v_\Theta)$ are labelled $\bigwedge_{\t\in\Theta}p_\t$;
\item edges of the form $(s, u_{\Theta})$ are labelled with
$\mathsf{before}(\Theta)$;
\item edges of the form  $(v_{\Theta}, u_{\Theta'})$ for $\Theta \prec \Theta'$
are labelled with $\mathsf{between}(\Theta, \Theta')$;
\item edges of the form
$(v_{\Theta}, t)$ are labelled with $\mathsf{after}(\Theta)$.
\end{itemize}

\begin{figure}[t]
\centering
\scalebox{1.2}{\includegraphics{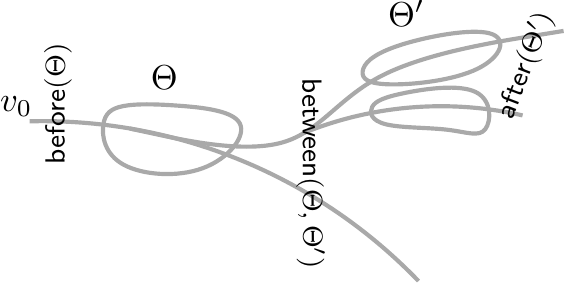}}
\caption{In the above diagram, $\Theta$ consists of one tree witness, and $\Theta'$ consists of two tree witnesses. Both are independent and flat, and $\Theta$ precedes $\Theta'$. %Now 
Here $\mathsf{before}(\Theta)$ is the segment between $v_0$
and $\Theta$,
$\mathsf{between}(\Theta, \Theta')$  comprises the segment between $\Theta$ and $\Theta'$
as well as the downwards branch that exits $\Theta$, and $\mathsf{after}(\Theta')$ consists of the two branches that leave $\Theta'$ on the side furthest from 
$v_0$.}
\label{fig:8}
\end{figure}

We claim that under any valuation of $p_\t$ and $p_\atom$, the vertex $t$ is accessible from $s$ if and only if
there is an independent subset $\hat \Theta \subseteq \twset$ (not necessarily flat) such that
$p_\t = 1$ for all $\t \in \hat\Theta$ and $p_\atom = 1$ for all atoms $\atom$ outside $\q_{\hat\Theta}$.
Indeed, any such $\hat \Theta$ splits  into flat ``layers'' $\Theta^1, \Theta^2, \dots \Theta^m$ which form a path
$s\to u_{\Theta^1} \to v_{\Theta^1} \to u_{\Theta^2} \to \cdots \to v_{\Theta^m} \to t$ in $G_P$ and whose edge labels evaluate to 1:
take $\Theta^1$ to be the set of all edges from $\hat \Theta$ that are accessible from $v_0$ via paths
which do not cross (that is come in and go out) any hyperedge of $\hat \Theta$; take  $\Theta^2$ to be the set of all edges from $\hat \Theta \setminus \Theta^1$ which are accessible from $v_0$ via paths
which do not cross any hyperedge of $\hat \Theta \setminus \Theta^1$, and so on.
Conversely, any path leading from $s$ to $t$ gives us a covering $\hat \Theta$ which is the union
of all flat sets that occur in the subscripts of vertices on this path.
\end{proof}

\bigskip

\noindent{\bf Theorem \ref{linear-lower}.}
There is a sequence of linear CQs $\q_n$ and ontologies $\Tmc_n$ of depth 2, both of polysize in $n$,
such that any PE-rewriting of $\q_n$ and $\Tmc_n$ is of size $n^{\Omega(\log n)}$.

\begin{proof}
It is known that there is a sequence $f_n$ of monotone Boolean functions that are computable by polynomial-size
monotone NBPs, but all monotone Boolean formulas computing $f_n$ are of size $n^{\Omega(\log n)}$, e.g., $s$-$t$-reachability in a directed graph \cite{DBLP:conf/stoc/KarchmerW88}.

Apply Theorem \ref{nbp-to-ihp} to the sequence $f_n$ mentioned above to obtain
a sequence of interval hypergraph programs $P_n$ based on interval hypergraphs $H_n$
which compute the functions $f_n$.
By Theorem~\ref{tree-hg-to-query}, there exist CQs $\q_n$ and ontologies
$\Tmc_n$ of depth 2 such that %$f_n$ is a subfunction of $\primfnn$. 
$f_n(\alpha) = 1$ iff \mbox{$f_{\q_n,\Tmc_n}^\primsuper(\vgamma)= 1$}, %$f_{\q_P, \Tmc_P}^{hom}(\gamma)=1$, 
where $\vgamma$ is defined as follows: 
$\vgamma(B_e) = 1$, $\vgamma(R_e)=\vgamma(R_e')=0$, % for all $e \in E$,
and $\vgamma(S_{ij})=\vgamma(S_{ij}')= \valpha(\l_P(\{v_i,v_j\}))$. 
We know from the construction that
$\q_n$ is a linear CQ and $\q_n$ and $\Tmc_n$ are both of polynomial size in $n$.
Since  $f_n$ is obtained from  $\primfnn$ through a simple substitution, the lower bound
$n^{\Omega(\log n)}$ still holds for $\primfnn$. It remains to apply Theorem~\ref{rew2prim}
to transfer this lower bound to PE-rewritings of $\q_n$ and~$\Tmc_n.$
\end{proof}

\bigskip

\noindent{\bf Theorem~\ref{bbcq-ndl}.}
Fix a constant $\ell > 1$.  Then
all tree-shaped CQs with at most $\ell$ leaves and arbitrary ontologies have polynomial-size NDL-rewritings.
%any tree-shaped CQ $\q$ with at most $\ell$ leaves and OWL 2 QL ontology $\Tmc$ has a polynomial-size NDL-rewriting.

\begin{proof}
Fix $\ell > 1$. By Theorem~\ref{bbthp-to-nbp}, there exists a polynomial $p$ such that for every tree-shaped CQ $\q$ with
at most $\ell$ leaves and every ontology $\Tmc$, there is a monotone NBP of size at most $p(|q| + |\Tmc|)$
that computes $\twfn$. We also know from \cite{Razborov91} that there is a polynomial $p'$
such that every function $f_P$ given by a monotone
NBP $P$ can be computed by a monotone Boolean circuit $\Cir_P$
of size at most $p'(P)$.
By composing these two translations, we obtain polysize monotone Boolean circuits
that compute the functions $\twfn$, for the class of tree-shaped CQs with at most $\ell$ leaves.
 It then remains to  apply Theorem~\ref{TW2rew}.
\end{proof}

\bigskip

\noindent{\bf Theorem~\ref{nbps-conditional}.}
The following are equivalent:
\begin{enumerate}
\item There exist polysize FO-rewritings for all linear CQs and depth $2$ ontologies;
\item There exist polysize FO-rewritings for all tree-shaped CQs with at most $\ell$ leaves and arbitrary ontologies \mbox{(for any fixed $\ell$);}
\item There exists a polynomial function $p$ such that every NBP of size at most $s$ is
computable by a formula of size $p(s)$. Equivalently,
%All functions in \nlpoly\ are computed by polysize Boolean formulas, that is,
$\nlpoly\subseteq \ncone$.
\end{enumerate}

\begin{proof}\ \\ $(2) \Longrightarrow (1)$: Trivial.\smallskip\\
$(1) \Longrightarrow (3)$:
Suppose (1) holds. In other words,  there exists a polynomial $p$ such that
any linear query $\q$ and an ontology $\Tmc$ of depth 2 have a rewriting of the size $p(|\q|+|\Tmc|)$.
Consider a sequence of functions $f_n$ computing s-t-reachability in directed graphs, which
is known to be $\nlpoly$-complete under $\ncone$-reductions \cite{wigderson1992complexity} (This function takes the adjacency matrix
of an undirected graph $G$ on $n$ vertices with two distinguished vertices $s$ and $t$ and returns 1
iff $t$ is accessible from $s$ in $G$.) Clearly, the functions $f_n$ are computed by a sequence of polynomial-size NBPs $P_n$.
Theorem \ref{nbp-to-ihp} gives us a sequence of hypergraph programs $P'_n$ which compute the $f_n$.
By Theorem~\ref{tree-hg-to-query}, % (and the construction above it) for each $n$
 there exist CQs $\q_{n}$ and ontologies $\Tmc_{n}$ such that
$f_n(\alpha) = 1$ iff \mbox{$f_{\q_n,\Tmc_n}^\primsuper(\vgamma)= 1$}, %$f_{\q_P, \Tmc_P}^{hom}(\gamma)=1$, 
for the valuation $\vgamma$ defined as follows: 
$\vgamma(B_e) = 1$, $\vgamma(R_e)=\vgamma(R_e')=0$, % for all $e \in E$,
and $\vgamma(S_{ij})=\vgamma(S_{ij}')= \valpha(\l_P(\{v_i,v_j\}))$. 
 % $f_n$ is a subfunction of $\primfnn$.
 By assumption, they have PE-rewritings $\q'_n$ of size $p(|\q|+|\Tmc|)$ which
is polynomial in $n$. Theorem~\ref{rew2prim} gives us a polysize Boolean formula
for computing $\primfnn$. Since $f^n$ is obtained from $\primfnn$ by some substitution,
it follows that we have a polysize formula for $f_n$, hence for all functions in $\nlpoly$.\smallskip\\
$(3) \Longrightarrow (2)$:
Suppose (3) holds. Fix some $\ell > 1$. Take a tree-shaped query $\q$ with at most $\ell$ leaves and an ontology $\Tmc$.
Since $\ell$ is fixed, by Theorem~\ref{bbthp-to-nbp}, there is a polysize
NBP $P$ which computes $\twfn$. By assumption, there is a polysize
FO formula computing $\twfn$, and Theorem~\ref{TW2rew} transforms it into
a FO-rewriting of $\q$ and $\Tmc$.
\end{proof}

\bigskip

\noindent{\bf Theorem~\ref{btw-ndl}.}
Fix constants  $\twidth >0$ and $d > 0$.
Then all CQs of
treewidth $\leq t$
 and ontologies of depth $\leq d$
have polysize
NDL-rewritings.

\begin{proof}
Fix constants $\twidth >0$ and $d >0$.
By Theorem~\ref{DL2THP}, we have that there is a polynomial $p'$ such that for any CQ $\q$ of treewidth
 at most $\twidth$ and any ontology $\Tmc$ of depth at most $d$ the $\THP$ $P$ computes
$\homfn$ and is of size at most $p'(|\q| + | \Tmc|)$. Now we apply Theorem~\ref{thm:thp_vs_sac}
and conclude that $\homfn$ may be computed by a polysize semi-unbounded fan-in circuit.
Therefore, by Theorem~\ref{Hom2rew}, there exists a polysize NDL-rewriting for $\q$ and $\Tmc$.
\end{proof}

\bigskip

\noindent{\bf Theorem~\ref{btw-fo}.}
The following are equivalent:
\begin{enumerate}
\item There exist polysize FO-rewritings for all tree-shaped CQs and depth 2 ontologies;
\item There exist polysize FO-rewritings for all CQs of treewidth at most $\twidth$ and ontologies of depth at most $d$
(for fixed constants $\twidth> 0$ and $d>0$);
\item There exists a polynomial function $p$ such that every semi-unbounded fan-in circuit of size at most $\size$ and depth at most $\log \size$ is
computable by a formula of size $p(\size)$. Equivalently,
$\sac\subseteq \ncone$.
\end{enumerate}

\begin{proof}\ \\ $(2) \Longrightarrow (1)$: Trivial.\smallskip\\
$(1) \Longrightarrow (3)$:
Suppose (1) holds. In other words,  there exists a polynomial $p''$ such that
every tree-shaped query $\q$ and ontology $\Tmc$ of depth 2 has a rewriting of the size $p''(|\q|+|\Tmc|)$.
Consider a semi-unbounded fan-in circuit $\Cir$ of size $\size$ and depth
at most $\log \size$ that computes the Boolean function $f$.
By Theorem~\ref{thm:thp_vs_sac}, $f$
is computed by a $\THP$ $P$ based on a tree hypergraph $H$ of size at most $p(\size)$ for the polynomial $p$ from Theorem~\ref{thm:thp_vs_sac}.
By Theorem~\ref{tree-hg-to-query}, there exists a tree-shaped query $\q_P$ and an ontology $\Tmc_P$ of depth 2  such
that $f$ is straightforwardly obtained from $\primfn$ via substitution. By the assumption, there exists an FO-rewriting for
$\q_P$ and $\Tmc_P$ of size at most $p''(|\q|+|\Tmc|)$. This number is polynomial in $\size$
(take the composition of $p$, the polynomial function from Theorem \ref{tree-hg-to-query} and $p''$).
Now by Theorem \ref{rew2prim}, there exists a polysize first-order formula for computing
$\primfn$ and hence also for $f$.\smallskip\\
$(3) \Longrightarrow (2)$:
Suppose (3) holds. Fix $\twidth > 0$ and $d > 0$.
Take query $\q$ of treewidth at most $\twidth$  and an ontology $\Tmc$ of depth at most $d$.
Since $\twidth$ is fixed, by Theorem~\ref{DL2THP}, there is a polysize
$\THP$ $P$ that computes $\homfn$. By assumption and Theorem~\ref{thm:thp_vs_sac}, there is a polysize
FO-formula computing $\homfn$. We can then apply Theorem~\ref{Hom2rew} to transform it into
an FO-rewriting of $\q$ and $\Tmc$.
\end{proof}

\bigskip

\noindent{\bf Theorem~\ref{depth-one-btw}.}
Fix  $t>0$. Then there exist polysize PE-rewritings for all CQs of treewidth at most $t$
and depth 1 ontologies.

\begin{proof}
Fix a constant $t>0$. Throughout the proof, we will consider CQs of treewidth at most~$t$,
and for every such query $\q$, we  will use $\tds(\q)$ (for `minimal tree decomposition')
to denote the minimum number of vertices over all tree decomposition of $\q$ that have width at most $t$.
%(i.e. the tree decomposition size  we denote by $\tds(\q)$ (tree decomposition size)
%we denote the minimal size of its tree decomposition with bags of size $d$. %a CQ $\q$,

We also fix an ontology $\Tmc$ of depth 1, and as before, we use
$\TW$ to denote the set of all tree witnesses for the query $\q$ and ontology $\Tmc$.
%For a CQ $\q$ by $\tds(\q)$ (tree decomposition size) we denote the minimal size of its tree decomposition with bags of size $d$.
%Fix $\q$ and $\Tmc$ of depth 1.
Since $\Tmc$ has depth 1, it is known from \cite{lics14-KKPZ} that every tree witness $\t=(\t_r, \t_i) \in \TW$ %for $\q$ and $\Tmc$
contains a unique interior point (i.e.\ $|\t_i|=1$), and no two distinct tree witnesses may share the same interior point. % (if $\t \neq \t'$, then $\t_i \neq \t_i'$).

Given a set $U \subseteq \vars(\q)$ of variables,
we will use $\TW(U)$ to refer to the tree witnesses whose interior point belongs to $U$.
If $\Omega$ is an independent subset of  $\TW(U)$, then the set $\bvars(U, \Omega)$ of border variables
for $U$ and $\Omega$ is defined as follows:
%we denote the set %
%Let $U \subseteq \vars(\q)$, $\TW(U)$ be the set of all tree witnesses for $\q$ and $\Tmc$ with the interior point from $U$ and let $\Omega$ be an independent subset of  $\TW(U)$. By $\bvars(U, \Omega)$ (border variables)
%we denote the set
$$\{u \in U \mid \mbox{ there is no } \t \in \Omega \mbox{ with } \t_i = \{u\}\} \cup \{ z \mid z \in \t_r
\mbox{ for some } \t \in\Omega\}.$$
We also define $\qwo = \q \setminus \{\atom \in \q \mid \vars(\atom) \subseteq
U \cup \bvars(U, \Omega)\}$.
For $z, z' \in \vars(\q) \setminus ( U \cup \bvars(U, \Omega))$, we set $z \sim z'$ if there is a
path in $G_\q$
from $z$ to $z'$ that does not pass through $\bvars(U, \Omega)$ (recall that $G_\q$ is undirected).
The $\sim$ relation can be lifted to the atoms in $\qwo$ by setting
$\atom_1 \sim \atom_2$ if all of the variables in $(\atom_1 \cup \atom_2) \setminus \bvars(U, \Omega)$ are $\sim$-equivalent.
Let $\q^{U, \Omega}_1$, \dots, $\q^{U, \Omega}_k$ denote the queries formed by the
$\sim$-equivalence classes of atoms in  $\qwo$ with $\avars(\q^{U, \Omega}_i) =
(\avars(\q) \cup \bvars(U, \Omega)) \cap \vars(\q^{U, \Omega}_i)$.
\smallskip

\noindent{\bf Claim.} For every query $\q$ of treewidth $t$, %and td-size $h$
there exists a subset $U$ of $\vars(\q)$ such that
$|U| \le t+1$ and for all subqueries $\q_i^{U, \Omega}$, we have $\tds(\q_i^{U, \Omega}) < \tds(\q) / 2 $.
\smallskip

\noindent\emph{Proof of claim.}
Consider some tree decomposition ($T, \lambda$) of $\q$ of width $t$ with $T=(V,E)$ and $|V| = \tds(\q)$.
It was shown in \cite{lics14-KKPZ} that there exists a vertex $v \in V$ such
that each connected component in the graph $T_{v}$ obtained by removing $v$ from $T$ has at most ${|V|}/{2}= \tds(\q)/2$ vertices.
Consider the set of variables $U= \lambda(v)$. Since ($T, \lambda$) has width $t$, we have that $|U| \leq t +1$.
As to the second property, we observe that it is sufficient to consider
queries of the form $\q_i^{U, \emptyset}$, since by definition,
$\q_i^{U, \Omega} \subseteq \q_i^{U, \emptyset}$ for any $\Omega \subseteq \TW(U)$.
We then remark that due to the connectedness condition on tree decompositions,
and the fact that we only consider paths in $G_\q$ that do not pass by variables in $U= \lambda(v)$,
every query $\q_i^{U, \emptyset}$ can be obtained by:
\begin{enumerate}
\item taking a connected component $C_i = (V_i, E_i)$ in the graph $T_v$;
\item considering the resulting tree decomposition $(C_i, \lambda_i)$, where $\lambda_i$ is the restriction of $\lambda$
to the vertices in $C_i$;
\item taking a subset of the atoms in $\{\atom \in \q \mid \vars(\atom) \subseteq \lambda(v') \text{ for some } v' \in V_i\}$.
\end{enumerate}
It then suffices to recall that each connected component of $T_v$ contains at most $\tds(\q)/2$ vertices.
(\emph{end proof of claim})
\medskip

We now use the claim to define a recursive rewriting procedure. Given a query $\q$
of treewidth at most $t$, we choose a set  $U$ of variables
that satisfies the properties of preceding claim, and define the rewriting $\q^\dagger$ of $\q$ w.r.t.\ $\Tmc$ as follows:
$$\q^\dagger = \exists \vec{y} \bigvee_{\substack{\Omega \subseteq \TW(U)\\Ê\text{ independent}}}
\left (\mathsf{at}(U,\Omega) \land
\mathsf{tw}(U,\Omega) \land \bigwedge_{i=1}^k (\qwo_i)^\dagger \right )$$
where
\begin{itemize}
\item $\vec{y}$ is the set of all existential variables in $U$,
\item $\mathsf{at}(U, \Omega) = \{ \atom \in \q \mid \vars(\atom) \subseteq U \mbox { and there is no }
\t \in \Omega \mbox{ with } \atom \in \q_\t\}$;
\item $$\mathsf{tw}(U, \Omega) =  \bigwedge_{\t \in \Omega} % \mathsf{tw}_\t$$ where
%$$\mathsf{tw}_\t(\tr) =
\left (\exists z' (\bigvee_{\substack{\t \in \Theta^\q_\Tmc[\varrho]\\ \varrho \in\rni }}  \rho_\varrho(z')) \land \bigwedge_{z \in \tr} (z = z') \right);$$
\item $(\qwo_i)^\dagger$ are rewritings of the queries $\qwo_i$, which are constructed recursively according to the same procedure.
\end{itemize}

The $\dagger$-rewriting we have just presented generalizes the rewriting procedure
for tree-shaped queries from \cite{lics14-KKPZ}, and correctness can be shown similarly
to the original procedure.

As to the size of the obtained rewriting, we remark that
since the set $U$ is always chosen according to the claim, and $\tds(\q) \leq (2|\q|-1)^2$
(cf.\ \cite{books/sp/Kloks94}), we know that the depth of the recursion is logarithmic in $|\q|$.
We also know that the branching is at most $2^{t+1}$
at each step, since there is at most one recursive call for each subset $\Omega\subseteq\TW(U)$,
and since $\Tmc$ has depth 1, we have $\TW(U)\leq |U|$.   %\mbox{\small ind } \Omega \subseteq \TW(U)
Thus, the resulting formula $\q^\dagger$ has the structure of a tree whose number of nodes is bounded by
$O((2^{t+1})^{\log |\q|}) = O(|\q|^{t+1})$.

\end{proof}

\section*{Proofs for Section \ref{sec:complexity}}
\medskip

\noindent{\bf Proposition \ref{logcfl-upper-prop}.} \emph{Every execution of \bbarbalgo\ terminates. 
There exists an execution of \bbarbalgo\ that returns \yes\ on input $(\Tmc,\Amc, \q, \vec{b})$ just in the case that $\Tmc, \Amc \models \q(\vec{b})$. }
\begin{proof}
%We now proceed to the formal proof that the procedure \bbarbalgo\ is as required. 
The following two claims, which can be easily seen to hold by examination of the procedure  and straightforward induction,  
resume some important properties of \bbarbalgo. 

\medskip

\noindent\textbf{Claim 1}. Every execution of \bbarbalgo\ satisfies the following statements:
\begin{itemize}
\item $\frontier$ always contains tuples $(v_1,v_2, c,n)$ such that $v_2$ is a child of $v_1$ in $T$.
\item  Once $(v_1, v_2, c,n)$ is added to $\frontier$, no other tuple of the form $(v_1, v_2, c', n')$ may ever be added to $\frontier$
in future iterations. 
\item At every iteration of the while loop, at least one tuple is removed from $\frontier$. 
\item If $(v_1, v_2, c,n)$ is removed from $\frontier$, 
then either the procedure returns \no\ or for every child $v_3$ of $v_2$, a tuple whose first two arguments 
are $(v_2,v_3)$ is added to $\frontier$. 
\end{itemize}

\smallskip

\noindent\textbf{Claim 2}. The while loop in Step 4 has the following loop invariants:
\begin{itemize}
\item $\stackheight$ is equal to number of symbols on $\stack$.
\item If $\stackheight > 0$, then all tuples $(u,v,c,n)$ with $n> 0$ have the same $c$. 
\item All tuples $(u,v,c,n)$ in $\frontier$ are such that $n \leq \stackheight$, and there exists at least one tuple with $n=\stackheight$. 
%\item The number of tuples in $\frontier$ is at most the number of leaves in $\q$.
\end{itemize}

\medskip

We now show the first statement of the proposition. 

\medskip

\noindent\textbf{Claim 3}. Every execution of \bbarbalgo\ terminates.
\smallskip

\noindent\emph{Proof of claim}. A simple examination of \bbarbalgo\ shows that the only possible source of non-termination 
is the while loop in Step 4, which continues so long as $\frontier$ is non-empty. 
It follows from the first statement of Claim 1 that the total number of tuples that may appear in $\frontier$ at some point cannot exceed the number of edges in $T$, 
which is itself bounded above by $|\q|$. We also know from the second and third statements of Claim 1 that every tuple is added at most once and is eventually removed from $\frontier$ . Thus, either we will exit the while loop by returning \no\ (if one of the checks fails), or we will eventually exit the while loop after reaching an empty $\frontier$.  (\emph{end proof of claim})

\medskip

The next two claims establish the second half of the proposition. 

\medskip

\noindent\textbf{Claim 4}. If $\Tmc, \Amc \models \q(\vec{b})$, then some execution of \bbarbalgo($\Tmc, \Amc, \q, \vec{b})$ returns \yes.

\smallskip
\noindent\emph{Proof of claim}. Suppose that $\Tmc, \Amc \models \q(\vec{b})$. Then there exists a homomorphism
$h: \q \rightarrow \canmod$ such that $h(\avars(\q)) = \vec{b}$, and without loss of generality we may choose $h$ so that the image of $h$ consists
of elements $aw$ with $|w| \leq 2 |\Tmc| + |\q|$. We use $h$ to specify an execution of \bbarbalgo($\Tmc, \Amc, \q, \vec{b})$ that 
returns \yes. In Step 1, we fix some arbitrary variable $v_0$ as root, and in Step 2, we let choose the element $h(v_0)=a_0w_0$.
Since $h$ defines a homomorphism of $\q(\vec{b})$ into $\canmod$, the call to \compnode\ or \compnodeanon\ %in Step 3 
will return \true.
In Steps 3, we will initialize $\stack$ to $w_0$, $\stackheight$ 
to $|w_0|$, and $\frontier$ to $\{(v_0, v_i, a_0, \stackheight) \mid v_i \text{ is a child of } v_0 \}$.
In Step 4, we enter the while loop. Our aim will be to make the non-deterministic choices in 
such a way as to satisfy the following invariant:
\begin{description}
\item[Inv]If $(v,v',c,m)$ is in $\frontier$ and $w=\stack[m]$, then $h(v)=cw$.
\end{description}
Recall that $\stack[m]$ designates the word obtained by concatenating the first $m$ symbols of $\stack$. % (so the bottom symbol in $\stack$ will be the first latter and the $m$th symbol in $\stack$ will be the final letter). 
 Observe that at the start of Step 4,  property \textbf{Inv} is satisfied. 
At the start of each iteration of the while loop, we proceed as follows:
\begin{description}
\item[Case 1] $\frontier$ contains an element $\tau = (v_1,v_2,c,0)$ such that $h(v_2) \in \ainds(\Amc)$. In this case, 
we will choose Option 1. In Step 4(a), we will remove $\tau$ from $\frontier$, and in 4(b), we guess the individual $h(v_2)$. 
As $c=h(v_1)$ (by \textbf{Inv}) and $h$ is a homomorphism, the calls to \compnode\ and \compedge\ will both return \true. 
We will thus continue to 4(c) where we will add $(v_2,v_3,h(v_2),0)$ to $\frontier$ for every child $v_3$ of $v_2$. Note 
that these additions to $\frontier$ preserve the invariant. % \textbf{Inv}. 
\item[Case 2] $\frontier$ contains $\tau = (v_1,v_2,c,\stackheight)$ such that $h(v_2) = h(v_1) S$. In this case,
we choose Option 2 and remove $\tau$ from $\frontier$ in 4(d). Note that we must have $\stackheight < 2|\Tmc|+|\q|$ since 
(i) by the invariant \textbf{Inv}, $h(v_1)=c w$ where $w=\stack[\stackheight]$, and (ii) by our choice of the homomorphism $h$, 
we have that $w S \leq 2|\Tmc|+|\q|$. We will thus continue on to Step 4(e), where we choose the role $S$. 
Because of the invariant, the fact that $h(v_2) = h(v_1) S$, and that $h$ is a homomorphism, one can show that 
none of the (undesired) properties in 4(e) holds, and so we will continue to 4(f). First consider the case in which $v_2$ has 
some child. In this case, we push $S$ onto $\stack$, increment $\stackheight$, and add $(v_2,v_3,c,\stackheight)$ to $\frontier$ 
for every child $v_3$ of $v_2$. Observe that \textbf{Inv} holds for the newly added tuples and continues to hold for existing tuples. 
If $v_2$ is a leaf in $T$, then no additions are made to $\frontier$, but we pop $\delta$ symbols from $\stack$ and decrement $\stackheight$
by $\delta$, where $\delta$ is the difference between $\stackheight$ and the maximal current value appearing in any tuple of $\frontier$. 
Since there are no additions, and the relevant initial segment of $\stack$ remains unchanged, \textbf{Inv} continues to hold. 
\item[Case 3] Neither Case 1 nor Case 2 holds. 
In this case, we choose Option 3, and remove all elements in $\deepest = \{(v_1,v_2,c,n) \in \frontier \mid n = \stackheight\}$ from $\frontier$.
Since neither Case 1 nor Case 2 applies, $\stackheight > 0$. Thus, in Step 4(g), we will not return \no\ and will instead 
pop the top symbol $R$ from $\stack$ and decrement $\stackheight$ by 1.  Since $\stackheight > 0$, it follows from Claim 2 that 
all tuples in $\deepest$ have the same individual $c$ in third position. By the invariant \textbf{Inv}, 
for every tuple $(v_1,v_2,c,n) \in \deepest$ is such that $h(v_1)=cwR$ where $wR =\stack[\stackheight]$. 
Moreover, since Case 2 was not applicable, we know that for every such tuple $(v_1,v_2,c,n)$, 
we have $h(v_2)=cw$. Using the fact that $h$ is a homomorphism, we can show that none of the undesired properties 
in Step 4(h) holds, and so we will continue on to 4(i), where we will set $\children = \{(v_2,v_3) \mid (v_1, v_2, c,n) \in \deepest, v_3 \text{ is a child of } v_2\}$. 
If $\children$ is non-empty, then we will add the tuple $(v_2, v_3, c, \stackheight)$ to $\frontier$ for each pair $(v_2,v_3) \in \children$.
Note that the invariant \textbf{Inv} is satisfied by all the new tuples. Moreover, since we only removed the last symbol in $\stack$, 
all the remaining tuples in $\frontier$ will continue to satisfy \textbf{Inv}. If $\children$ is empty, then we pop $\delta$ symbols 
from $\stack$ and decrement $\stackheight$ by $\delta$, where $\delta = \stackheight - \mathbf{max} \{\ell \mid (v,v',d,\ell) \in \frontier\}$.
We can use the same reasoning as in Option 2 to show that \textbf{Inv} continues to hold. 
\end{description}
Since we have shown how to make the non-deterministic choices in the while loop without returning \no, 
we will eventually leave the while loop (by Claim 3), and return \yes\ in Step 5. (\emph{end proof of claim})
\medskip

\noindent\textbf{Claim 5}. If some execution of \bbarbalgo($\Tmc, \Amc, \q, \vec{b})$ returns \yes, then $\Tmc, \Amc \models \q(\vec{b})$.

\smallskip
\noindent\emph{Proof of claim}. Consider an execution of \bbarbalgo($\Tmc, \Amc, \q, \vec{b})$ that returns \yes. 
Since \yes\ can only be returned in Step 5, it follows that the while loop was successfully exited after reaching an empty $\frontier$. 
Let $L$ be the total number of iterations of the while loop. We inductively define a sequence $h_0, h_1, \ldots, h_L$ of partial functions 
from $\vars(\q)$ to $\Delta^{\canmod}$ by considering the guesses made during the different iterations of the while loop. 
We will ensure that the following properties hold for every $0 \leq i < L$:
\begin{description}
\item[P1] If $i>0$, then $\dom(h_{i_1}) \subseteq \dom(h_{i})$, and if $v \in \dom(h_{i-1})$ is defined, then $h_{i}(v) = h_{i-1}(v)$.
\item[P2] If tuple $(v_1,v_2, c,n)$ belongs to $\frontier$ at the beginning of iteration $i+1$, then:
\begin{itemize}
\item[(a)] $h_i(v_1)=c w$ where $w= \stack[n]$ (recall that $\stack[n]$ consists of the first $n$ symbols of $\stack$)
\item[(b)] neither $v_2$ nor any of its descendants belongs to $\dom(h_{i})$. 
\end{itemize}
\item[P3] $h_i$ is a homomorphism from $\q_i$ to $\canmod$, where $\q_i$ is the restriction of $\q$ to the variables in $\dom(h_i)$.
\end{description}
Note that above and in what follows, we use $\dom(h_i)$ to denote the domain of the partial function $h_i$. 

We begin by setting $h_0(v_0)=u_0$ (and leaving $h_0$ undefined for all other variables). Property \textbf{P1} is not applicable. 
Property \textbf{P2}(a) holds because of the initial values of $\frontier$, $\stack$, and $\stackheight$,
and \textbf{P2}(b) holds because only $v_0 \in \dom(h_0)$, and $v_0$ cannot be its own child (hence cannot appear in the second argument of 
a tuple in $\frontier$). To see why \textbf{P3} is satisfied, first suppose that $u_0 \in \ainds(\Amc)$. Then in Step 2, the subprocedure \compnode\ 
was called on input $(\Tmc,\Amc,\q, \vec{b}, v_0, u_0)$ and returned \yes. It follows that
\begin{itemize}
% \item $u \not \in \Delta^{\canmod}$
  \item if $v_0=z_j$, then $u_0 = b_j$; 
 \item if $\q$ contains $A(v_0)$, then $u_0 \in A^{\canmod}$;
 \item if $\q$ contains $r(v,v)$, then $(u_0,u_0) \in r^{\canmod}$;
 \end{itemize}
and hence that $h_0$ defines a homomorphism of $\q_0$ into $\canmod$. 
The other possibility is that $u_0=a_0w_0$ for some non-empty word $w_0 = w_0' R$, 
and so in Step 2, \compnodeanon\ was called on input $(\Tmc,\q, v_0, R)$ and returned \yes.
It follows that 
\begin{itemize}
\item  $v_0 \not \in \avars(\q)$; 
\item if $\q$ contains $A(v)$, then $\Tmc \models \exists y R(y,x) \rightarrow A(x)$ (hence: $u_0 \in A^{\canmod}$);
\item $\q$ does not contain any atom of the form $S(v,v)$;
 \end{itemize}
and hence $h_0$ maps all atoms of $\q_0$ into $\canmod$. We have thus shown that the initial partial function $h_0$
satisfies the three requirements. 

Next we show how to inductively define $h_{i}$ from $h_{i-1}$ while preserving properties \textbf{P1} --\textbf{P3}.
The variables that belong to $\dom(h_{i}) \setminus \dom(h_{i-1})$ are precisely those 
variables that appear in the second position of a tuple removed from $\frontier$ during iteration $i$ (since these are the
variables for which we guess a domain element). 
The choice of where to map these variables depends on which of three options was selected:\smallskip

\noindent\textbf{Option 1}: In this case, we removed a tuple $(v_1,v_2, c, 0)$ and guessed an individual $d \in \ainds(\Amc)$. 
We set $h_{i}(v_2)=d$ and $h_{i}(v)=h_{i-1}(v)$ for all variables in $\dom(h_{i-1})$ (all other variables remain undefined). 
Property \textbf{P1} is trivially satisfied. 

For property \textbf{P2}, let $\stack_{i-1}$ designate $\stack$ at the beginning of iteration $i$,
and let $\stack_{i}$ designate $\stack$ at the beginning of iteration $i+1$. 
Consider some tuple $\tau=(v,v',a,p)$ that belongs to $\frontier$ at the beginning 
of iteration $i+1$ (equivalently, the end of iteration $i$). 
If the tuple $\tau$ was already in $\frontier$ at the beginning of iteration $i$, then we can use the fact that 
$h_{i-1}$ satisfies \textbf{P2} to obtain that:
\begin{itemize}
\item[(a)] $h_{i-1}(v)=c w$ where $w= \stack_{i-1}[n]$ 
\item[(b)] neither $v'$ nor any of its descendants belongs to $\dom(h_{i-1})$ 
\end{itemize}
Since $\stack_{i} = \stack_{i-1}$ and $h_i(v)=h_{i-1}(v)$, it follows that statement (a) continues to hold for $\tau$. 
Moreover, since $\tau$ was not removed from $\frontier$ during iteration $i$, 
we have that $\tau \neq (v_1,v_2, c, 0)$, and so using Claim 1, we can conclude that $v' \neq v_2$.
It follows that neither $v'$ nor any descendant is in  $\dom(h_{i})$. 
The other possibility is that the tuple $\tau$ was added to $\frontier$ during 
iteration $i$, in which case $\tau = (v_2,v_3, d, 0)$ for some child $v_3$ of $v_2$.
Condition (a) is clearly satisfied (since $\stack_{i}[0]=\epsilon$). 
Since $h_{i-1}$ satisfies \textbf{P2}, we know that $v_3$ (being a descendant of $v_2$) is not in $\dom(h_{i-1})$, 
and so remains undefined for $h_{i}$. 

To show property \textbf{P3}, we first note that since $h_{i}$ agrees with $h_{i-1}$ on all variables in $\dom(h_i)$,
it is only necessary to consider the atoms in $\q_{i}$ that do not belong to $\q_{i-1}$. There are four kinds of such atoms:
%atoms of the form $A(v_2)$ and atoms $R(v,v')$ such that either $v=v_2$ or $v'=v_2$. 
\begin{itemize}
\item Atoms of the form $A(v_2)$: if $A(v_2) \in \q$, then \compnode($\Tmc,\Amc,\q, \vec{b}, v_2,d$)=\true\ implies that
$h_{i}(v_2) = d \in A^{\canmod}$.
\item Atoms of the form $R(v_2,v_2)$:  if $R(v_2,v_2) \in \q$, then we can again use the fact that  \compnode($\Tmc,\Amc,\q, \vec{b}, v_2,d$)=\true\
to infer that $(h_{i}(v_2),h_{i}(v_2))=(d,d) \in R^{\canmod}$.
\item Atoms of the form $R(v_2,v)$ with $v \neq v_2$: since $R(v_2,v) \in \q_{i}$, we know that $v$ must belong to $\dom(h_{i})$,
so $v$ must be the parent $v_1$ (rather than one of $v_2$'s children). We can thus use the fact that 
\compedge($\Tmc, \Amc, \q, \vec{b}, v_1, v_2, c, d$)=\true\  to obtain $(h_{i}(v_2),h_{i}(v))= (c,d) \in R^{\canmod}$.
\item Atoms of the form $R(v,v_2)$ with $v \neq v_2$: analogous to the previous case.
\end{itemize}
We have thus shown that property \textbf{P3} holds for $h_{i}$. 

\smallskip

\noindent\textbf{Option 2}:  If Option 2 was selected during iteration $i$, then 
a tuple $(v_1,v_2, c,n)$ was removed from $\frontier$ with $n$ equal to the value of $\stackheight$, and then a role $S$ was guessed.  
We set $h_{i}(v_2)=h_{i-1}(v_1) S$. Note that we are sure that $h_{i-1}(v_1)$ is defined, since $h_{i-1}$ satisfies property \textbf{P2}. 
Moreover, the first two checks in Step 4(e) ensure that $h_{i-1}(v_1) S$ belongs to the domain of $\canmod$. 
We also set $h_{i}(v)=h_{i-1}(v)$ for all variables in $\dom(h_{i-1})$ and leave the remaining variables undefined. 

Property \textbf{P1} is immediate from the definition of $h_i$, and property \textbf{P2}(b) can be shown exactly as for Option 1. To show 
\textbf{P2}(a), we define $\stack_{i-1}$ and $\stack_{i}$ as in Option 1, and consider 
a tuple $\tau=(v,v',a,p)$ that belongs to $\frontier$ at the beginning of iteration $i+1$. 
If $\tau$ was present in $\frontier$ at the beginning of iteration $i$, then $h_{i-1}(v)=c w$ where $w= \stack_{i-1}[p]$
(since $h_{i-1}$ satisfies \textbf{P2}). Since $\stack_{i} = \stack_{i-1} \, S$, $p \leq |\stack_{i-1}|$ and $h_i(v)=h_{i-1}(v)$, 
it follows that statement (a) continues to hold for $\tau$. The other possibility is that $\tau$ was added to $\frontier$
during iteration $i$, in which case $\tau$ must take the form $(v_2, v_3, c, n+1)$ for some child $v_3$ of $v_2$. 
Since $h_{i-1}$ satisfies \textbf{P2}, we know that $h_{i-1}(v_1)= c \cdot \stack_{i-1}[n]$. Statement (a) follows then 
from the fact that $h_i(v_2)= h_{i-1}(v_1) S$ and $\stack_{i} = \stack_{i-1} \, S$. 

We now turn to property \textbf{P3}. As explained in the proof for Option 1, it is sufficient to consider the atoms in 
$\q_{i} \setminus \q_{i-1}$, which can be of the following four types:
\begin{itemize}
\item Atoms of the form $A(v_2)$:  if $A(v_2) \in \q$, then \compnodeanon($\Tmc,\q,v_2,S$)=\true\ implies that
$\Tmc \models \exists y S(y,x)\rightarrow A(x)$, hence $h_{i}(v_2)=h_{i-1}(v_1) S\in A^{\canmod}$.
\item  Atoms of the form $R(v_2,v_2)$:  \compnodeanon($\Tmc,\q,v_2,S$)=\true\ implies that no such atom occurs in $\q$. 
\item  Atoms of the form $R(v_2,v)$ with $v \neq v_2$: if $R(v_2,v) \in \q_{i}$, the only possibility is that $v=v_1$ (cf.\ proof for Option 1).
We know from the third check in Step 4(e) that $\Tmc \models S(x,y) \rightarrow R(y,x)$, which shows that $(h_i(v_2),h_i(v)) = (h_{i-1}(v_1) S, h_{i-1}(v_1)) \in R^{\canmod}$.
\item Atoms of the form $R(v,v_2)$ with $v \neq v_2$: analogous to the previous case.
\end{itemize}
This establishes that $h_i$ is a homomorphism from $\q_i$ into $\canmod$, so $h_i$ satisfies \textbf{P3}.
\smallskip

\noindent\textbf{Option 3}:  If it is Option 3 that was selected during iteration $i$, then 
the tuples in $\deepest = \{(v_1,v_2,c,n) \in \frontier \mid n = \stackheight\}$ were removed 
from $\frontier$, and the role $R$ was popped from $\stack$. We know from Claim 2 that all tuples in 
$\deepest$ contain the same individual $c$ in their third position. 
For every variable $v \in \mathsf{DVars}=\{v_2 \mid v_1,v_2,c,n) \in \deepest\}$, we set $h_i(v) = c w$, 
where $w$ is equal to $\stack$ after $R$ has been popped. As for the other two options, 
we set $h_{i}(v)=h_{i-1}(v)$ for all variables in $\dom(h_{i-1})$ and leave the remaining variables undefined. 

Property \textbf{P1} is again immediate, and the argument for property \textbf{P2}(b) is the same as for Option 1. 
For property \textbf{P2}(a), 
let $\stack_{i-1}$ and $\stack_{i}$ be defined as earlier, l
and let  $\tau=(v,v',a,p)$ be a tuple that in $\frontier$ at the beginning of iteration $i+1$. 
If $\tau$ was present in $\frontier$ at the beginning of iteration $i$, then $h_{i-1}(v)=c w$ where $w= \stack_{i-1}[p]$,
and $p$ must be smaller than the value of $\stackheight$ at the start of iteration $i$.
We know that $\stack_{i}$ is obtained from $\stack_{i-1}$ by popping one or more symbols, 
and that at the end of iteration $i$, $\stackheight$ is equal to the largest value appearing in a tuple of $\frontier$. 
We thus know that at the start of iteration $i+1$, $p \leq \stackheight$, 
and so \textbf{P2}(a) continues to hold for $\tau$. Next consider the other possibility, which is that the tuple 
$\tau=(v,v',a,p)$ was added
to $\frontier$ during the $i$th iteration of the while loop. In this case, we know that 
$v \in \mathsf{DVars}$, $h_i(v)= c \stack_{i}$, and $p=|\stack_{i}|$, from which property \textbf{P2}(a) follows.

For property \textbf{P3}, the argument is similar to the other two options and involves considering
the different types of atoms that may appear in $\q_{i} \setminus \q_{i-1}$: 
\begin{itemize}
\item Atoms of the form $A(v)$ with $v\in \mathsf{DVars}$:  if $A(v_2) \in \q$, then either
\begin{itemize}
\item $|\stack_{i-1}|=1$ and \compnode($\Tmc, \Amc, \q, \vec{b}, v_2, c$)= \true, or 
\item $|\stack_{i-1}|>1$ and \compnodeanon($\Tmc,\q,v_2,S$)=\true, where $S$ is next-to-top symbol in $\stack_{i-1}$
\end{itemize}
In both cases, we may infer $h_{i}(v) \in A^{\canmod}$ (see Options 1 and 2). 
\item Atoms of the form $P(v,v)$ with $v\in \mathsf{DVars}$: in this case, 
we must have $\stackheight=0$, $h_i(v)=c$, and \compnode($\Tmc,\Amc,\q, \vec{b}, v,c$)=\true. 
The latter implies that $(h_{i}(v),h_{i}(v))=(c,c) \in P^{\canmod}$.
\item  Atoms of the form $P(v,v')$ with $v \neq v'$ and $v\in \mathsf{DVars}$: if $P(v,v') \in \q_{i}$, the only possibility is that $v'$ is the parent of $v$ (cf.\ proof for Option 1).
We know from the third check in Step 4(h) that $\Tmc \models R(y,x) \rightarrow P(x,y)$, which shows that $(h_i(v),h_i(v')) = (c \, \stack_{i}, c \, \stack_{i}\, R) \in P^{\canmod}$.
\item Atoms of the form $P(v',v)$ with $v \neq v'$ and $v\in \mathsf{DVars}$: analogous to the previous case.
\end{itemize}

%\textbf{STOPPED HERE!!!}
\medskip

We claim that the final partial function $h_L$ is a homomorphism of $\q$ to $\canmod$. 
Since $h_L$ is a homomorphism of $\q_L$ into $\canmod$, it suffices to show that $\q=\q_L$, 
or equivalently, that all variables of $\q$ are in $\dom(h_L)$. This follows from Claim 1 and the 
fact that $\dom(h_{i+1}) = \dom(h_i) \cup \{v' \mid (v,v',c,n) \text { is removed from } \frontier \text{ during iteration } i\}$. 
(\emph{end proof of claim})
\end{proof}

\bigskip 

%We split the proof of Proposition \ref{logcfl-upper-prop} into the following ... lemmas:

To complete our proof of the \LOGCFL\ upper bound, we prove the following proposition. 

\begin{proposition}
\bbarbalgo\ can be implemented by an NAuxPDA
\end{proposition}
\begin{proof}
It suffices to show that \bbarbalgo\ runs in non-deterministic logarithmic 
space and polynomial time. 

In Step 1, we non-deterministically fix a root variable $v_0$, but do not 
actually need to store the induced directed tree $T$ in memory, since it suffices   % (cf.\ proof of Theorem \ref{nl-bb}). 
to be able to decide given two variables $v,v'$ whether $v$ is the parent of $v'$ in $T$, and the 
latter problem clearly belongs to \NL. 

In Step 2, we need only logarithmic space to store the individual $a_0$. The word $w_0 = \varrho_1 \ldots \varrho_N$ 
can be guessed symbol by symbol and pushed onto $\stack$. 
We recall that $a_0w_0 \in \Delta^{\canmod}$ just in the case that:
\begin{itemize}
\item $\Tmc, \Amc \models \exists y \varrho_1(a,y)$ and  $\Tmc, \Amc \not \models \varrho_1(a,b)$ for any $b \in \ainds(\Amc)$; 
\item for every $1 \leq i < N$: $\mathcal{T} \models \exists y\, \varrho_i(y,x) \rightarrow \exists y \, \varrho_{i+1}(x,y)$
  and $\Tmc \not \models \varrho_i(x,y) \rightarrow \varrho_{i+1}(y,x)$. %and $R_i^- \ne R_{i+1}$.
\end{itemize}
Thus, it is possible to perform the required entailment checks incrementally as the symbols of $w_i$ are guessed. 
Finally, to ensure that the guessed word $w_0$ does not exceed the length bound, 
each time we push a symbol onto $\stack$, we increment $\stackheight$ by $1$.
If $\stackheight$ reaches $2 |\Tmc| + |\q|$, then no more symbols may be guessed. 
We next call either sub-procedure \compnode\ or \compnodeanon. It is easy to see that both can be made to run in 
non-deterministic logarithmic space. % (cf.\ proof of Theorem \ref{nl-bb}). 

The initializations of $\stack$ and $\stackheight$ in Step 3 were already handled in our discussion of Step 2. 
Since the children of a node in $T$ can be identified in \NL, % (again we refer back to the proof of Theorem \ref{nl-bb}),
we can decide in non-deterministic logspace whether a tuple $(v_0,v_i, a_0, \stackheight)$ should be included
in $\frontier$. Moreover, since the input query $\q$ is a tree-shaped query with a bounded number of leaves,
we know that only constantly many tuples can be added to $\frontier$ in Step 4. Moreover, it is clear that 
every tuple can be stored using in logarithmic space. More generally, using 
Claims 1 and 2 from the proof of Proposition \ref{logcfl-upper-prop}, one can show that 
$|\frontier|$ is bounded by a constant throughout the execution of the procedure, 
and the tuples added during the while loop can also be stored using only logarithmically many bits. 

Next observe that every iteration of while loop in Step 4 involves a polynomial number of the
following elementary operations:
\begin{itemize}
\item remove a tuple from $\frontier$, or add a tuple to $\frontier$
\item pop a role from $\stack$, or push a role onto $\stack$
\item increment or decrement $\stackheight$ by a number bounded by $2 |\Tmc| + |q|$
\item test whether $\stackheight$ is equal to $0$ or to $2 |\Tmc| + |q|$
\item guess a single individual constant or symbol 
\item identify the children of a given variable
\item locate an atom in $\q$
\item test whether $\Tmc \models \alpha$, for some inclusion $\alpha$ involving symbols from $\Tmc$
\item make a call to one of the sub-procedures \compnode,  \compnodeanon, or \compedge
\end{itemize}
For each of the above operations, it is either easy to see, or has already been explained, that the 
operation can be performed in non-deterministic logarithmic space. 
To complete the argument, we note that it follows 
from Claim 1 (proof of Proposition \ref{logcfl-upper-prop}) that there are at most $|\q|$ many iterations of the while loop.
\end{proof}

\bigskip
\begin{proposition}
For a logspace-uniform family $\{\Cir_l\}_{l=1}^\infty$ of \sac\ circuits in normal form,  
the sequences $\q_{\Cir_l}^\mathsf{lin}$ and $(\Tmc_{\Cir_l}^{\vec{x}}, \Amc_{\Cir_l})$ are also logspace uniform.
\end{proposition}

\begin{proof}
Consider a  circuit $\Cir$ in normal form with $2d+1$ layers of gates, where $d$ is logarithmic in number of its inputs $l$. We show that $\q_{\Cir}^\mathsf{lin}$ and $(\Tmc_{\Cir}^{\vec{x}}, \Amc_{\Cir})$ can be constructed 
using $O(\log(l))$ worktape memory.

\begin{itemize}
\item To produce the query $\q_{\Cir}^\mathsf{lin}$, we can generate the word $w_d$ letter by letter and insert the corresponding variables. 
This can be done by a simple recursive procedure of depth $d$, using the worktape to remember the current position in the recursion tree
as well as the index of the current variable $y_i$. Note that $|w_d|$ (hence the largest index of the query variables) may be exponential in $d$, 
but is only polynomial in $l$, and so we need only logarithmic space to store the index of the current variable. 
\item The ontology $\Tmc_{\Cir}^{\vec{x}}$ is obtained by 
making a single pass over a (graph representation) of the circuit and generating the 
axioms that correspond to the gates of $\Cir$ and the links between $\Cir$'s gates. 
To decide which axioms of the form $G_i(x) \rightarrow A(x)$ to include,  we must also look up the value 
of the variables associated to the input gates under the valuation $\vec{x}$. 
%doing a single pass over
%the circuit $\Cir$ and the database $D_{\Cir}^\vec{x}$. On this pass, we map the 
%conductors  of the circuit and database atoms into ontology axioms and do not 
%require any worktape memory.
\item $\Amc_{\Cir}$ consists of a single constant atom.
\end{itemize}
\end{proof}

\bigskip

\noindent{\bf Proposition \ref{logcfl-lower-prop}.}
\emph{ $\Cir$ accepts input $\vec{x}$ iff 
$\Tmc_{\Cir}^{\vec{x}}, \Amc_{\Cir} \models \qclin(a)$.}

\begin{proof}
Denote by $p_\q$ the natural homomorphism from $\qclin$ to $\qc$, and
by $p_\Cmc$ the natural homomorphism from $\CmC$ to $\dcx$.
As it is proven in \cite{DBLP:journals/jacm/GottlobLS01} that 
$\Cir$ accepts input $\vec{x}$ iff there is a homomorphism $h$ from  $\qc$ to $\dcx$, it 
suffices to show that there exists a homomorphism $f$ from $\qclin$ to $\CmC$ 
iff there is a homomorphism $h$ from $\qc$ to $\dcx$.

\begin{center}
\begin{tikzpicture}[auto]
\node at (-0.6,2) {(a)};
\node (qclin) at (0,2) {$\q_\Cir^{\mathsf{lin}}$};
\node (qc) at (0,0) {$\q_\Cir$};
\node (dc) at (2,0) {$D_\Cir^\vec{x}$};
\node (cm) at (2,2) {$\CmC$};
\draw[->] (qclin) to node [left] {$p_\q$} (qc);
\draw[->] (cm) to node {$p_\Cmc$} (dc);
\draw[->] (qc) to node [below]  {$h$} (dc);
\draw[->,dashed] (qc) to node [below right] {$h'$} (cm);
\draw[->,dashed] (qclin) to node [above] {$f$} (cm);
\node at (3.4,2) {(b)};
\node (qclin1) at (4,2) {$\q_\Cir^{\mathsf{lin}}$};
\node (qc1) at (4,0) {$\q_\Cir$};
\node (dc1) at (6,0) {$D_\Cir^\vec{x}$};
\node (cm1) at (6,2) {$\CmC$};
\draw[->] (qclin1) to node [left] {$p_\q$} (qc1);
\draw[->] (cm1) to node {$p_\Cmc$} (dc1);
\draw[->] (qclin1) to node [above] {$f$} (cm1);
\draw[->,dashed] (qc1) to node [below right] {$f'$} (cm1);
\draw[->,dashed] (qc1) to node [below]  {$h$} (dc1);
\end{tikzpicture}
\end{center}

($\Rightarrow$) Suppose that $h$ is a homomorphism  from $\qc$ to $\dcx$.
We define the homomorphism $h' : \q_\Cir \to \CmC$ inductively moving from
the root $n_1$ of $\q_\Cir$ to its leaves. First,  we set $h'(n_1) = a$. Note that
$\CmC \models G_1(a)$. Then we proceed by induction.
Suppose that $n_j$ is a child of $n_i$, $h'(n_i)$ is defined, $\CmC \models G_{i'}(h(n_i))$
and  $h(n_j) = g_{j'}$. In this case, we set $h'(n_j) = h'(n_i) P^-_{i'j'}$.  It follows from the definition of
$\Tmc_\Cir^\vec{x}$ that $\CmC \models G_{j'}(h'(n_j))$, 
which enables us to continue the induction. It should be clear that
$h'$ is indeed a homomorphism from $\qc$ into $\CmC$. Since the composition of homomorphisms is again a homomorphism, we can 
obtain the desired homomorphism $f: \qclin \rightarrow \CmC$ by setting $f = p_\q \,\circ\, h'$. This is illustrated in diagram (a) above.  
% and that the diagram (a) commutes.
%Thus we can take $f$ equal to the composition of $h'$ and the projection $p_\q$.

($\Leftarrow$) Suppose that $f$ is a homomorphism  from $\qclin$ to $\CmC$. 
We prove that for all its variables $y_i, y_j$ (with $i < j$) $p_\q(y_i) = p_\q(y_j)$ implies
$f(y_i) = f(y_j)$ by induction on $|j - i|$. The base case ($|j - i| = 0$) is trivial.
For the inductive step, we may assume without loss of generality that
between $y_i$ and $y_j$ there are no intermediate variable $y_k$ with $ p_\q(y_i) = p_\q(y_k) =p_\q(y_j)$
(otherwise, we can simply use the induction hypothesis together with the transitivity of equality). It follows that 
$p_\q(y_{i+1}) = p_\q(y_{j-1})$, and the atom between $y_{j-1}$ and $y_{j}$ is oriented from $y_{i-1}$ towards $y_{j}$,
while the atom between $y_i$ and $y_{i+1}$ goes from $y_{i+1}$ to $y_i$. %use an inverse binary predicate ($\rn$, while 
%the atom between $y_{j-1}$ and $y_{j}$ is an inverse.
 Indeed, it holds if the node $n = p_\q(y_i) = p_\q(y_j)$ is an \OR-node since there are exactly two variables in $\qclin$ which are mapped to $n$, 
 and they bound the subtree in $\qc$ generated by $n$. For an \AND-node, this also holds because of our assumption about intermediate variables.
 By the induction hypothesis, we have
$f(y_{i+1}) =  f(y_{j-1}) = aw\varrho$ for some word $aw\varrho$. Since the only parent of $aw\varrho$ in $\CmC$ is $aw$, 
all arrows in relations $U$, $L$ and $R$ are oriented towards the root, 
and $f$ is known to be a homomorphism, it follows that $f(y_{i}) =  f(y_{j}) = aw$. This concludes the inductive argument.

Next define the function $f': \q_\Cir \to \CmC$ by setting $f'(x)=f(y)$ where $y$ is such that $p_\q(y)=x$. Since $p_\q(y_i) = p_\q(y_j)$ implies
$f(y_i) = f(y_j)$, we have that $f'$ is well-defined, and because $f$ is a homomorphism, the same holds for $f'$. To obtain the desired 
homomorphism from $\qc$ to $\dcx$, it suffices to consider the composition $h$  of $f'$ and $p_\Cmc$.
\end{proof}

%\input{appendix.tex}

\end{document}